\documentclass{article}
\usepackage[accepted]{icml2020}
\usepackage{csquotes}
\MakeOuterQuote{"}
\usepackage{wrapfig}
\usepackage{natbib}
\usepackage{xr}
\usepackage{xspace}
\usepackage{microtype}
\usepackage{graphicx}
\usepackage{booktabs} 
\usepackage{xi-paper}
\usepackage{subcaption}

\usepackage[title]{appendix}
\usepackage{mathtools}
\usepackage{float}
\graphicspath{ {./images/} }
\usepackage{thmtools}
\usepackage{thm-restate}
\usepackage{cleveref}

\usepackage{amsfonts}       
\usepackage{nicefrac}       

\def\theta{w}
\def\linfeps{\ell_\infty(\varepsilon)}
\def\w{\mathbf{w}}
\def\<#1,#2>{\langle #1,\,#2\rangle}

\begin{document}

\twocolumn[

\icmltitle{Concise Explanations of Neural Networks using Adversarial Training}
\begin{icmlauthorlist}
\icmlauthor{Prasad Chalasani}{x}
\icmlauthor{Jiefeng Chen}{uw}
\icmlauthor{Amrita Roy Chowdhury}{uw}
\icmlauthor{Somesh Jha}{x,uw}
\icmlauthor{Xi Wu}{g}
\end{icmlauthorlist}

\icmlaffiliation{x}{XaiPient}
\icmlaffiliation{uw}{University of Wisconsin (Madison)}
\icmlaffiliation{g}{Google}

\icmlcorrespondingauthor{Prasad Chalasani}{pchalasani@gmail.com}

\icmlkeywords{Attribution, Explainability, Adversarial, Logistic}
\vskip 0.3in]

\printAffiliationsAndNotice{}

\begin{abstract}

We show new connections between adversarial learning and
explainability for deep neural networks (DNNs). One form of 
explanation of the output of a neural network model in terms of its
input features, is a vector of feature-attributions.
Two desirable characteristics of an attribution-based
explanation are: (1) \textit{sparseness}: the attributions of
irrelevant or weakly relevant features should be negligible, thus
resulting in \textit{concise} explanations in terms of the significant
features, and (2) \textit{stability}: it should not vary significantly
within a small local neighborhood of the input. Our first contribution
is a theoretical exploration of how these two properties 
(when using attributions based on Integrated Gradients, or IG)
are related to adversarial training, for a class of
1-layer networks (which includes logistic regression models for binary and multi-class classification);
for these networks we show that (a) adversarial training using an
$\ell_\infty$-bounded adversary produces models with sparse
attribution vectors, and (b) natural model-training while encouraging
stable explanations (via an extra term in the loss function), is
equivalent to adversarial training. 
Our second contribution is an
empirical verification of phenomenon (a), which we show, somewhat
surprisingly, occurs \textit{not only in 1-layer networks, but also
  DNNs trained on standard image datasets}, and extends beyond IG-based attributions, 
  to those based on DeepSHAP:   adversarial training with
$\linf$-bounded perturbations yields significantly sparser
attribution vectors, with little degradation in performance on natural
test data, compared to natural training. Moreover, the sparseness of
the attribution vectors is significantly better than that achievable
via $\ell_1$-regularized natural training.

\end{abstract}

\section{Introduction}
\label{sec-intro}

Despite the recent dramatic success of deep learning models in a
variety of domains, two serious concerns have surfaced about these
models.

\paragraph{Vulnerability to Adversarial Attacks:} 
We can abstractly think of a neural network model as a function
$F(\bfx)$ of a $d$-dimensional input vector $\bfx \in \R^d$, and the
range of $F$ is either a discrete set of class-labels, or a continuous
set of class probabilities.  Many of these models can be foiled by an
adversary who imperceptibly (to humans) alters the input $\bfx$ by
adding a perturbation $\delta \in \R^d$ so that $F(\bfx + \delta)$ is
very different from $F(\bfx)$ \cite{Szegedy2013-tx, Goodfellow2014-cy,
  Papernot2015-pg, Biggio2013-rx}.  \textit{Adversarial training} (or
\textit{adversarial learning}) has recently been proposed as a method
for training models that are robust to such attacks, by applying
techniques from the area of Robust Optimization~\cite{Madry2017-zz,
  Sinha2018-mj}.  The core idea of adversarial training is simple: we
define a set $S$ of allowed perturbations $\delta \in \R^d$ that we
want to "robustify" against (e.g. $S$ could be the set of $\delta$
where $||\delta||_\infty \leq \varepsilon$), and perform
model-training using Stochastic Gradient Descent (SGD) exactly as in
natural training, except that each training example $x$ is perturbed
adversarially, i.e. replaced by $x + \delta^*$ where $\delta^* \in S$
maximizes the example's loss-contribution.

\paragraph{Explainability:} 
One way to address the well-known lack of explainability of deep
learning models is \textit{feature attribution}, which aims to explain
the output of a model $F(\bfx)$ as an attribution vector $A^F(\bfx)$
of the contributions from the \textit{features} $\bfx$.
There are several feature-attribution techniques in the literature,
such as \textit{Integrated Gradients} (IG)~\cite{Sundararajan2017-rz},
\textit{DeepSHAP}~\cite{lundberg2017unified}, and \textit{LIME}~\cite{Ribeiro2016-nj}.  For such an
explanation to be human-friendly, it is highly desirable
\cite{molnar2019} that the attribution-vector is \textit{sparse}, i.e., only
the features that are truly predictive of the output $F(\bfx)$ should
have significant contributions, and irrelevant or weakly-relevant
features should have negligible contributions.  A sparse attribution
makes it possible to produce a \textit{concise} explanation, where
only the input features with significant contributions are included.
For instance, if the model $F$ is used for a loan approval decision,
then various stakeholders (like customers, data-scientists and
regulators) would like to know the reason for a specific decision in
simple terms.  In practice however, due to artifacts in the training
data or process,
the attribution vector is often not sparse and irrelevant or
weakly-relevant features end up having significant
contributions~\cite{Tan2013-ot}.  Another desirable property of a good
explanation is \textit{stability}: the attribution vector should not
vary significantly within a small local neighborhood of the input $x$.
Similar to the lack of concise explainability, natural training often
results in explanations that lack
stability~\cite{Alvarez-Melis2018-qr}. 

Our paper shows new connections between adversarial
robustness and the above-mentioned desirable properties of explanations, namely
conciseness and stability. Specifically, let $\tilde{F}$ be an
adversarially trained version of a classifier $F$, and for a given input vector $\bfx$ and 
attribution method $A$, let $A^F(\bfx)$ and $A^{\tilde{F}}(\bfx)$ denote the 
corresponding attribution vectors. The central research question this paper addresses is:
\begin{quote}
  Is $A^{\tilde{F}}(\bfx)$ sparser and more stable than 
  $A^F(\bfx)$?
\end{quote}
The main contributions of our paper are as follows:

\noindent
{\it Theoretical Analysis of Adversarial Training}:  Our first set of
results show via a \textit{theoretical} analysis that
$\linfeps$-adversarial training 1-layer networks tends to produce
sparse attribution vectors for IG, which in turn leads to concise explanations.
In particular, under some assumptions, we show (Theorems
\ref{thm-exp-sgd} and \ref{thm:multi_class}) that for a general class of convex loss functions
(which includes popular loss functions used in 1-layer networks, such
as logistic and hinge loss, used for binary or multi-class classification), 
and adversarial perturbations $\delta$
satisfying $||\delta||_\infty \leq \varepsilon$, the weights of "weak"
features are on average more aggressively shrunk toward zero than
during natural training, and the rate of shrinkage is proportional to
the amount by which $\varepsilon$ exceeds a certain measure of the
"strength" of the feature.  This shows that $\linfeps$-adversarial
training tends to produce sparse \textit{weight vectors} in popular
1-layer models.  In Section \ref{sec-attr} we show (Lemma
\ref{lem-1-layer-attr}) a closed form formula for the IG vector of
1-layer models, that makes it clear that in these models, sparseness
of the \textit{weight vector} directly implies sparseness of the
\textit{IG vector}.

\noindent
{\it Empirically Demonstrate Attribution Sparseness}:  In Section \ref{sec-exp} we
\textit{empirically} demonstrate that this "sparsification" effect of
$\linfeps$-adversarial training holds not only for 1-layer networks
(e.g. logistic regression models), but also for Deep Convolutional
Networks used for image classification, and extends beyond IG-based attributions, to 
those based on DeepSHAP.
Specifically, we show this
phenomenon via experiments applying $\linfeps$-adversarial training to
(a) Convolutional Neural Networks on public benchmark image datasets
MNIST \cite{lecun-mnisthandwrittendigit-2010} and Fashion-MNIST
\cite{Xiao2017-zp}, and (b) logistic regression models on the Mushroom
and Spambase tabular datasets from the UCI Data Repository
\cite{Dua:2017}.  In all of our experiments, we find that it is
possible to choose an $\ell_\infty$ bound $\varepsilon$ so that
adversarial learning under this bound produces attribution vectors
that are sparse on average, \textit{with little or no drop in
  performance on natural test data}.  A visually striking example of
this effect is shown in Figure \ref{fig-intro-cifar10} (the Gini
Index, introduced in Section \ref{sec-exp}, measures the sparseness of
the map).

It is natural to wonder whether a traditional
\textit{weight-regularization} technique such as
$\ell_1$-regularization can produce models with sparse attribution vectors.  In
fact, our experiments show that for logistic regression models,
$\ell_1$-regularized training does yield attribution vectors that are on
average significantly sparser compared to attribution vectors from natural
(un-regularized) model-training, and the sparseness improvement is
almost as good as that obtained with $\linfeps$-adversarial training.
This is not too surprising given our result (Lemma
\ref{lem-1-layer-attr}) that implies a direct link between sparseness
of \textit{weights} and sparseness of \textit{IG vectors}, \textit{for
  1-layer models}.  Intriguingly, this does \textit{not} carry over to DNNs:
for multi-layer models (such as the ConvNets we trained for the image
datasets mentioned above) we find that with $\ell_1$-regularization,
the sparseness improvement is significantly inferior to that
obtainable from $\linfeps$-adversarial training (when controlling for
model accuracy on natural test data), as we show in Table
\ref{tab-summary}, Figure \ref{fig-boxplots} and
Figure \ref{fig-varplots}.  
Thus it appears that for DNNs, the \textit{attribution-sparseness} that results from 
adversarial training is not necessarily related to \textit{sparseness of weights}.

{\it Connection between Adversarial Training and Attribution Stability}:  We also show
theoretically (Section \ref{sec-stable}) that training 1-layer networks
naturally, while encouraging stability of explanations (via a suitable
term added to the loss function), is in fact equivalent to adversarial
training.

\begin{figure}[htb]
	\centering 
	\begin{minipage}{\linewidth}
		\hspace{3.2cm} \begin{tiny}{\textbf{Natural Training}}\end{tiny} \hspace{0.5cm} \begin{tiny}{\textbf{Adversarial Training}}\end{tiny}
	\end{minipage}
	\begin{subfigure}{\textwidth}
		\begin{subfigure}{0.15\textwidth}
			\includegraphics[width=\linewidth,bb=0 0 449 464]{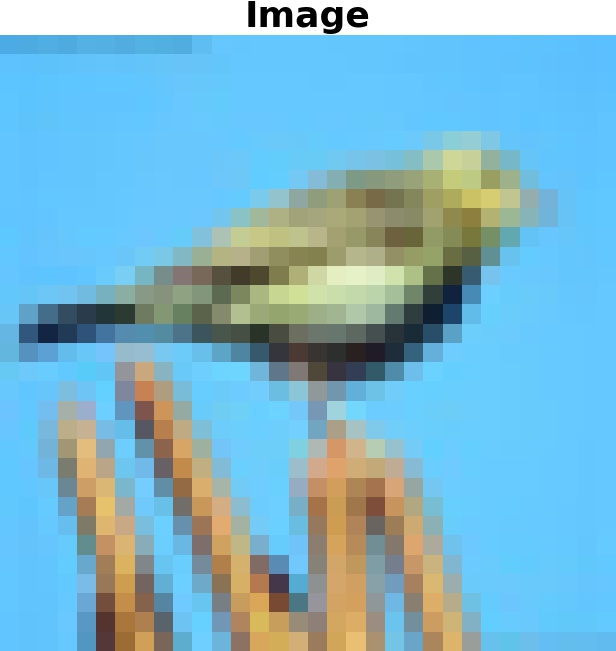} 
			\captionsetup{justification=centering}
			\caption*{}
		\end{subfigure}
		\begin{subfigure}{0.15\textwidth}
			\includegraphics[width=\linewidth,bb=0 0 449 464]{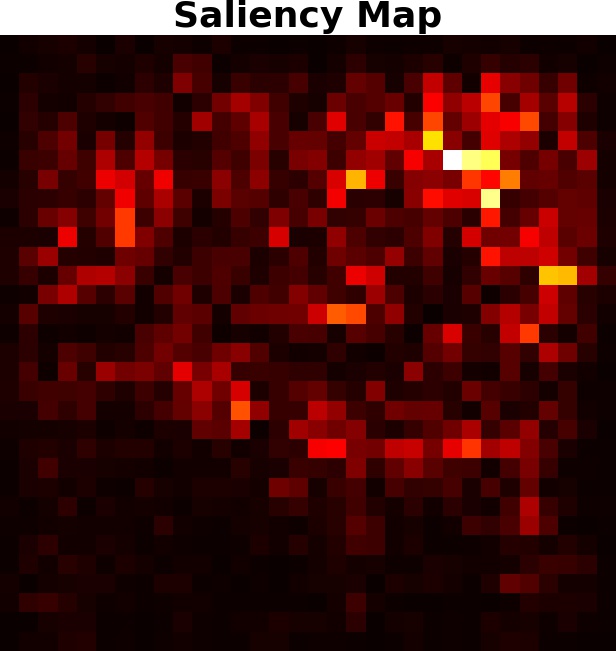}
			\captionsetup{justification=centering}
			\caption*{\begin{small}{Gini: 0.6150}\end{small}}
		\end{subfigure}
		\begin{subfigure}{0.15\textwidth}
			\includegraphics[width=\linewidth,bb=0 0 449 464]{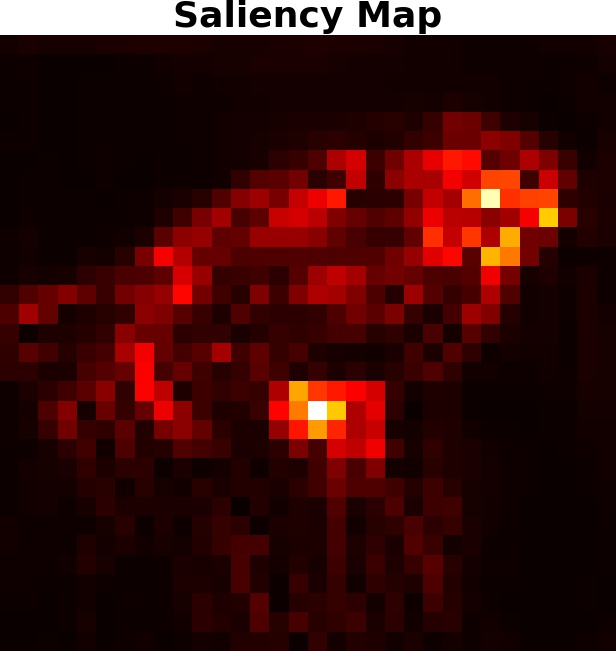} 
			\captionsetup{justification=centering}
			\caption*{\begin{small}{Gini: 0.7084}\end{small}}
		\end{subfigure}
	\end{subfigure}
	
	\caption{Both models correctly predict "Bird", but the IG-based saliency map  of the adversarially trained model is much sparser than that of the naturally trained model.}
	\label{fig-intro-cifar10}    
	
\end{figure}

\section{Setup and Assumptions} \label{sec-notation}

\newcommand{\lossxyw}{\mathcal{L}(\bfx, y; \bfw)}
\newcommand{\deltainf}{||\delta||_\infty}
\newcommand{\deltainfeps}{\Delta_\infty(\varepsilon)}
For ease of understanding, we consider the case of binary classification for the rest of our discussion in the main paper.
We assume there is a distribution 
$\mathcal{D}$ of data points $(\bfx, y)$ where $\bfx \in \R^d$ is an input feature vector, and $y \in \{\pm 1\}$ is its true label\footnote{It is trivial to convert -1/1 labels to 0/1 labels and vice versa}. 
For each $i \in [d]$, the $i$'th component of $\bfx$ represents an input feature, and is denoted by $x_i$.
The model is assumed to have learnable parameters ("weights")
$\bfw \in \R^d$, and for a given data point $(\bfx, y)$,
the \textit{loss} is given by some function  $\lossxyw$. 
\textit{Natural model training}
\footnote{Also referred to as \textit{standard training} by \cite{Madry2017-zz}}
 consists of minimizing the expected loss, known as \textit{empirical risk}:
\begin{equation} \label{eq-nat-loss}
	\E_{(\bfx,y) \sim \calD}[ \lossxyw ].
\end{equation}
We sometimes assume the existence of an $\linfeps$-\textit{adversary}  who may perturb the input example $\bfx$ by adding a vector $\delta \in \R^d$ whose $\linf$-norm is bounded by $\varepsilon$; such a perturbation $\delta$ is referred to as an $\linfeps$-perturbation.
For a given data point $(\bfx, y)$ and a given loss function $\loss(.)$,
an $\linfeps$-\textit{adversarial} perturbation is a $\delta^*$ that maximizes the \textit{adversarial loss} $\loss(\bfx + \delta^*, y; \bfw)$.

\newcommand\aifxu{A_i^F(\bfx, \bfu)}
\newcommand\axu{A_i(\bfx, \bfu)}
\newcommand\afxu{A^F(\bfx, \bfu)}

\newcommand\igifxu{\IG_i^F(\bfx, \bfu)}
\newcommand\igxu{\IG_i(\bfx, \bfu)}
\newcommand\igfxu{\IG^F(\bfx, \bfu)}

\newcommand\shifxu{\SH_i^F(\bfx, \bfu)}
\newcommand\shxu{\SH_i(\bfx, \bfu)}
\newcommand\shfxu{\SH^F(\bfx, \bfu)}

Given a function $F: \Real^d \rightarrow [0,1]$ representing a neural network,
an input vector $\bfx \in \Real^d$, 
and a suitable baseline vector $\bfu \in \Real^d$, 
an \textit{attribution} of the prediction of $F$ at input $\bfx$ relative to $\bfu$ is a vector
$A^F(\bfx, \bfu) \in \Real^d$ whose $i$'th component
$A_i^F(\bfx, \bfu)$
represents the "contribution" of $x_i$ to the prediction $F(\bfx)$. 
A variety of \textit{attribution methods} have been proposed in the literature (see \cite{Arya2019-zy} for a survey),
but in this paper we will focus on two of the most popular ones: 
Integrated Gradients \cite{Sundararajan2017-rz},
and DeepSHAP \cite{lundberg2017unified}.
When discussing a specific attribution method, 
we will denote the IG-based attribution vector as $\igfxu$,
and the DeepSHAP-based attribution vector as $\shfxu$. In all cases we will drop the superscript $F$ and/or the baseline vector $\bfu$ when those are clear from the context.

The aim of \textit{adversarial training}\cite{Madry2017-zz} is to train a model that is \textit{robust} to an $\linfeps$-adversary (i.e. performs well in the presence of such an adversary),
and consists of minimizing the expected  $\linfeps$-adversarial loss:
\begin{equation} \label{eq-adv-loss}
	\E_{(\bfx,y) \sim \calD}
	    [ \max_{\deltainf \leq \varepsilon} 
	    \loss(\bfx + \delta, y; \bfw)
	    ].
\end{equation}
In the expectations (\ref{eq-nat-loss}) and (\ref{eq-adv-loss}) we often drop the subscript under $\E$ when it is clear that the expectation is over $(\bfx, y) \sim \calD$.

Some of our theoretical results make assumptions regarding the form and properties of the loss function $\mathcal{L}$,  the properties of its first derivative.
For the sake of clarify, we highlight these assumptions (with mnemonic names) here for ease of future reference. 

\newtheorem{innercustomass}{Assumption}
\newenvironment{ass}[1]
  {\renewcommand\theinnercustomass{#1}\innercustomass}
  {\endinnercustomass}

\newtheorem{innercustomdef}{Definition}
\newenvironment{defn}[1]
  {\renewcommand\theinnercustomdef{#1}\innercustomdef}
  {\endinnercustomdef}
 
\begin{ass}{LOSS-INC} \label{ass-loss-non-dec}
The loss function is of the form $\lossxyw = g(-y\langle \bfw, \bfx \rangle)$ where $g$ is a \textit{non-decreasing} function.
\end{ass}

\begin{ass}{LOSS-CVX} \label{ass-loss-cvx}
The loss function is of the form $\lossxyw = g(-y\langle \bfw, \bfx \rangle)$ where $g$ is \textit{non-decreasing}, almost-everywhere differentiable, and convex.
\end{ass}

Section \ref{sec-app-loss-fn} in the Supplement shows that these Assumptions are satisfied by popular loss functions such as logistic and hinge loss.
Incidentally, note that for any differentiable function $g$, $g$ is convex if and only if its first-derivative $g'$ is non-decreasing, 
and we will use this property in some of the proofs.


\begin{ass}{FEAT-TRANS}
\label{ass-feat-exp}
For each $i \in [d]$, if $x'_i$ is the feature in the original dataset, 
$x_i$ is the \textbf{translated} version of $x'_i$ defined by 
$x_i = x'_i - [ E(x'_i | y=1) + E(x'_i | y=-1) ] / 2$.
\end{ass}

In Section \ref{sec-app-ass-feat-exp} (Supplement) we show that 
this mild assumption implies that for each feature $x_i$ there is a constant $a_i$ such that
\begin{align} 
\E(x_i | y ) &= a_i y  \label{eq-prop-exp} \\
\E(yx_i) &= \E[ \E(yx_i|y) ] = \E[ y^2 a_i ] = a_i \label{eq-eyxi} \\
\E(yx_i|y) &= y\E[ x_i | y ] = y^2 a_i = a_i \label{eq-eyxigiveny} 
\end{align}
For any $i \in [d]$, we can think of $\E(yx_i)$ as 
the \textit{degree of association}\footnote{When the features are standardized to have mean 0, $\E(y x_i)$ is in fact the covariance of $y$ and $x_i$.}
 between feature $x_i$ and label $y$.
Since $\E(yx_i) = a_i$ (Eq. \ref{eq-eyxi}), 
 we refer to $a_i$ as the \textit{directed strength}\footnote{This is related to the feature "robustness" notion introduced in \cite{Ilyas2019-kp}}
  of feature $x_i$, and $|a_i|$ is referred to as the \textit{absolute strength} of $x_i$. 
  In particular when $|a_i|$ is large (small) we say that $x_i$ is a \textit{strong} (\textit{weak}) feature.  
  
\subsection{Averaging over a group of features}  \label{sec-av}
For our main theoretical result (Theorem \ref{thm-exp-sgd}), we need a notion of weighted average defined as follows:

\newcommand\barqws{q^{\w}_S}
\newcommand\baraws{a^{\w}_S}
\newcommand\barnablaws{\nabla^{\w}_S}
\begin{defn}{WTD-AV}
Given a quantity $q_i$ defined for each feature-index $i \in [d]$, 
a subset $S \subset [d]$ (where $|S| \geq 1$) of feature-indices, and 
a feature weight-vector $\w$ with $w_i \neq 0$ for at least one $i \in S$, 
the \textbf{$\w$-weighted average of $q$ over $S$} is defined as 
\begin{equation} \label{eq-wtd-av}
	\barqws := 
 	\frac{\sum_{i \in S} w_i q_i}
 	     {\sum_{i \in S} |w_i|}       
\end{equation}
\end{defn}

Note that the quantity $w_i q_i$ can be written as $|w_i| \sgn(w_i) q_i$, so 
$\barqws$ is essentially a $|w_i|$-weighted average of $\sgn(w_i) q_i$ over $i \in S$. 

For our result we will use the above $\w$-weighted average definition for two particular quantities $q_i$. The first one is $q_i := a_i$, the directed strength of feature $x_i$ (Eq. \ref{eq-eyxi}). 
Intuitively, the quantity $\sgn(w_i) \E(y x_i) = a_i \sgn(w_i) $ captures 
the \textit{aligned strength} of feature $x_i$ in the following sense: 
if this quantity is large and positive (large and negative), it indicates both that 
the current weight $w_i$ of $x_i$ is aligned (misaligned) with the directed strength of $x_i$, 
\textit{and} that this directed strength is large. 
Thus $\baraws$ represents an average of the aligned strength over the feature-group $S$.

\newcommand\delbar{\overline{\Delta}}
\newcommand\delbarws{\overline{\Delta}^{\w}_S}
The second quantity for which we define the above $\w$-weighted average is
$q_i := \delbar_i$, where $\delbar_i := - \E [\partial \loss / \partial w_i]$
represents the \textit{expected} SGD update (over random draws from the data distribution)
of the weight $w_i$, given the loss function $\loss$, for a unit learning rate
(details are in the next Section).
The quantity $\sgn(w_i) \delbar_i$ has a natural interpretation, analogous to the above interpretation of $a_i \sgn(w_i)$: 
a large positive (large negative) value of 
$\sgn(w_i) \delbar_i$ corresponds to an \textit{large expansion} (\textit{large shrinkage}), in expectation,  
of the weight $w_i$ away from zero magnitude (toward zero magnitude).
Thus the $\w$-weighted average $\delbarws$ represents the $|w_i|$-weighted average of this effect over the feature-group $S$.


\section{Analysis of SGD Updates in Adversarial Training}
\label{sec-sgd-updates}
One way to understand the characteristics of the weights in an adversarially-trained neural network model,
is to analyze how the weights  evolve during adversarial training 
under Stochastic Gradient Descent (SGD) optimization. 
One of the main results of this work is a theoretical characterization of the weight updates during 
a single SGD step, when applied to a
randomly drawn data point $(\bfx, y) \sim \calD$ 
that is subjected to an $\linfeps$-adversarial perturbation.

Although the holy grail would be to do this for general DNNs
(and we expect this will be quite difficult)
we take a first step in this direction by analyzing 
\textit{single-layer networks} for binary or multi-class classification, 
 where each weight is associated with an input feature.
Intriguingly, our results (Theorem \ref{thm-exp-sgd} for binary classification and \ref{thm:multi_class} for multi-class classification in the Supplement) show that for these models, 
$\linfeps$-adversarial training tends to selectively reduce the weight-magnitude of \textit{weakly relevant} 
or \textit{irrelevant} features, and 
does so much more aggressively than natural training.
In other words, natural training can result in models where many weak features have significant weights, 
whereas adversarial training would tend to push most of these weights close to zero. The resulting model weights would thus be more sparse, 
and the corresponding IG-based attribution vectors would on average be more sparse as well
(since in linear models,  sparse  weights imply sparse IG vectors; 
this is a consequence of Lemma \ref{lem-1-layer-attr})
compared to naturally-trained models.

Our experiments (Sec. \ref{sec-exp}) show  that indeed
for logistic regression models 
(which satisfy the conditions of Theorem \ref{thm-exp-sgd}),
adversarial training leads to sparse IG vectors.
Interestingly, our extensive experiments with Deep Convolutional Neural Networks 
on public image datasets demonstrate that 
 this phenomenon extends to DNNs as well, and to attributions based on DeepSHAP, even though our theoretical results only apply to 1-layer networks and IG-based attributions.

As a preliminary, it is easy to show the following expressions related to the $\linfeps$-adversarial perturbation $\delta^*$
(See Lemmas \ref{lem-adv-closed-form} and \ref{lem-grad-loss-adv} 
in Section \ref{sec-app-loss-grad} of the Supplement):
For loss functions satisfying Assumption
\ref{ass-loss-non-dec}, the $\linfeps$-adversarial perturbation $\delta^*$ is given by:
\begin{equation} \label{eq-adv-closed-form}
\delta^* = -y \sign(\bfw) \varepsilon,
\end{equation}
the corresponding $\linfeps$-adversarial loss is 
\beq \label{eq-linfeps-adv-loss}
\loss(\bfx + \delta^*, y; \; \bfw) = 
    g(\varepsilon ||\bfw||_1 - y \wdotx),
\eeq
and the gradient of this loss w.r.t. a weight $w_i$ is 
\begin{multline}
	\label{eq-grad-adv}
\frac{\partial \loss(\bfx + \delta^*, y; \; \bfw)}{\partial w_i}
           =  \\
-g'(\varepsilon ||\bfw||_1 - y \langle \bfw, \bfx\rangle) \;
	   (y x_i - \sgn(w_i)\varepsilon). 
\end{multline}
In our main result, the expectation of the $g'$ term in (\ref{eq-grad-adv}) plays an important role, so we will use the following notation:
\begin{equation}
	\label{eq-gprimebar}
	\gprimebar := \E \big[ 
	        g'(\varepsilon ||\bfw||_1 - 
	             y \langle \bfw, \bfx\rangle) 
	              \big],
\end{equation}
and by Assumption \ref{ass-loss-non-dec},  \textbf{$\gprimebar$ is non-negative}.

Ideally, we would like to understand the nature of the weight-vector $\bfw^*$ that minimizes the expected adversarial loss 
(\ref{eq-adv-loss}). 
This is quite challenging, 
so rather than analyzing the \textit{final} optimum of (\ref{eq-adv-loss}), 
we instead analyze how an SGD-based optimizer for 
(\ref{eq-adv-loss}) \textit{updates} the model weights $\bfw$.
We assume an idealized SGD process:  
(a) a data point $(\bfx, y)$ is drawn  from distribution $\calD$,
(b) $\bfx$ is replaced by $\bfx' = \bfx + \delta^*$
where $\delta^*$ is an $\linfeps$-adversarial perturbation 
with respect to the loss function $\loss$,
(c) each weight $w_i$ is updated by an amount 
$\deltawi = -\partial \loss(\bfx', y; \bfw) / \partial w_i$
(assuming a unit learning rate to avoid notational clutter).
We are interested in the \textit{expectation} 
$\delbar_i := \E \deltawi = -\E[ \partial \loss(\bfx', y; \w)/\partial w_i]$, 
in order to understand how a weight $w_i$ evolves \textit{on average} during a single SGD step.
Where there is a \textit{conditionally independent} feature subset $S \in [d]$
(i.e. the features in $S$ are conditionally independent of the rest given the label $y$),
our main theoretical result characterizes the behavior of $\delbar_i$ for $i \in S$,
and the corresponding $\w$-weighted average $\delbarws$:



\begin{restatable}[Expected SGD Update in Adversarial Training]{thm}{thmsgd} 
\label{thm-exp-sgd}
For any loss function $\loss$ satisfying Assumption \ref{ass-loss-cvx},
a dataset $\calD$ satisfying Assumption \ref{ass-feat-exp},
a subset $S$ of features that are conditionally independent of the rest given the label $y$,
if a data point $(\bfx, y)$ is randomly drawn from $\calD$, 
and $\bfx$ is perturbed to $\bfx' = \bfx + \delta^*$, where $\delta^*$ is an $\linfeps$-adversarial perturbation, 
then during SGD using the $\linfeps$-adversarial loss $\loss(\bfx', y; \bfw)$, 
the expected weight-updates $\delbar_i := \E \Delta w_i$ for $i \in S$ 
and the corresponding $\w$-weighted average $\delbarws$ satisfy the following properties:
\begin{enumerate}
	\item If $w_i = 0\; \forall i \in S$, then for each $i \in S$,
\begin{equation} \label{eq-thm-wzero}
	\delbar_i  = \gprimebar \; a_i,
\end{equation} 
	\item and otherwise, 
\begin{equation} \label{eq-thm-wnonzero}
	\delbarws \leq \gprimebar (\baraws - \varepsilon) ,
\end{equation}
and equality holds in the limit as $w_i \rightarrow 0\; \forall i \in S$,
\end{enumerate}
where 
$\gprimebar$ is the expectation in (\ref{eq-gprimebar}),
$a_i = \E(x_i y)$ is the directed strength of feature $x_i$ 
from Eq. (\ref{eq-eyxi}), and $\baraws$ is the corresponding 
$\w$-weighted average over $S$.
\end{restatable}

For space reasons, a detailed discussion of the implications of this result is presented in 
 Sec. \ref{sec-implication} of the Supplement, but here we note the following.
 Recalling the interpretation of the $\w$-weighted averages $\baraws$ and $\delbarws$ in Section \ref{sec-av},
%
we can interpret the above result as follows. 
 For any conditionally independent feature subset $S$, 
 if the weights of all features in $S$ are zero, 
 then by Eq. (\ref{eq-thm-wzero}), an SGD update causes, on average, 
 each of these weights $w_i$ to grow (from 0) in a direction consistent with the directed feature-strength $a_i$ (since $\gprimebar \geq 0$ as noted above). 
 If at least one of the features in $S$ has a non-zero weight,  
 (\ref{eq-thm-wnonzero}) implies $\delbarws < 0$, i.e., an aggregate shrinkage of the weights of features in $S$, if either of the following hold: (a) 
 $\baraws < 0$, i.e., the weights of features in $S$ are mis-aligned on average,  or 
(b) the weights of features in $S$ are aligned on average, i.e., $\baraws$ is positive, but dominated by $\varepsilon$, i.e. the features $S$ are \textit{weakly correlated} with the label. In the latter case the weights of features in $S$ are (in aggregate and in expectation) aggressively pushed toward zero, 
and this aggressiveness is proportional to the extent to which $\varepsilon$ dominates 
$\baraws$.
A partial generalization of the above result for the multi-class setting 
(for a single conditionally-independent feature)
is presented in Section \ref{sec:multi-class} (Theorem \ref{thm:multi_class}) of the Supplement.



\section{\uppercase{Feature Attribution using Integrated Gradients}}
\label{sec-attr}
Theorem \ref{thm-exp-sgd} showed that $\linfeps$-adversarial training tends to shrink the \textit{weights} of features that are "weak"
(relative to $\varepsilon$).
We now show a link between weights and \textit{explanations},
specifically explanations in the form of a vector of feature-attributions given by the \textit{Integrated Gradients} (IG) method \cite{Sundararajan2017-rz}, which is defined as follows:
Suppose $F: \mathbb{R}^d \rightarrow \mathbb{R}$ is a real-valued function of an input vector. For example $F$ could represent the output of a neural network, or even a loss function
$\loss(\bfx, y; \bfw)$ when the label $y$ and weights $\bfw$ are held fixed.
Let $\bfx \in \R^d$ be a specific input, and 
$\bfu \in \mathbb{R}^d$ be a baseline input. 
The IG is defined as the path integral of the gradients along the straight-line path from the baseline $\bfu$ to the input $\bfx$. The IG along the $i$'th dimension for an input $\bfx$ and baseline $\bfu$ is defined as:
\begin{equation}
	\igifxu := (x_i - u_i) \times 
	\int_{\alpha=0}^1 \partial_i F(\bfu + \alpha(\bfx- \bfu)) d\alpha,
	\label{eq-ig-def}
\end{equation}
where $\partial_i F(\bfz)$ denotes the gradient of $F(\bfv)$ along the $i$'th dimension, at $\bfv=\bfz$. The vector of all IG components $\igifxu$ is denoted as $\igfxu$.
Although we do not show $\bfw$ explicitly as an argument in 
the notation $\igfxu$, it should be understood that the IG depends on the model weights $\bfw$ since the function $F$ depends on $\bfw$.

The following Lemma (proved in Sec. \ref{app-lem-ig} of the Supplement)
shows a closed form exact expression for the $\igfxu$ when $F(\bfx)$ is of the form 
\begin{equation} \label{eq-1-layer-fn}
		F(\bfx) = A(\<\bfw, \bfx>),
\end{equation}
where  
$\bfw \in \R^d$ is a vector of weights, 
  $A$ is a  differentiable scalar-valued function, and 
$\<\bfw, \bfx>$ denotes the dot product of $\bfw$ and $x$. 
Note that this form of $F$ could represent a single-layer neural network 
with any differentiable activation function
(e.g., logistic (sigmoid) activation $A(\bfz) = 1/[1 + \exp(-\bfz)]$ or 
Poisson activation $A(\bfz) = \exp(\bfz)$),
or a differentiable loss function, such as those that satisfy Assumption
\ref{ass-loss-non-dec} for a fixed label $y$ 
and weight-vector $\bfw$.
For brevity, we will refer to a
function of the form (\ref{eq-1-layer-fn}) as representing a
"1-Layer Network", with the understanding that it could equally well represent a suitable loss function.

\begin{restatable}[IG Attribution for 1-layer Networks]{lem}{lemig} \label{lem-1-layer-attr}
	If  $F(\bfx)$ is computed by a 1-layer  network (\ref{eq-1-layer-fn}) with weights vector $\bfw$,
then the Integrated Gradients for all dimensions of $\bfx$ relative to a baseline $\bfu$ are given by:
\begin{equation}
	\igfxu = [ F(\bfx) - F(\bfu) ] \frac{(\bfx - \bfu) \odot \bfw }{ \< \bfx - \bfu, \bfw >}, \label{eq-ig-1-layer}
\end{equation}
where the $\odot$ operator denotes the entry-wise product of vectors.
\end{restatable}

Thus for 1-layer networks, the IG of each feature is essentially proportional to the feature's fractional contribution to the 
logit-change $\< \bfx - \bfu, \bfw >$.
This makes it clear that in such models, if the weight-vector $\bfw$ is sparse, then the IG vector will also be correspondingly sparse.
%

%
%


\section{Training with Explanation Stability is equivalent to  Adversarial Training} \label{sec-stable}

Suppose we use the IG method described in Sec. \ref{sec-attr} 
as an explanation for the output of a model $F(\bfx)$ on a specific input $\bfx$.
A desirable property of an explainable model is that the explanation
for the value of $F(\bfx)$ is \textit{stable}\cite{Alvarez-Melis2018-qr}, i.e., does not change much under small perturbations of the input $\bfx$.
One way to formalize this is to say the  following \textit{worst-case} $\ell_1$-norm of the change in $\IG$  should be small:
\beq \label{eq-max-ig-change}
\max_{\bfx' \in N(\bfx, \varepsilon)}
||\IG^F(\bfx', \bfu) -  \IG^F(\bfx, \bfu)||_1,
\eeq
where $N(\bfx, \varepsilon)$ denotes a suitable $\varepsilon$-neighborhood of $\bfx$, and $\bfu$ is an appropriate 
baseline input vector.
If the model $F$ is a single-layer neural network,
it would be a function of $\<\bfw, \bfx>$ for some weights $\bfw$,
and typically when training such networks the loss  
is a function of $\<\bfw, \bfx>$ as well,
so we would not change the essence of (\ref{eq-max-ig-change})  much if instead of $F$ in each IG, we use $\lossxyw$ for a fixed $y$; let us denote this function by $\loss_y$.
\newcommand\lossy{\loss_y}
Also intuitively, 
$||\IG^{\lossy}(\bfx', \bfu) - \IG^{\lossy}(\bfx, \bfu)||_1$
is not too different from 
$||\IG^{\lossy}(\bfx', \bfx)||_1$.
These observations motivate the following definition of \textit{Stable-IG Empirical Risk}, which is a modification of the usual empirical risk 
(\ref{eq-nat-loss}), with a regularizer to encourage stable IG explanations:
\begin{multline} \label{eq-nat-loss-stable}
	\E_{(\bfx,y) \sim \calD}
	\Big[ \;
	\lossxyw \;+\; \\
	  \max_{||\bfx' - \bfx||_\infty \leq \varepsilon}
	     || \IG^{\lossy}(\bfx, \bfx') ||_1 	 
	 \; \Big].
\end{multline}

The following somewhat surprising result is proved in Section \ref{app-thm-stable} of the Supplement.

\begin{restatable}[Equivalence of Stable IG and Adversarial Robustness]{thm}{thmstable} \label{thm-stable}
For loss functions $\lossxyw$ 
satisfying Assumption \ref{ass-loss-cvx},
the augmented loss inside the expectation 
(\ref{eq-nat-loss-stable}) equals the 
$\linfeps$-adversarial loss inside the expectation 
(\ref{eq-adv-loss}), i.e.
\begin{multline}
 \label{eq-stable-robust}
	\lossxyw \;+\;
	 \max_{||\bfx' - \bfx||_\infty \leq \varepsilon}
	     || \IG^{\lossy}(\bfx, \bfx') ||_1 
 \;\;=   \\
\max_{\deltainf \leq \varepsilon} 
	    \loss(\bfx + \delta, y; \bfw) 
\end{multline}
\end{restatable}
This implies that for loss functions satisfying Assumption 
\ref{ass-loss-cvx},  minimizing the Stable-IG Empirical Risk 
(\ref{eq-nat-loss-stable}) is equivalent to 
minimizing the expected $\linfeps$-adversarial loss.
In other words, for this class of loss functions, 
\textit{natural model training while encouraging IG stability 
is equivalent to $\linfeps$-adversarial training!
}
Combined with Theorem \ref{thm-exp-sgd} and the corresponding experimental results in Sec \ref{sec-exp}, this equivalence implies that, for this class of loss functions, and data distributions satisfying Assumption \ref{ass-feat-exp}, the explanations for the models produced by $\linfeps$-adversarial training are both \textit{concise} (due to the sparseness of the models), and \textit{stable}.

\section{Experiments}
\label{sec-exp}

\subsection{Hypotheses}
Recall that one implication of Theorem \ref{thm-exp-sgd} is the following:
For 1-layer networks where the loss function satisfies Assumption 
\ref{ass-loss-cvx}, 
$\linfeps$-adversarial training tends to more-aggressively prune the \textit{weight-magnitudes} of "weak" features 
compared to natural training. 
In Sec. \ref{sec-attr} we observed that a consequence of Lemma \ref{lem-1-layer-attr} is that for 1-layer models the sparseness of the weight vector implies 
sparseness of the IG vector.
Thus a reasonable conjecture is that, for 1-layer networks, $\linfeps$-adversarial training leads to models with sparse attribution vectors in general (whether using IG or a different method, such as DeepSHAP).
We further conjecture that this sparsification phenomenon extends to practical 
multi-layer Deep Neural Networks, not just 1-layer networks, and that this benefit can be realized without significantly impacting accuracy on natural test data.
Finally, we hypothesize that the resulting sparseness  of \textit{attribution vectors} 
is better than what can be achieved by a traditional \textit{weight  regularization} 
technique such as L1-regularization, for a comparable level of natural test accuracy.

\subsection{Measuring Sparseness of an Attribution Vector}
For an attribution method $A$, we quantify the sparseness of the attribution vector $\afxu$ using the 
\textit{Gini Index} applied to the vector of absolute values $\afxu$.
For a vector $\bfv$ of non-negative values, the Gini Index, denoted $G(\bfv)$
(defined formally in Sec. \ref{app-gini} in the Supplement),
is a metric for sparseness of $\bfv$ that is known \cite{Hurley2009-iw} to satisfy a number of desirable properties,
and has been used to quantify sparseness of weights in a neural network \cite{Guest2017-tf}.
The Gini Index by definition lies in [0,1], and a higher value indicates more sparseness.

Since the model $F$ is clear from the context, and the baseline vector $\bfu$ are fixed for a given dataset, 
we will denote the attribution vector 
on input $\bfx$ simply as $A(\bfx)$, 
and our measure of sparseness is $G(| A(\bfx)|)$, 
which we denote for brevity as $G[A](\bfx)$, and refer to informally as the \textit{Gini of $A$}, 
where $A$ can stand for $\IG$ (when using IG-based attributions) or $\SH$ (when using DeepSHAP for attribution).
As mentioned above, one of our hypotheses is that the   sparseness of attributions of models produced by $\linfeps$-adversarial training is much better than what can be achieved by natural training using $\ell_1$-regularization, for a comparable level of accuracy. 
To verify this hypothesis we will compare the sparseness of attribution vectors resulting from three types of models:
(a) n-model: \textit{naturally-trained} model with no adversarial perturbations and no $\ell_1$-regularization,
(b) a-model: \textit{$\linfeps$-adversarially trained} model, 
and
(c) l-model: naturally trained model with \textit{$\ell_1$-regularization} strength $\lambda > 0$.
For an attribution method $A$, we denote the  Gini indices $G[A](\bfx)$ resulting from these models respectively as 
$G^n[A](\bfx)$, 
$G^a[A](\bfx; \; \varepsilon)$ and 
$G^l[A](\bfx; \; \lambda)$.

In several of our datasets, individual feature vectors are already quite sparse: for example in the MNIST dataset, most of the area consists of black pixels, and in the Mushroom dataset, after 1-hot encoding the 22 categorical features, the resulting 120-dimensional feature-vector is sparse.
On such datasets, even an n-model can achieve a "good" level of sparseness of attributions in the absolute sense, i.e. $G^n[A](\bfx)$ 
can be quite high.
Therefore for all datasets we compare the 
\textit{sparseness improvement}
resulting from an a-model relative to an n-model,
with that from an l-model relative to a n-model.
Or more precisely, we will compare the two quantities defined below, for a given attribution method A:
\begin{align} 
	dG^a[A](\bfx; \; \varepsilon) &:= G^a[A](\bfx; \; \varepsilon) - G^n[A](\bfx), 
	\label{eq-dgiga}
	\\
    dG^l[A](\bfx; \; \lambda) &:= G^l[A](\bfx; \; \lambda) - G^n[A](\bfx).
    	\label{eq-dgigl}
\end{align}    
The above quantities define the IG sparseness improvements for a \textit{single} example $\bfx$. 
It will be convenient to define the overall sparseness improvement from a model, as measured on a test dataset, by averaging over all examples $\bfx$ in that dataset. 
We denote the corresponding \textit{average} sparseness metrics by $G^a[A](\varepsilon)$, 
$G^l[A](\lambda)$ and $G^n[A]$ respectively.
We then define the \textit{average sparseness improvement} of an a-model and l-model as:
\begin{align} 
	dG^a[A](\varepsilon) &:= 
	G^a[A](\varepsilon) - G^n[A], 
	\label{eq-d-av-giga}
	\\
    dG^l[A](\lambda) &:= 
    G^l[A](\lambda) - G^n[A].
    	\label{eq-d-av-gigl}
\end{align}    
We can thus re-state our hypotheses in terms of this notation:
For each of the attribution methods $A \in \{IG, SH\}$, 
the average sparseness improvement $dG^a[A](\varepsilon)$ resulting from
type-a models is high, and is significantly higher than
the average sparseness improvement $dG^l[A](\lambda)$ resulting from type-l models.

\subsection{Results}

We ran experiments on five standard public benchmark datasets: three image datasets MNIST, Fashion-MNIST, and CIFAR-10, and two tabular datasets from the UCI Data Repository: Mushroom and Spambase.
Details of the datasets and training methodology are in Sec. \ref{sec-datasets} of the Supplement.
The code for all experiments is at this repository: \texttt{https://github.com/jfc43/advex}.

For each of the two tabular datasets (where we train logistic regression models), for a given model-type (a, l or n), we found the average Gini index of the attribution vectors is virtually identical when using IG or DeepSHAP.
This is not surprising: as pointed out in \cite{Ancona2017-tw}, DeepSHAP is a variant of DeepLIFT, and for simple linear models, DeepLIFT gives a very close approximation of IG.
To avoid clutter, we therefore omit DeepSHAP-based results on the tabular datasets.
Table \ref{tab-summary} shows a summary of some results on the above 5 datasets,
and Fig. \ref{fig-boxplots} and \ref{fig-varplots} display results graphically
\footnote{The official implementation of DeepSHAP (\url{https://github.com/slundberg/shap}) doesn't support the network we use for CIFAR-10 well, so we do not show results for this combination.}

\begin{table}[h]
	\caption{Results on 5 datasets. 
		For each dataset, "a" indicates an $\linfeps$-adversarially trained model with the indicated $\varepsilon$,
		and 
		"l" indicates a naturally trained model with 
		the indicated $\ell_1$-regularization strength $\lambda$. 
		The \textbf{attr} column indicates the feature attribution method (IG or DeepSHAP).
		Column \textbf{dG} shows the average sparseness improvements  
		of the models
		relative to the baseline naturally trained model,
		as measured by the $dG^a[A](\varepsilon)$ and $dG^l[A](\lambda)$ 
		defined in Eqs. (\ref{eq-d-av-giga}, \ref{eq-d-av-gigl}).
		Column \textbf{AcDrop} indicates the drop in accuracy relative to the baseline model.
	}
	\label{tab-summary}
	\begin{tabular}{@{}lllll@{}}
		\toprule
		\textbf{dataset} & \textbf{attr} & \textbf{model}        & \textbf{dG} & \textbf{AcDrop} \\ \midrule
		MNIST           & IG & a $(\varepsilon=0.3)$ & 0.06          & 0.8\%            \\
		& IG & l $(\lambda=0.01)$     & 0.004         & 0.4\%            \\
		& SHAP & a $(\varepsilon=0.3)$     & 0.06         & 0.8\%            \\ 
		& SHAP & l $(\lambda=0.01)$     & 0.007         & 0.4\%            \\                                 
		\midrule
		Fashion & IG &  a $(\varepsilon=0.1)$ & 0.06          & 4.7\%            \\
		 -MNIST             & IG   & l $(\lambda=0.01)$    & 0.008         & 3.4\%            \\ 
					  & SHAP   & a $(\varepsilon=0.1)$    & 0.08         & 4.7\%          \\            
		              & SHAP  & l $(\lambda=0.01)$    & 
		                  0.003         & 3.4\%            \\ 
		
		                 \midrule
		CIFAR-10 & IG &  a $(\varepsilon=1.0)$ &    0.081    & 0.57\%            \\
		& IG   & l $(\lambda=10^{-5})$    &    0.022    &  1.51\%            \\ 
		\midrule
		Mushroom         
		& IG & a $(\varepsilon=0.1)$ & 0.06          & 2.5\%            \\
		& IG &  l $(\lambda=0.02)$    & 0.06          & 2.6\%            \\ 
		
		\midrule
		Spambase         
		& IG & a $(\varepsilon=0.1)$ & 0.17          & 0.9\%            \\
		& IG & l $(\lambda=0.02)$     & 0.15         & 0.1\%            \\ 
		
		\bottomrule
	\end{tabular}
\end{table}

\begin{figure*}[h]
\centering
\includegraphics[scale=0.7]{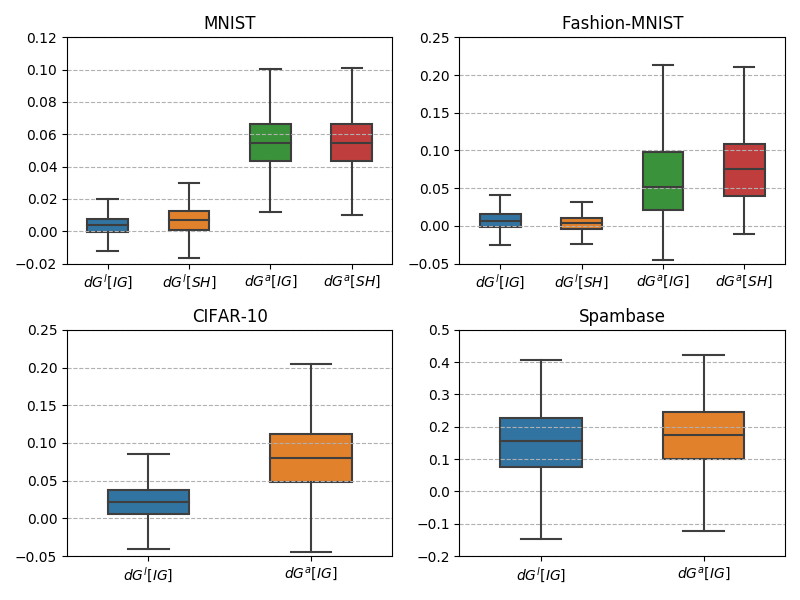}
\caption{Boxplot of pointwise sparseness-improvements 
from adversarially trained models ($dG^a[A](\bfx, \varepsilon)$)
and  naturally trained models with $\ell_1$-regularization 
($dG^l[A](\bfx, \lambda)$),
for attribution methods $A \in \{IG,SH\}$.
} 
\label{fig-boxplots}
\end{figure*}
\begin{figure*}[!htp]
\centering
\includegraphics[scale=0.8]{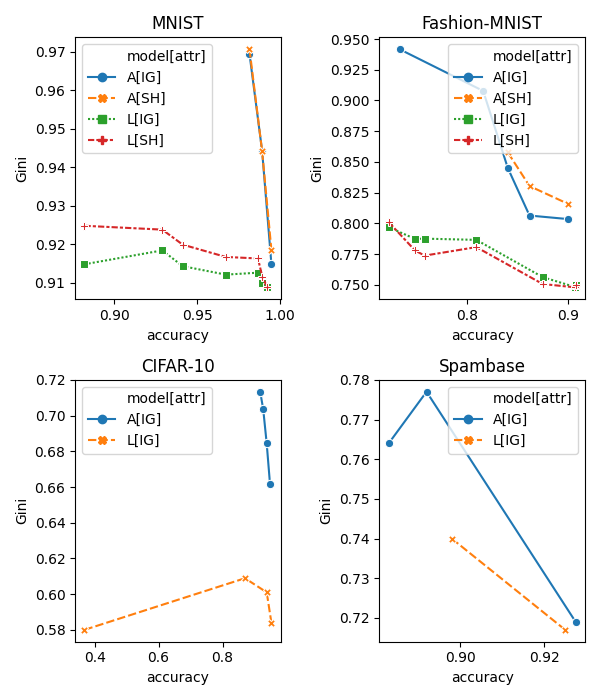}
\caption{
For four benchmark datasets, each line-plot labeled $M[A]$ shows the tradeoff between Accuracy and  
Attribution Sparseness achievable by various combinations of models $M$ and attribution methods $A$.
$M=A$ denotes $\linfeps$-adversarial training, and the plot shows the accuracy/sparseness for various choices of $\varepsilon$. 
$M=L$ denotes $\ell_1$-regularized natural training, and the plot shows 
accuracy/sparseness for various choices of $\ell_1$-regularization parameter $\lambda$.
$A=IG$ denotes the IG attribution method, whereas $A=SH$ denotes DeepSHAP.
Attribution sparseness is measured by the Average Gini Index over the dataset
($G^a[A](\varepsilon)$ and $G^l[A](\lambda)$, for adversarial training and $\ell_1$-regularized natural training respectively). In all cases, and especially in the image datasets, 
adversarial training achieves a significantly better accuracy vs attribution sparseness tradeoff.
}
\label{fig-varplots}
\end{figure*}

These results make it clear that for comparable levels of accuracy, the 
sparseness of attribution vectors from $\linfeps$-adversarially trained models 
is much better than the sparseness from natural training with 
$\ell_1$-regularization. The effect is especially pronounced in the
two image datasets. 
The effect is less dramatic in the two tabular datasets,
for which we train logistic regression models.
Our discussion at the end of Sec. \ref{sec-attr} suggests a possible explanation.

In the Introduction we gave an example of a \textit{saliency map}
\cite{DBLP:journals/corr/SimonyanVZ13, Baehrens2010-af}
 (Fig. \ref{fig-intro-cifar10})
to dramatically highlight the sparseness induced by adversarial training. 
We show several more examples of saliency maps in the supplement  (Section \ref{sec-app-images}).

\section{Related Work} \label{sec-related}
In contrast to the growing body of work on defenses against 
adversarial attacks 
\cite{Yuan2017-vz, Madry2017-zz, Biggio2013-rx}
or explaining adversarial examples 
\cite{Goodfellow2014-cy, Tsipras2018-fu},
the focus of our paper is the connection between 
adversarial robustness and explainability. 
We view the process of adversarial training as a tool to produce more explainable models.
A recent series of papers \cite{Tsipras2018-fu, Ilyas2019-kp}
essentially argues that adversarial examples exist because 
standard training produces models are heavily reliant on 
highly predictive but \textit{non-robust features}
(which is similar to our notion of "weak" features 
in Sec \ref{sec-sgd-updates})
which are vulnerable  to an adversary who can "flip" them and cause performance to degrade. 
Indeed the authors of \cite{Ilyas2019-kp} touch upon some connections between explainability and robustness, 
and conclude, "\textit{As such, producing human-meaningful explanations that remain faithful to underlying models cannot be pursued independently from the training of the models themselves}", 
by which they are implying that 
good explainability may require intervening 
in the model-training procedure itself; 
this is consistent with our findings.
We discuss other related work in the Supplement Section \ref{sec-app-related}.

\section{Conclusion} \label{sec-conc}

We presented theoretical and experimental results that show a strong connection between adversarial robustness (under $\linf$-bounded perturbations) and two desirable properties of model explanations: conciseness and stability.
Specifically, we considered model explanations in the form of feature-attributions based on the Integrated Gradients (IG) and DeepSHAP techniques.
For 1-layer models using a popular family of loss functions, we theoretically showed that 
$\linfeps$-adversarial training tends to produce \textit{sparse} and \textit{stable}
IG-based attribution vectors.
With extensive experiments on benchmark tabular and image datasets, we demonstrated 
that the "attribution sparsification" effect extends to Deep Neural Networks, when using two popular attribution methods.
Intriguingly, especially in DNN models for image classification, the attribution sparseness from
natural training with $\ell_1$-regularization is much inferior 
to that achievable via $\linfeps$-adversarial training.
Our theoretical results are  a first step in explaining some of these phenomena.


\newpage


\bibliography{paper.bib}
\bibliographystyle{icml2020}

\newpage

\appendix

\begin{appendices}
\renewcommand{\theequation}{\thesection.\arabic{equation}}
\renewcommand\thefigure{\thesection.\arabic{figure}}  
\renewcommand\thetable{\thesection.\arabic{table}}  

\newpage
\onecolumn

\section{Additional Related Work} \label{sec-app-related}

Section \ref{sec-related} discussed some of the work most directly related to this paper. Here we describe some additional related work.

\paragraph{Adversarial Robustness and Interpretability.} 
Through a very different analysis, \cite{Yeh2019-vb} show a result closely related to our Theorem \ref{thm-stable}: 
the show that adversarial training is analogous to making gradient-based explanations more "smooth", which lowers the sensitivity of gradient explanation.
The paper of \cite{Noack2019DoesIO} considers a question that is the converse of the one we examine in our paper: They show evidence that models that are
forced to have interpretable gradients are more robust to
adversarial examples than models trained in a standard
manner. 
Another recent paper \cite{kim2019safeml} analyzes the effect of adversarial training on the interpretability of 
neural network loss gradients.

\paragraph{Relation to work on Regularization Benefits of AML.} There has been prior work on the \textit{regularization benefits} of adversarial training 
\cite{Xu2009-gz,Szegedy2013-tx,Goodfellow2014-cy,Shaham2015-lg,Sankaranarayanan2017-vy,Tanay2018ANA}, 
primarily in image-classification applications: when a model is adversarially trained, its classification accuracy on natural (i.e. un-perturbed) test data can improve. 
All of this prior work has focused on the \textit{performance-improvement} (on natural test data) aspect of regularization, 
but none have examined the \textit{feature-pruning} benefits explicitly.
In contrast to this work, our primary interest is in the explainability benefits of adversarial training, and specifically 
the ability of adversarial training to significantly improve feature-concentration while maintaining (and often improving) performance on natural test data.

\paragraph{Adversarial Training vs Feature-Selection.} Since our results show that adversarial training can effectively shrink the weights of irrelevant or weakly-relevant features (while preserving weights on relevant features),
a legitimate counter-proposal might be that one could weed out such features beforehand via a pre-processing step where features with negligible label-correlations can be "removed" from the training process.
Besides the fact that this scheme has no guarantees whatsoever with regard to adversarial robustness, there are some practical reasons why correlation-based feature selection is not as effective as adversarial training, in producing pruned models: (a) With adversarial training, one needs to simply try different values of the adversarial strength parameter $\varepsilon$ and find a level where accuracy (or other metric such as AUC-ROC) is not impacted much but model-weights are significantly more concentrated; on the other hand with the correlation-based feature-pruning method, one needs to set up an iterative loop with gradually increasing correlation thresholds, and each time the input pre-processing pipeline needs to be re-executed with a reduced set of features. (b) When there are categorical features with large cardinalities, where just some of the categorical values  have negligible feature-correlations, it is not even clear how one can "remove" these specific feature \textit{values}, since the feature itself must still be used; at the very least it would require a re-encoding of the categorical feature each time a subset of its values is "dropped" (for example if a one-hot encoding or hashing scheme is used). Thus correlation-based feature-pruning is a much more cumbersome and inefficient process compared to adversarial training.

\paragraph{Adversarial Training vs Other Methods to Train Sparse Logistic Regression Models.}
\cite{Tan2013-ot, Tan2014-yw} propose an approach to train sparse logistic regression models based on a min-max optimization problem that can be solved by the cutting plane algorithm. 
This requires a specially implemented custom optimization procedure.
By contrast, $\linfeps$-adversarial training can be implemented as a 
simple and efficient "bolt-on" layer on top of existing ML pipelines based on TensorFlow, PyTorch or SciKit-Learn, which makes it
highly practical. 
Another paper \cite{Abramovich2017-jw} proposes a feature selection procedure based on penalized maximum likelihood with a complexity
penalty on the model size, but once again this requires 
special-purpose optimization code.

\section{Discussion of Assumptions}


\subsection{Loss Functions Satisfying Assumption \ref{ass-loss-cvx}}
\label{sec-app-loss-fn}

We show here that several popular loss functions satisfy the Assumption \ref{ass-loss-cvx}.

\textbf{Logistic NLL (Negative Log Likelihood) Loss.} \\
$\lossxyw = -\ln(\sigma(y \<\bfw, \bfx>)) 
= \ln(1 + \exp(-y\<\bfw, \bfx>))$, which can be written as 
$g(-y \<\bfw, \bfx>)$ where $g(z) = \ln(1 + e^z)$ is a non-decreasing and convex function.

\textbf{Hinge Loss}\\
$\lossxyw = (1 - y \<\bfw, \bfx>)^+$, which can be written as $g(-y\<\bfw, \bfx>)$ where $g(z) = (1 + z)^+$ is non-decreasing and convex.

\textbf{Softplus Hinge Loss.}\\
$\lossxyw = \ln(1 + \exp(1 - y \<\bfw, \bfx>))$, which can be written as $g(-y \< \bfw, \bfx>)$ where $g(z) = \ln(1 + e^{1+z})$, and clearly $g$ is non-decreasing. 
Moreover the first derivative of $g$,
$g'(z) = 1/(1 + e^{-1-z})$ is non-decreasing, and therefore $g$ is convex.

\subsection{Implications of Assumption \ref{ass-feat-exp}}
\label{sec-app-ass-feat-exp}


\begin{lemma}
Given random variables $X',Y$ where $Y \in \{\pm 1\}$,
if we define $X = X' -  [\E(X' | Y=1) + \E(X' | Y=-1)]/2$, then:
\begin{align} 
\E(X | Y ) &= a Y  \label{eq-prop-exp-sup} \\
\E(YX) &= \E[ \E(YX|Y) ] = \E[Y \E(X|Y)] = \E[ Y^2 a ] = a \label{eq-eyxi-sup} \\
\E(YX|Y) &= Y\E[ X | Y ] = Y^2 a = a, \label{eq-eyxigiveny-sup} 
\end{align}
where $a = [\E(X' | Y=1) - \E(X' | Y=-1)]/2$.
\end{lemma}
\begin{proof}
Consider the function $f(Y) = \E(X' | Y)$, 
and let $b_0 := f(-1)$ and $b_1 := f(1)$. 
Since there are only \textit{two} values of $Y$ that are of interest, 
we can represent $f(Y)$ by a \textit{linear} function  $aY + c$,
and it is trivial to verify that $a = (b_1 - b_0)/2$ and $c = (b_1 + b_0)/2$ are the unique values that are consistent with $f(-1) = b_0$ and $f(1) = b_1$.
Thus if $X = X' - c$, then $\E(X|Y) = aY$, proving (\ref{eq-prop-exp-sup}), and the other two properties follow trivially.
\end{proof}


\section{Expressions for adversarial perturbation and 
 loss-gradient} \label{sec-app-loss-grad}

We show two simple preliminary results for loss functions that satisfy Assumption \ref{ass-loss-non-dec}: 
Lemma \ref{lem-adv-closed-form} shows a simple closed form expression for the $\linfeps$-adversarial perturbation,
and we use this result to derive an expression for the \textit{gradient} of the 
$\linfeps$-adversarial loss 
$\loss(\bfx + \delta^*, y; \bfw)$ with respect to a weight $w_i$ 
(Lemma \ref{lem-grad-loss-adv}).

%

\begin{lemma}[Closed form for $\linfeps$-adversarial perturbation]
\label{lem-adv-closed-form}
For a data point $(\bfx,y)$, given model weights $\bfw$, if the loss function $\lossxyw$ satisfies Assumption \ref{ass-loss-non-dec}, the $\linfeps$-adversarial perturbation $\delta^*$ is given by:
\begin{equation} \tag{\ref{eq-adv-closed-form}}
\delta^* = -y \sign(\bfw) \varepsilon,
\end{equation}
and the corresponding $\linfeps$-adversarial loss is 
\beq \tag{\ref{eq-linfeps-adv-loss}}
\loss(\bfx + \delta^*, y; \; \bfw) = 
    g(\varepsilon ||\bfw||_1 - y \wdotx)
\eeq
\end{lemma}
\begin{proof}

Assumption \ref{ass-loss-non-dec} implies that the loss is non-increasing in 
$y \langle \bfw, x \rangle$, and therefore the 
$\linfeps$-perturbation $\delta^*$  of $x$ that maximizes the loss would be such that, for each $i \in [d]$, $x_i$ is changed by an amount $\varepsilon$ in the direction of $-y\sign(w_i)$, and the result immediately follows.
\end{proof}


\begin{lemma}[Gradient of adversarial loss] \label{lem-grad-loss-adv}
For any loss function satisfying 
Assumption \ref{ass-loss-non-dec}, for a given data point $(\bfx,y)$,
the gradient of the $\linfeps$-adversarial loss is given by:
\begin{equation}
	\tag{\ref{eq-grad-adv}}
	\frac{\partial \loss(\bfx + \delta^*, y; \; \bfw)}{\partial w_i} = 
	-g'(\varepsilon ||\bfw||_1 - y \langle \bfw, \bfx\rangle) \;
	   (y x_i - \sgn(w_i)\varepsilon)
\end{equation}
\end{lemma}
\begin{proof}
This is straightforward by substituting the expression (\ref{eq-adv-closed-form}) for $\delta^*$ in $g(-y \langle \bfw, \bfx + \delta^*\rangle )$, 
and applying the chain rule.
\end{proof}

\section{Expectation of SGD Weight Update} \label{sec-app-thm-exp-sgd}

The following Lemma will be used to prove Theorem \ref{thm-exp-sgd}.

\subsection{Upper bound on $\E [Z f(Z, V)]$}
\begin{lemma}[Upper Bound on expectation of $Z f(Z, V)$ when $f$ is non-increasing in $Z$, $(Z \perp V) | Y$, and $\E(Z|Y) = \E(Z)$]
\label{lem-exp-bound}
For any random variables $Z, V$, if:
\begin{itemize}
	\item $f(Z,V)$ is non-increasing in $Z$,
	\item $Z$, $V$ are conditionally independent given a third r.v. $Y$, and 
	\item  $\E(Z | Y) = \E (Z)$, 
	\end{itemize}  
then
\begin{equation} 
\E[ Z f(Z, V) ] \;\leq\; \E(Z) \E[f(Z, V)] \label{eq-exp-bound}
\end{equation}
\end{lemma}
\newcommand{\zbar}{\overline{z}}
\begin{proof}
Let $\zbar = \E(Z) = \E(Z|Y)$ and note that
\begin{align}
\E[Z f(Z,V)] - \E[Z] \E[f(Z,V)]  &= 
       \E[Z f(Z,V)] - \zbar \E[f(Z, V)] \label{eq-cov1} \\
              &= \E[(Z - \zbar) f(Z, V)]  \label{eq-cov-exp1}
\end{align}
We want to now argue that $\E[(Z - \zbar) f(\zbar, V)]=0$. To see this, apply the Law of Total Expectation by conditioning on $Y$:
\begin{align}
\E[(Z - \zbar) f(\zbar, V)]  &= 
	  \E \Big[ 
	      \E \big[ 
	         (Z - \zbar) f(\zbar, V) | Y 
	         \big]
	     \Big] & \nonumber \\ 
	     &= \E \Big[ 
	             \E \big[ (Z - \zbar) | Y \big] 
	             \E \big[ f(\zbar, V) | Y \big] 
	             \Big] & \text{(since}\; (Z \perp V) | Y \text{)} \\
         &= 0.  & \text{(since} \; \E(Z|Y) =  \E(Z) = \zbar \text{)}
\end{align}
Since $\E[(Z - \zbar) f(\zbar, V)]=0$, we can subtract it from the last expectation in (\ref{eq-cov-exp1}), and by linearity of expectations the RHS of (\ref{eq-cov-exp1}) can be replaced by
\begin{align}
\E\big[ (Z - \zbar) (f(Z, V) - f(\zbar,V)) \big].  \label{eq-cov-exp2}
\end{align}
That fact that $f(Z,V)$ is non-increasing in $Z$ implies that
$(Z - \zbar)(f(Z,V) - f(\zbar,V)) \leq 0$ for any value of $Z$ and $V$, with equality when $Z = \zbar$. 
Therefore the expectation (\ref{eq-cov-exp2}) is bounded above by zero, which implies the desired result.
\end{proof}	


\thmsgd*

\begin{proof}
Consider the adversarial loss gradient expression 
(\ref{eq-grad-adv}) from Lemma \ref{lem-grad-loss-adv}. 
For the case where $w_i=0$ for all $i \in S$, for any given $i \in S$, the negative expectation of the adversarial loss gradient is 
\begin{align} \label{eq-exp1}
	\delbar_i &= 
	 \E \big[
	      y x_i \,
     	 g'(\varepsilon ||\bfw||_1 - y\<\bfw, \bfx>)
	   \big] & \nonumber \\
	   &= 
	   	 \E \Big[ 
	   	      \E \big[
         	      y x_i \,
     	          g'(\varepsilon ||\bfw||_1 - y\<\bfw, \bfx>) \; 
     	          | y
          	   \big] 
          	\Big] & \nonumber \;\; \text{(Law of Total Expectation)} \\
	   &= 
	   	 \E \Big[ 
	   	     y \, \E \big[
         	      x_i \,
     	          g'(\varepsilon ||\bfw||_1 - y\<\bfw, \bfx>) \; 
     	          | y
          	   \big] 
          	\Big], & \nonumber           	
\end{align}
and in the last expectation above, we note that since $w_i=0\; \forall i \in S$,
the argument of $g'$ does not depend on $x_i$ for any $i \in S$,
and since the features in $S$ are conditionally independent of the rest given the label $y$,
$x_i$ is independent of the $g'$ term in the inner conditional expectation. 
Therefore the inner conditional expectation can be factored as a product of conditional expectations, which gives
\begin{align}
\delbar_i	&= 
	   	 \E \Big[ 
	   	     y \E (x_i | y) 
	   	     \E \big[ 
     	          g'(\varepsilon ||\bfw||_1 - y\<\bfw, \bfx>) \; 
     	          | y
          	   \big] 
          	\Big] & \nonumber \\
          	&= 
	   	 \E \Big[ 
	   	     y^2 a_i
	   	     \E \big[ 
     	          g'(\varepsilon ||\bfw||_1 - y\<\bfw, \bfx>) \; 
     	          | y
          	   \big] 
          	\Big] &  \text{(Assumption \ref{ass-feat-exp}, Eq \ref{eq-prop-exp-sup})}
          	                      \nonumber \\
          	&= 
	   	 a_i \E \Big[ 
   	   	     \E \big[ 
     	          g'(\varepsilon ||\bfw||_1 - y\<\bfw, \bfx>) \; 
     	          | y
          	   \big] 
          	\Big] &  \text{(since $y = \pm 1$)} \nonumber \\
          	&= a_i \gprimebar,
\end{align}
which establishes the first result.

Now consider the case where $w_i \neq 0$ for at least one $i \in S$.
Starting with the adversarial loss gradient expression 
(\ref{eq-grad-adv}) from Lemma \ref{lem-grad-loss-adv}, 
for any $i \in S$, 
multiplying throughout by $-\sgn(w_i)$ and taking expectations results in 
\begin{equation} \label{eq-exp2}
	\sgn(w_i) \delbar_i = 
	 \E \Big[
	     \big[ 
	      y x_i \sgn(w_i) - \varepsilon 
	     \big] \; 
     	 g'(\varepsilon ||w||_1 - y\<\bfw, \bfx>)
	   \Big]
\end{equation}
where the expectation is with respect to a random choice of data-point $(\bfx, y)$. 
The argument of $g'$ can be written as 
\begin{equation*}	
\varepsilon ||w||_1 - y \<\bfw, \bfx> = 
-\sum_{j=1}^d |w_j| (yx_j \sgn(w_j) - \varepsilon),
\end{equation*}
and for $j \in [d]$ if we let  $Z_j$ denote the random variable 
corresponding to $yx_j \sgn(w_j) - \varepsilon$, then
(\ref{eq-exp2}) can be written as
\begin{equation} \label{eq-exp3}
\sgn(w_i)\delbar_i = \E \left[ Z_i \; g'\left(- \sum_{j=1}^d |w_j| Z_j \right) \right].
\end{equation}
Taking the $|w_i|$-weighted average of both sides of (\ref{eq-exp3}) over $i \in S$ yields
\begin{equation} 
\delbarws = \frac{1}{\sum_{i\in S} |w_i|} 
           \E \left[  g'\left(- \sum_{j=1}^d |w_j| Z_j \right) \;
			           \sum_{i \in S} \left(|w_i| Z_i\right)
           \right]. \label{eq-exp4}
\end{equation}
If we now define $Z_S := \sum_{i \in S} (|w_i| Z_i)$, the argument of $g'$ in the expectation above can be written as $V_S - Z_S$ where $V_S$ denotes the negative sum of $|w_j|Z_j$ terms over all $j \not\in S$, and thus (\ref{eq-exp4}) can be written as
\begin{equation}  \label{eq-exp5}
\delbarws = \frac{1}{\sum_{i\in S} |w_i|} 
           \E \left[  g'\left(V_S - Z_S\right) Z_S
           \right].
\end{equation}
Note that  $Z_S$ is a function of $Y$ and the features in $S$, 
and $V_S$ is a function of $Y$ and the features in the \textit{complement} of $S$.
Since the features in $S$ are conditionally independent of the rest given the label $Y$
(this is a condition of the Theorem), it follows that 
$(V_S \perp Z_S) | Y$. 
Since by Assumption \ref{ass-loss-cvx}, $g'$ is a non-decreasing function, $g'(V_S - Z_S)$ is \textit{non-increasing} in $Z_S$. Thus all three conditions of Lemma \ref{lem-exp-bound} are satisfied, with the random variables $Z,V,Y$ and function $f$ in the Lemma corresponding to random variables $Z_S, V_S, Y$ and function $g'$ respectively in the present Theorem. 
It then follows from Lemma \ref{lem-exp-bound} that 
\begin{equation} \label{eq-exp5.1}
	\delbarws \leq \frac{1}{\sum_{i\in S} |w_i|}  E(Z_S) \gprimebar. 
\end{equation}
The definition of $Z_S$, and the fact that $\E(Z_i) = \sgn(w_i) \E(yx_i) - \varepsilon$ = 
$a_i \sgn(w_i) - \varepsilon$ (property (\ref{eq-eyxi})), imply 
\begin{equation*}
\E(Z_S) = \sum_{i\in S} \left[ |w_i| \sgn(w_i) a_i \right] - \varepsilon \sum_{i\in S} |w_i|,
\end{equation*}
and the definition of $\baraws$ allows us to simplify (\ref{eq-exp5.1}) to
\begin{equation*}
	\delbarws \leq \gprimebar (\baraws - \varepsilon), \\
\end{equation*}
which establishes the upper bound (\ref{eq-thm-wnonzero}). 

To analyze the limiting case where  $w_i \rightarrow 0$ for all $i \in S$, write Eq. (\ref{eq-exp5}) as follows:
\newcommand\given[1][]{\:#1\vert\:}
\begin{equation}  \label{eq-exp6}
\delbarws = 
           \E \left[  g'\left(V_S - Z_S\right) \frac{Z_S}{\sum_{i \in S} |w_i|}
           \right].
\end{equation}
If we let $|w_i| \rightarrow 0$ for all $i \in S$, the $Z_S$ in the argument of $g'$ can be 
set to 0, and we can write the RHS of 
(\ref{eq-exp6}) as 
\begin{equation}  \label{eq-exp7}
\delbarws = 
           \E \left[  g'(V_S) \frac{Z_S}{\sum_{i \in S} |w_i|}    \right]
            = 
           \E \bigg[ \E \left[ g'(V_S) \frac{Z_S}{\sum_{i \in S} |w_i|} \; \given[\Big] Y \right]
           \bigg],
\end{equation}
where the inner conditional expectation can be factored as a product of conditional expectations
since $(Z_S \perp V_S | Y)$:
\begin{equation} \label{eq-exp8}
	\delbarws =  \E \bigg[ 
	      \E \left[ g'(V_S) \given[\Big] Y \right] 
	      \E \left[ \frac{Z_S}{\sum_{i \in S} |w_i|} \; \given[\Big] Y \right]
	      \bigg].
\end{equation}
Now notice that 
\begin{equation}
\E(Z_S | Y) = \E \left[ \sum_{i \in S} (|w_i| Z_i) \given[\Big] Y \right] = 
			  \sum_{i\in S} |w_i| \;
			       \E \left[ \sgn(w_i) Y x_i - \varepsilon \;|\; Y \right].
\end{equation}
From Property (\ref{eq-eyxigiveny}) of datasets satisfying Assumption \ref{ass-feat-exp}, 
$\E [Yx_i | Y] = \E(Y x_i) = a_i$, and so the second inner expectation in (\ref{eq-exp8})
simplifies to a constant:
\begin{equation}
	      \E \left[ \frac{Z_S}{\sum_{i \in S} |w_i|} \; \given[\Big] Y \right] 
	      = \frac{ \sum_{i \in S} \left[ |w_i| \sgn(w_i) a_i \right] } 
	             { \sum_{i \in S} |w_i| } - \varepsilon
	      = \baraws - \varepsilon.
\end{equation}
Eq. (\ref{eq-exp8}) can therefore be simplified to 
\begin{equation} \label{eq-exp9}
	\delbarws =  \E \bigg[ 
	      \E \left[ g'(V_S) \given[\Big] Y \right] 
	      \bigg] 
	       \left( \baraws - \varepsilon \right) = 
	       \gprimebar ( \baraws - \varepsilon),
\end{equation}
which shows the final statement of the Theorem, namely, that if $w_i \rightarrow 0$ for all $i \in S$, then (\ref{eq-thm-wnonzero}) holds with equality.
\end{proof}

\subsection{Implications of Theorem \ref{thm-exp-sgd}} 
\label{sec-implication}
Keeping in mind the interpretations of the $\w$-weighted average quantities $\baraws$ and $\delbarws$ described in the paragraph after the statement of Theorem \ref{thm-exp-sgd},
we can state the following implications of this result:

\textbf{If all weights of $S$ are zero, then they grow in the correct direction.}
When $w_i = 0$ for all $i \in S$ (recall that $S$ is a subset of features, conditionally independent of the rest given the label $y$), the expected SGD update $\delbar_i$ for each $i \in S$ is proportional to the directed strength $a_i$ of feature $x_i$, and if $\gprimebar \neq 0$, this means that on average the SGD update causes the weight $w_i$ to grow from zero in the \textit{correct direction}. This is what one would expect from an SGD training procedure.

\textbf{If the weights of $S$ are mis-aligned weights on average, then they shrink at a rate proportional to $\varepsilon + |\baraws|$.} 
Suppose for at least one $i \in S$, $w_i \neq 0$, and $\baraws < 0$, i.e. the weights of the features in $S$ are mis-aligned on average. In this case by (\ref{eq-thm-wnonzero}), $\delbarws < 0$, i.e. the weights of the features in $S$, in aggregate (i.e. in the $|w_i|$-weighted sense) shrink toward zero in expectation. The aggregate rate of this shrinkage is proportional to $\varepsilon + |\baraws|$. In other words, all other factors remaining the same, adversarial training (i.e. with $\varepsilon > 0$) shrinks mis-aligned faster than natural training (i.e. with $\varepsilon = 0$).

\textbf{If the weights of $S$ are aligned on average, and $\varepsilon > |\baraws|$ then they shrink at a rate proportional to $\varepsilon - |\baraws|$.}
Suppose that $w_i \neq 0$ for at least one $i \in S$, and the weights of $S$ are aligned on average, i.e. $\baraws > 0$. Even in this case, the weights of $S$ shrink on average, \textit{provided} 
the alignment strength $\baraws$ is dominated by the adversarial $\varepsilon$;
the rate of shrinkage is proportional to $\varepsilon - |\baraws|$, by Eq. (\ref{eq-thm-wnonzero}).
Thus  adversarial training with a \textit{sufficiently large} $\varepsilon$ that dominates the 
average strength of the features in $S$, will cause the weights of these features to shrink on average. This observation is key to explaining the "feature-pruning" behavior of adversarial training: "weak" features (relative to $\varepsilon$) are weeded out by the SGD updates.

\textbf{If the weights of $S$ are aligned, $\varepsilon < |\baraws|$, then the weights of $S$ expand up to a certain point.}
Consider the case where at least one of the $S$ weights is non-zero, 
and the adversarial $\varepsilon$ does not dominate the average strength $\baraws$.
Again from Eq. (\ref{eq-thm-wnonzero}), if  $\baraws > 0$ and $\varepsilon < |\baraws|$, then the upper bound (\ref{eq-thm-wnonzero}) on $\delbarws$ is non-negative. 
Since the Theorem states that equality holds in the limit as $w_i \rightarrow 0$ for all $i \in S$,
this means if all $|w_i|$ for $i \in S$ are sufficiently small, the expected SGD update $\delbarws$  is non-negative, i.e., the $S$ weights expand on average. In other words, the weights of 
a conditionally independent feature-subset $S$, if they are aligned on average, then their aggregate weights  expand on average up to a certain point, if $\varepsilon$ does not dominate their strength.

Note that Assumption \ref{ass-loss-cvx} implies that $\gprimebar \geq 0$,
and when the model $\bfw$ is "far" from the optimum, the values of $-\ywdotx$ will tend to be large, 
and since $g'$ is a non-decreasing function (Assumption \ref{ass-loss-cvx}), $\gprimebar$ will be large as well. 
So we can interpret $\gprimebar$ as being a proxy for "average model error".
Thus during the initial iterations of SGD, this quantity will tend to be large and positive, and shrinks toward zero as the model approaches optimality.
Since $\gprimebar$ appears as a factor in (\ref{eq-thm-wzero}) and (\ref{eq-thm-wnonzero}), we can conclude that the above effects will be more pronounced in the initial stages of SGD and less so in the later stages.
The experimental results described in Section \ref{sec-exp} are consistent with several of the above effects.

\section{Generalization of Theorem \ref{thm-exp-sgd} for the multi-class setting}\label{sec:multi-class}
\subsection{Setting and Assumptions}
Let there be $k\geq 3$ classes. For a given data point $\mathbf{x}\in \mathbb{R}^d$,  its true label,  $i \in [k]$, is represented by a vector $\mathbf{y}=[\underbrace{-1 \cdots -1}_\text{i-1} 1 \underbrace{-1 \cdots -1}_\text{k-i}]$. We assume that the input $(\mathbf{x},\mathbf{y})$ is drawn from the distribution $\mathcal{D}$. For this multi-class classification problem, we assume the usage of the standard one-vs-all classifiers, i.e., there are $k$ different classifiers with the $i$-th classifier (ideally) predicting $+1$ iff the true label of $\mathbf{x}$ is $i$, else it predicts $-1$. Let $\mathbf{w}$ represent the $k\times d$ weight matrix where $\mathbf{w_i}$ represents the $1\times d$ weight vector for the $i$-th classifier. $w_{ij}$ represents the $j$-th entry of $\mathbf{w_i}$.
Let $y_i$ represent the $i$-th entry of $\mathbf{y}$.

The assumptions presented in the main paper (Sec. \ref{sec-notation}) are slightly tweaked as follows and hold true for each of the $k$ one-vs-all classifiers:
\\\\
\textbf{Assumption LOSS-INC}: \textit{The loss function for each of the one-vs-all classifier is of the form $\mathcal{L}(\mathbf{x},y_i;\mathbf{w_i}) = g(-y_i\langle \mathbf{w_i},\mathbf{x} \rangle)$ where $g$
 is a non-decreasing function}.\\\\
 \textbf{Assumption LOSS-CVX}: \textit{The loss function for each of the one-vs-all classifier is of the form $\mathcal{L}(\mathbf{x},y_i;\mathbf{w_i}) = g(-y_i\langle \mathbf{w_i},\mathbf{x} \rangle)$ where $g$
 is  non-decreasing, almost-everywhere differentiable and convex.}\\\\
 \textbf{Assumption FEAT-INDEP}: \textit{The features $\mathbf{x}$ are conditionally independent given the label $y_i$ for the $i$-th one-vs-all classifier, i.e., for any two distinct induces $s,t$, $x_s$ is independent of $x_t$ given $y_i$, or more compactly, $(x_s\perp x_t)\thinspace | \thinspace y_i$.}\\\\
 \textbf{Assumption FEAT-EXP}: \textit{For each feature $x_j, j \in [d]$ and the $i$-th one-vs-all classifier $\mathbb{E}(x_j|y_i)=a_{ij}.y_i$ for  some constant $a_{ij}$}.\\\\
 Additionally, we introduce a new assumption on the distribution $\mathcal{D}$ as follows.\\\\
 \textbf{Assumption  DIST-EXPC}: \textit{The input distribution $\mathcal{D}$ satisfies the following expectation for a function $h_i, i \in [k]$ (defined by Eqs. \eqref{a1},\eqref{a2}, \eqref{a3}) and constant $g^*$ (defined by Eq. \eqref{g*})} \begin{gather}\mathbb{E}\Big[h_i(\text{sgn}(w_{ij})y_i,\mathbf{x},\mathbf{w_i},\epsilon)\Big]=0\label{E}\end{gather} 
\begin{gather*}Pr_{\mathcal{D}}[ \epsilon < x
_j < \epsilon + \rho]=0, \hspace{0.3cm}\rho \text{ is a small constant} \numberthis\label{rho} \\x_j\geq \epsilon+\rho\implies (x_j-\epsilon)g^*_i\geq (x_j+\epsilon)g'(-y_i\langle \mathbf{w_i},\mathbf{x}+\delta^*
\rangle)\numberthis \label{g*}\\\hspace{-8.2cm}\text{ If } y_i\text{sgn}(w_{ij})=-1, \text{ then}\\h(\text{sgn}(w_{ij})y_i,\mathbf{x},\mathbf{w_i},\epsilon) \geq (x_j+\epsilon)g^*_i-(x_j-\epsilon) g'(-y_i\langle \mathbf{w_i},\mathbf{x}+\delta^*
\rangle)\numberthis \label{a1}\\  \hspace{-7.2cm}\text{ If }  y_i\text{sgn}(w_{ij})=1 \wedge x_j > \epsilon, \text{ then}\\ -\Big((x_j-\epsilon)g^*_i- (x_j+\epsilon)g'(-y_i\langle \mathbf{w_i},\mathbf{x}+\delta^*
\rangle)\Big)\leq h(\text{sgn}(w_{ij})y_i,\mathbf{x},\mathbf{w_i},\epsilon) \leq 0 \numberthis \label{a2}\\ \hspace{-7.2cm}\text{ If }  y_i\text{sgn}(w_{ij})=1 \wedge x_j \leq \epsilon, \text{ then}\\h(\text{sgn}(w_{ij})y_i,\mathbf{x},\mathbf{w_i},\epsilon)\geq (x_j+\epsilon)g'(-y_i\langle\mathbf{w_i},\mathbf{x}+\delta^*\rangle)-(x_j-\epsilon)g^*_i\numberthis\label{a3}\end{gather*}
This assumption is not as restrictive as it may appear. Eq. \ref{rho} can be satisfied naturally for discrete domains. For example, for images $x_j \in \{0, 1, 2,\cdots, 254, 255\}$; thus $\rho \in (0,1)$. For continuous domains, $\rho$ can be set to a small value and the   values of $x_j$ can be appropriately rounded  in the input dataset. \\ For the rest of the discussion let us consider the case where $g'(z)=c$ (for example for hinge loss function $c=1$ for $z>-1$) and $x_j \in [0,1]$. Now consider,
\begin{gather*}
g^*=(1+\epsilon)c/\rho, \hspace{0.1cm} \rho=0.01\\
f_1(x):=((x+\epsilon)g^*-(x-\epsilon)c)\\
f_2(x):=((x+\epsilon)c -(x-\epsilon)g^*)\\h_i(\text{sgn}(w_{ij})y_i,\mathbf{x},\mathbf{w_i},\epsilon)=\begin{cases}f_1(x_j)\hspace{0.4cm} \mbox{if } \text{sgn}(w_{ij})y_i=+1 \\f_2(x_j)\hspace{0.4cm} \mbox{otherwise}\end{cases}\\
\hspace{-1.1cm}\mathbb{E}[h_i(\text{sgn}(w_{ij})y_i,\mathbf{x},\mathbf{w_i},\epsilon)]=\int_{0}^1 Pr[x_j|y_i\text{sgn}(w_{ij})=-1]\cdot f_1(x_j)dx +\\\hspace{1cm} \int_{\epsilon}^1\hspace{-0.2cm}Pr[x_j|y_i\text{sgn}(w_{ij})=+1]\cdot f_2(x) dx+ \int_{0}^\epsilon Pr[x_j|y_i\text{sgn}(w_{ij})=+1]\cdot f_2(x) dx
\end{gather*}
We observe that $f_1(x)$ is increasing in $x\in[0,1]$, and $x\geq\epsilon+\delta  \implies f_2(x)\leq 0$ and $f_2(x)$ is decreasing in $x \in [\epsilon+\delta,1]$. 
Thus intuitively for $\mathbb{E}(h_i)$ to be zero, $Pr[x_j|\text{sgn}(w_{ij})y_i=-1]$ must have high values for lower magnitudes of $x_j$ (say $x_j<0.5$), and $Pr[x_j|\text{sgn}(w_{ij})y_i=-1]$ has low values for $x_j
\leq\epsilon$ and high values for $x_j\geq\epsilon+\delta$. 
For example, let us assume $\epsilon=0.1$ and that the distributions $Pr[x_i|\text{sgn}(w_{ij})y_i=+1]$  and $Pr[x_i|\text{sgn}(w_{ij})y_i=-1]$ can be approximated by truncated Gaussian distributions (with appropriate adjustments to ensure $Pr[\epsilon < x_j < \epsilon + \delta]=0]$) with means $m_1$ and $m_2$ respectively. Then, it can be seen that there exists $h_i$ for $m_1<0.3$ and $m_2>0.6$ such that $\mathbb{E}[h_i]=0$.  \\\\
The overall loss function, $\mathcal{L}_T$ for the multi-class classifier is the sum of the loss functions of each individual one-vs-all classifiers and is given by  \begin{gather*}\mathcal{L}_T(\mathbf{x},\mathbf{y};\mathbf{w})=\sum_{i=1}^k\mathcal{L}(\mathbf{x},y_i;\mathbf{w_i})\\=\sum_{i=1}^k g(-y_i\langle\mathbf{w_i},\mathbf{x}\rangle)\hspace{1cm}[\mbox{From Assumption \textbf{LOSS-CVX}}]\end{gather*}
The \textbf{expected SGD update } $\overline{\Delta w_{ij}}$ is defined as follows:
\begin{gather}  \Delta w_{ij} =
    \begin{cases}
      \mathbb{E}\frac{\partial \mathcal{L}_T(\mathbf{x}+\delta^*,\mathbf{y};\mathbf{w})}{\partial w_{ij}} & \text{when $w_i=0$}\\
      -\mbox{sgn}(w_{ij})\mathbb{E}\frac{\partial \mathcal{L}_T(\mathbf{x}+\delta^*,\mathbf{y};\mathbf{w})}{\partial w_{ij}}
      & \text{when $w_i\neq 0$}
    \end{cases}   \end{gather}
    Also let \begin{gather} 
    \overline{g'_i}:=\mathbb{E}[g'(-y_i\langle\mathbf{w_i},\mathbf{x}+\delta^*\rangle)], i \in [k] \label{g2}\end{gather}

\begin{restatable}[Expected SGD Update in Adversarial Training for Multi-Class Classification]{thm}{thmmulticlass} \label{thm:multi_class} For any loss function $\mathcal{L}$ satisfying assumptions \textbf{LOSS-CVX}, \textbf{FEAT-INDEP} and \textbf{FEAT-EXP}, if a data point $(\mathbf{x},\mathbf{y})$ is randomly drawn from $\mathcal{D}$ that satisfies Assumption \textbf{DIST-EXPC}, and $\mathbf{x}$ is perturbed to $\mathbf{x}'=\mathbf{x}+\delta^*$, where $\delta^*$ is an $l_\infty(\epsilon)$-adversarial perturbation, then under the $l_\infty(\epsilon)$-adversarial loss $\mathcal{L}_T(\mathbf{x},\mathbf{y};\mathbf{w})$  the expected SGD-update of weight $w_{ij}$, namely $\overline{\Delta w_{ij}}$ satisfies the following properties
\begin{enumerate}\item If $w_{ij}=0$, then \begin{gather*} \overline{\Delta w_{ij}} = a_{ij}\overline{g'_i}\end{gather*} 
\item If $w_{ij}\neq 0$, then \begin{gather*} \overline{\Delta w_{ij}} \leq \tilde{g}[a_{ij}\thinspace \mbox{sgn}(w_{ij})-\epsilon]\\ \tilde{g} \in \{\overline{g'_i},g^*_i\}\end{gather*}
\end{enumerate}
\end{restatable}
\begin{proof}For \begin{gather}\delta^* \in \mathbb{R}^d \text{ s.t}\\\mathcal{L}_T(\mathbf{x}+\delta^*,\mathbf{y};\mathbf{w})=\mathbb{E}_{\mathbf{x},\mathbf{y}\sim \mathcal{D}}[\max_{||\delta||_{\infty}\leq \epsilon} \mathcal{L}_T(\mathbf{x}+\delta,\mathbf{y};\mathbf{w})]
\end{gather} the shift by $\epsilon$ in $x_j$ can be in either of the two directions, $y_i \text{sgn}(w_ij)$ or $-y_i\text{sgn}(w_ij)$. We have, 
\begin{gather}\frac{\partial\lT)}{\partial w_{ij}}=\sum_{s=1}^k\frac{\partial\big(  g(-y_s\langle\mathbf{w}_s,\mathbf{x}+\delta^*\rangle)\big)}{\partial w_{ij}}=\frac{\partial g(-y_i\langle\mathbf{w}_i,\mathbf{x}+\delta^*\rangle)}{\partial w_{ij}}\end{gather}
Thus by  Assumption \textbf{LOSS-CVX} either of the following two equations hold true 
\begin{gather}\frac{\partial\lT)}{\partial w_{ij}}=g'(-y_i\langle\mathbf{w}_i,\mathbf{x}+\delta^*\rangle)\sgnpos \label{pos}\\\frac{\partial\lT)}{\partial w_{ij}}=g'(-y_i\langle\mathbf{w}_i,\mathbf{x}+\delta^*\rangle)\sgnneg \label{neg}\end{gather}

Thus when $w_{ij}=0$,
\begin{gather*}\overline{\Delta w_{ij}}=\mathbb{E}\Big[y_ix_j
g'(-y_i\langle\mathbf{w_i},\mathbf{x}+\delta^*\rangle)\Big]\\=a_{ij}\overline{g'_i}[\text{Follows from the proof in Theorem \ref{thm-exp-sgd} in the paper}]\end{gather*}
Now let us consider the case when $w_{ij}\neq 0$.\\
\textcolor{red}{\textbf{Case I}}: Eq.~\ref{pos} is satisfied\\
This means that for $\delta^*$, $x_j$ is changed by $\epsilon$ in the direction of $-y_i\text{sgn}(w_{ij})$.
Thus after multiplying throughout with $-\text{sgn}(w_{ij})$ and taking expectations, we have
\begin{gather}\overline{\Delta w_{ij}}=\mathbb{E}\Big[[y_ix
_j\text{sgn}(w_{ij})-\epsilon]g'(\sum_{l=1, l\neq j}^ds_l\epsilon |w_{il}|+\epsilon|w_{ij}|-y_i\langle\mathbf{w_i},\mathbf{x}\rangle)\Big]\end{gather}
where $s_l$ represents the corresponding sign for the value $\epsilon|w_l|$ (based on which direction is $x_l$ perturbed in) and the expectation is with respect to a random choice of data point $(\mathbf{x},\mathbf{y})$. Let us define two random variables $V$ and $Z$ as follows \begin{gather}Z=y_ix_j\text{sgn}(w
_{ij})-\epsilon\\
V=-\sum_{l=1, l \neq j}^d |w_{il}|\Big(y_ix_l\text{sgn}(w_{il})-s_l\epsilon\Big)\end{gather} Thus, \begin{gather}\overline{\Delta w_{ij}}=\mathbb{E}[Zg'(V-|w_{ij}|Z)]\end{gather}
Let random variable $Y$ correspond  to the label $y_i$ of the data point. Since Z is a function of feature $x_j$ and 
$Y$ , and 
$V$ is a function of the remaining features and $Y$, Assumption \textbf{FEAT-INDEP}  implies $(V \perp Z)|Y$. Additionally by Assumption \textbf{FEAT-EXP} \begin{gather}
\mathbb{E}(Z) = \mathbb{E}(Z|Y) = a_j \text{sgn}(w_{ij}) - \epsilon \end{gather}
Since by Assumption \textbf{LOSS-CVX}, g' is a non-decreasing function, $g'(V - |w_{ij}|Z)$ is non-increasing in $Z$. Thus all three conditions of Lemma \ref{lem-exp-bound} are satisfied, with the random variables $Z, V, Y$ and function $f$ in the Lemma corresponding to random variables $Z, V, Y$ and function $g'$ respectively in the present theorem. Following an analysis similar to the one presented in the proof for Theorem \ref{thm-exp-sgd}, we have
\begin{gather}
\overline{\Delta w_{ij}} \leq \mathbb{E}(Z)\overline{g'_i}  = \overline{g'_i}[a_{ij} \text{sgn}(w_{ij}) - \epsilon]
\end{gather}
\textcolor{red}{\textbf{Case II:}} Eq. \ref{neg} is satisfied\\In this case, for $\delta^*$ $x_j$ is perturbed by $\epsilon$ in the direction $y_i\text{sgn}(w_{ij})$. Now multiplying both sides by $-\text{sgn}(w_{ij})$ 
\begin{gather}\Delta w_{ij}=(y_ix_j\text{sgn}(w_{ij})+\epsilon)g'(-y_i\langle\mathbf{w_i},\mathbf{x}+\delta^*\rangle)  \label{w} \end{gather}

Now let us consider the case when $y_i\text{sgn}(w_{ij})=-1$.
From Assumption \textbf{DIST-EXPC} (Eqs. \eqref{w},\eqref{g*} and \eqref{a1}), we have \begin{gather}\Delta w_{ij}\leq (y_ix_j\text{sgn}(w_{ij})-\epsilon)g^*_i+h_i(\text{sgn}(w_{ij})y_i,\mathbf{x},\mathbf{w_i},\epsilon)\label{wf}\end{gather}
For the case when $y_i\text{sgn}(w_{ij})=-1 \wedge x_j > \epsilon$, again from Eqs. \eqref{w},\eqref{g*} and \eqref{a2} Eq. \eqref{wf} holds true. Similarly, for the case of $y_i\text{sgn}(w_{ij})=-1 \wedge x_j \leq \epsilon$, the validity of Eq. \eqref{wf} can be verified from Eqs. \eqref{w},\eqref{g*} and \eqref{a3}.
Now taking expectations over both sides of Eq. \eqref{wf} results in \begin{gather}\overline{\Delta w_{ij}}\leq (a_{ij}\text{sgn}(w_{ij})-\epsilon)g^*_i \hspace{1cm}[\text{From Eq. \eqref{E}}] \end{gather}
This concludes our proof.
\end{proof}

\section{Proof of Lemma \ref{lem-1-layer-attr}}
\label{app-lem-ig}

\lemig*
\begin{proof}
\renewcommand{\sigma}{A}
Since the function $F$, 
the baseline input $\bfu$ and weight vector $\bfw$ are fixed, we omit them from $\igfxu$ and $\igifxu$ for brevity.
Consider the partial derivative $\partial_i F(\bfu + \alpha(\bfx-\bfu))$ in the definition (\ref{eq-ig-def}) of $\IG_i(\bfx)$. For a given $\bfx$, $\bfu$ and $\alpha$, let $\bfv$ denote the vector $\bfu + \alpha(\bfx- \bfu)$. 
Then $\partial_i F(\bfv) = \partial F(\bfv)/\partial v_i$, and by applying the chain rule we get:
\newcommand\aprimewv{A'(z)}
\begin{align*}
	\partial_i F(\bfv) &:= 
	 \frac{\partial F(\bfv)}{\partial v_i} = 
	 \frac{\partial A(\<\bfw, \bfv>)}{\partial v_i} = 
	 A'(z) \frac{\partial \<\bfw, \bfv>}{\partial v_i} = 
	 w_i A'(z),
\end{align*}
where $A'(z)$ is the gradient of the activation $A$ at $z = \<\bfw, \bfv>$.
This implies that:
\begin{align*}
\frac{\partial F(\bfv)}{\partial \alpha} 
			&= \sum_{i=1}^d \left( \frac{\partial F(\bfv)}{\partial v_i} 
                             \frac{\partial v_i}{\partial \alpha} \right)\\
			&= \sum_{i=1}^d [ \theta_i \aprimewv (x_i - u_i) ] \\
			&= \< \bfx - \bfu, \bfw> \aprimewv
\end{align*}
We can therefore write 
\begin{equation*}
	dF(\bfv) = \<\bfx - \bfu, \bfw > \aprimewv d\alpha,
\end{equation*}
and since $\<\bfx - \bfu, \bfw >$ is a scalar, this yields
\begin{equation*}
	\aprimewv d\alpha = \frac{dF(\bfv)}{\< \bfx - \bfu, \bfw >} 
\end{equation*}
Using this equation the integral in the definition of $IG_i(x)$ can be written as

\begin{align}
	\int_{\alpha=0}^1 \partial_i F(\bfv) d\alpha 
	 &= 	\int_{\alpha=0}^1 \theta_i \aprimewv d\alpha \nonumber \\ 
	 &= 	\int_{\alpha=0}^1 \theta_i \frac{d F(\bfv)}
	                            { \<\bfx - \bfu, \bfw >} \nonumber\\
	 &= 	\frac{w_i}{\< \bfx - \bfu,  \bfw >} \int_{\alpha=0}^1 d F(\bfv)  
	 \label{eq-ig-prf-int}\\
	 &= 	\frac{\theta_i}{\<\bfx - \bfu, \bfw >} [ F(\bfx) - F(\bfu)], \nonumber	 
\end{align}
where (\ref{eq-ig-prf-int}) follows from the fact that $(\bfx - \bfu)$ and $\bfw$ do not depend on $\alpha$. Therefore from the definition (\ref{eq-ig-def}) of $\IG_i(\bfx)$:
\begin{equation*}
	\IG_i(\bfx) =  [ F(\bfx) - F(\bfu) ] \frac{ (x_i - u_i)w_i }
	                                           {\< \bfx - \bfu, \bfw >},
\end{equation*}
and this yields the expression (\ref{eq-ig-1-layer}) for $\IG(\bfx)$.
\end{proof}

\section{Proof of Theorem \ref{thm-stable}}
\label{app-thm-stable}

\thmstable*
\begin{proof}
Recall that Assumption \ref{ass-loss-cvx} implies 
$\lossxyw = g(-y \wdotx)$ for some non-decreasing, differentiable, convex function $g$.
Due to this special form of $\lossxyw$,  the function $\lossy$ is a differential function of $\wdotx$, and by 
 Lemma \ref{lem-1-layer-attr} the $i$'th component of the $\IG$ term in (\ref{eq-stable-robust}) is 
\beqstar
\IG_i^{\lossy}(\bfx, \bfx'; \; \bfw) = 
    \frac{\bfw_i(\bfx' - \bfx)_i}{\langle \bfw, \bfx'-\bfx \rangle}
    \cdot 
    \big(g(-y\langle \bfw, \bfx'\rangle) - 
           g(-y\langle \bfw,  \bfx\rangle),
    \big)
\eeqstar
and if we let $\Delta = \bfx'-\bfx$ (which satisfies that $\|\Delta\|_\infty \le \varepsilon)$, its absolute value can be written as 
  \begin{align*}
    \frac{\big|
        g(
          -y\langle \bfw, \bfx\rangle - 
           y\langle \bfw, \Delta \rangle 
           )
    \;-\; g(-y\langle \bfw, \bfx\rangle)
       \big|}
        {|\langle \bfw, \Delta \rangle|}
    \cdot |\bfw_i\Delta_i|
  \end{align*}
Let $z = -y\langle \bfw, \bfx \rangle$ and $\delta = -y\langle \bfw, \Delta\rangle$, this is further
  simplified as $\frac{|g(z+\delta)-g(z)|}{|\delta|}|w_i \Delta_i|$.
By Assumption \ref{ass-loss-cvx}, $g$ is convex, and therefore 
the "chord slope" $[g(z+\delta) - g(z)]/\delta$ cannot decrease as 
$\delta$ is increased. 
In particular  
  to maximize the $\ell_1$-norm  of the $\IG$ term in Eq (\ref{eq-stable-robust}),
  we can set $\delta$ to be largest possible value subject to the constraint   $||\Delta||_\infty \leq \varepsilon$,
  and we achieve this by setting 
  $\Delta_i=-y\sign(\bfw_i)\varepsilon$,
   for each dimension $i$.
   This yields  $\delta = \|\bfw\|_1\varepsilon$, and  
  the second term on the LHS of (\ref{eq-stable-robust}) becomes
  \begin{align*}
    |g(z+\delta) - g(z)| \cdot \frac{\sum_i|\bfw_i\Delta_i|}{|\delta|}
    &= |g(z+\varepsilon\|\bfw\|_1) - g(z)| \cdot \frac{\sum_i|\bfw_i|\varepsilon}{\|\bfw\|_1\varepsilon} \\
    &= |g(z+\varepsilon\|\bfw\|_1) - g(z)| \\
    &= g(z+\varepsilon\|\bfw\|_1) - g(z)
  \end{align*}
  where the last equality follows because $g$ is nondecreasing. 
  Since $\lossxyw = g(z)$ by Assumption \ref{ass-loss-cvx},  the LHS of (\ref{eq-stable-robust}) simplifies to
  $$g(-y\langle \bfw, \bfx \rangle + \varepsilon\|\bfw\|_1),$$
  and by Eq. (\ref{eq-adv-closed-form}), this is exactly the 
  $\linfeps$-adversarial loss on the RHS of (\ref{eq-stable-robust}).
\end{proof}

\section{Aggregate IG Attribution over a Dataset} \label{app-agg-ig}

Recall that in Section \ref{sec-attr}
we defined $\igfxu$ in Eq. (\ref{eq-ig-def})
for a \textit{single} input  $\bfx$ (relative to a baseline input $\bfu$). This gives us a sense of the "importance" of each input feature in explaining a \textit{specific} model prediction $F(\bfx)$.
Now we describe some ways to produce \textit{aggregate} importance metrics over an entire dataset. For brevity let us simply write $\IG(\bfx)$ and $\IG_i(\bfx)$ and omit $F$ and $\bfu$ since these are fixed for a given model and a given dataset.

Note that in Eq. \ref{eq-ig-def},
$\bfx$ is assumed to be an input vector in "exploded" space, i.e.,
all categorical features are (explicitly or implicitly) one-hot encoded, 
and 
$i$ is the position-index corresponding to either a specific numerical feature, or a categorical feature-\textit{value}. 
Thus if $i$ corresponds to a categorical feature-value, then for any input $\bfx$ where $x_i = 0$  (i.e. the corresponding categorical feature-value is not "active" for that input), $\IG_i(\bfx) = 0$.
A natural definition of the overall importance of a feature (or feature-value) $i$ for a given model $F$ and dataset $\mathcal{D}$, is the average of $|\IG_i(\bfx)|$ over all inputs $\bfx \in \mathcal{D}$,  which we refer to as the \textbf{Feature Value Impact} $FV_i[\mathcal{D}]$.
For a categorical feature with $m$ possible values, we can further define its \textbf{Feature-Impact} (FI) as the \textit{sum} of 
$FV_i[\mathcal{D}]$ over all $i$ corresponding to possible values of this categorical feature. 

The FI metric is particularly useful in tabular datasets to gain an understanding of the aggregate importance of high-cardinality categorical features. 

\section{Definition of the Gini Index} \label{app-gini}
The definition is adapted from \cite{Hurley2009-iw}: Suppose we are given 
a vector of non-negative values 
$\bfv = [v_1, v_2, v_3, \ldots, v_d]$.
The vector is first \textit{sorted} in non-decreasing order, so that 
the resulting indices after sorting are $(1), (2), (3), \ldots, (d)$, 
i.e., $v_{(k)}$ denotes the $k$'th value in this sequence.
Then the Gini Index is given by:
\begin{equation}
	G(\bfv) = 1 - 2 \sum_{k=1}^d \frac{v_{(k)}}{||\bfv||_1} 
	              \Big ( \frac{d - k + 0.5}{d} \Big). \label{eq-gini}
\end{equation}
Another equivalent definition of the Gini Index is based on plotting the cumulative fractional contribution of the sorted values. In particular if the sorted non-negative values are $[v_{(1)}, v_{(2)}, \ldots, v_{(d)}]$, 
 and for  $k \in [d]$, we 
 plot $k/d$ (the fraction of dimensions up to $k$) vs 
 $\frac{\sum_{i=1}^k v_(i)}{||\bfv||_1}$ (the fraction of values until the $k$'th dimension),
 then the Gini Index $G(\bfv)$ is 0.5 minus the area under this curve

The Gini Index by definition lies in [0,1], and a higher value indicates more sparseness. 
For example if just one of the $v_i > 0$ and all the rest are 0, then $G(v) = 1.0$, indicating perfect sparseness.
At the other extreme, if all $v_i$ are equal to some positive constant, 
then $G(v) = 0$.

\section{Experiments}

\subsection{Experiment Datasets and Methodology}
\label{sec-datasets}

We experiment with 5 public benchmark datasets. Below we briefly describe each dataset and model-training details.

\textbf{MNIST.} 
This is a classic image benchmark dataset consisting of grayscale images of handwritten digits 0 to 9 in the form of 28 x 28 pixels, along with the correct class label (0 to 9) \cite{lecun-mnisthandwrittendigit-2010}. 
We train a Deep Neural Network consisting of two convolutional layers with 32 and 64 filters respectively,
each followed by 2x2 max-pooling, and a fully connected layer of size 1024.
Note that this is identical to the state-of-the-art adversarially trained model used by \cite{Madry2017-zz}.
We use 50,000 images for training, and 10,000 images for testing. When computing the IG vector for an input image, we use the predicted probability of the \textit{true class} as the function $F$ in the definition (\ref{eq-ig-def}) of IG. 
For training each of the model types on MNIST, we use the Adam optimizer with a learning rate $10^{-4}$, with a batch size of 50. For the naturally-trained model (with or without $\ell_1$-regularization)
we use 25,000 training steps. For adversarial training, we use 100,000 training steps overall,
and to generate adversarial examples we use Projected Gradient Descent (PGD) with random start.
The PGD hyperparameters depend on the specific $\varepsilon$ bound on the $\linf$-norm of the adversarial perturbations: the number of PGD steps was set as $\varepsilon * 100 + 10$,
and the PGD step size was set to $0.01$.

\textbf{Fashion-MNIST.} This is another image benchmark dataset which is a drop-in replacement for MNIST \cite{Xiao2017-zp}. Images in this dataset depict wearables
such as shirts and boots instead of digits. The image format,
the number of classes, as well as the number of train/test examples are
all identical to MNIST. We use the same model and training details as for MNIST.

\textbf{CIFAR-10.} The CIFAR-10 dataset~\cite{krizhevsky2009learning} is a dataset of 32x32 color images with ten classes, each consisting of 5,000 training images and 1,000 test images. The classes correspond to dogs, frogs, ships, trucks, etc. The pixel values are in range of $[0, 255]$. We use a wide Residual Network~\cite{he2016deep}, which is identical to the state-of-art adversarially trained model on CIFAR-10 used by \cite{Madry2017-zz}. When computing the IG vector for an input image, we use the predicted probability of the \textit{true class} as the function $F$ in the definition (\ref{eq-ig-def}) of IG.  For training each of the model types on CIFAR-10, we use Momentum Optimizer with weight decay. We set momentum rate as 0.9, weight decay rate as 0.0002, batch size as 128, and training steps as 70,000. We use learning rate schedule: the first 40000 steps, we use learning rate of $10^{-1}$; after 40000 steps and before 60,000 steps, we use learning rate of $10^{-2}$; after 60,000 steps, we use learning rate of $10^{-3}$. We use Projected Gradient Descent (PGD) with random start to generate adversarial examples.  The PGD hyperparameters depend on the specific $\varepsilon$ bound on the $\linf$-norm of the adversarial perturbations: the number of PGD steps was set as $\varepsilon + 1$, and the PGD step size was set to $1$.

\textbf{Mushroom.} This is a standard tabular public dataset from the UCI Data Repository \cite{Dua:2017}. The dataset consists of 8142 instances, each of which corresponds to a different mushroom species, and has 22 categorical features (and no numerical features), whose cardinalities are all under 10.
 The task is to classify an instance as edible (label=1) or not (label=0).
We train a simple \textit{logistic regression} model to predict the probability that the mushroom is edible, with a 70/30 train/test split, and use a 0.5 threshold to make the final classification. 
We train the models on 1-hot encoded feature vectors, and the IG computation is on these (sparse) 1-hot vectors, with the output function $F$ being the final predicted probability.
We train logistic regression models for this dataset, and for natural model training (with or without $\ell_1$-regularization) we use the Adam optimizer with a learning rate of 0.01, batch size of 32, and 30 training epochs.
Adversarial training is similar, except that each example batch is perturbed using the closed-form expression (\ref{eq-adv-closed-form}).

\textbf{Spambase.} This is another tabular dataset from the UCI Repository, consisting of 4601 instances with 57 numerical attributes (and no categorical ones). The instances are various numerical features of a specific email, and the task is classify the email as spam (label = 1) or not (label = 0). The model and training details are similar to those for the mushroom dataset.

The code for all experiments (included along with the supplement) was written using Tensorflow 2.0. The following subsections contain results that were left out of the main body of the paper due to space constraints.

\subsection{Mushroom Dataset: Average IG-based Feature Impact}

We contrast between the weights learned by natural training and adversarial training with $\varepsilon=0.1$. 
Since all features in this dataset are categorical, many with cardinalities close to 10, 
there are too many features in the "exploded" space to allow a clean display, so we instead look at the 
average Feature Impact (FI, defined in Section \ref{app-agg-ig}) over the 
(natural, unperturbed) test dataset, see Figure \ref{fig-mushroom-attribs-bars}. 
It is worth noting that several features that have a significant impact on the naturally-trained model have essentially no impact on the adversarially trained model.

\begin{figure}[h] 
\centering
\includegraphics[scale=0.45]{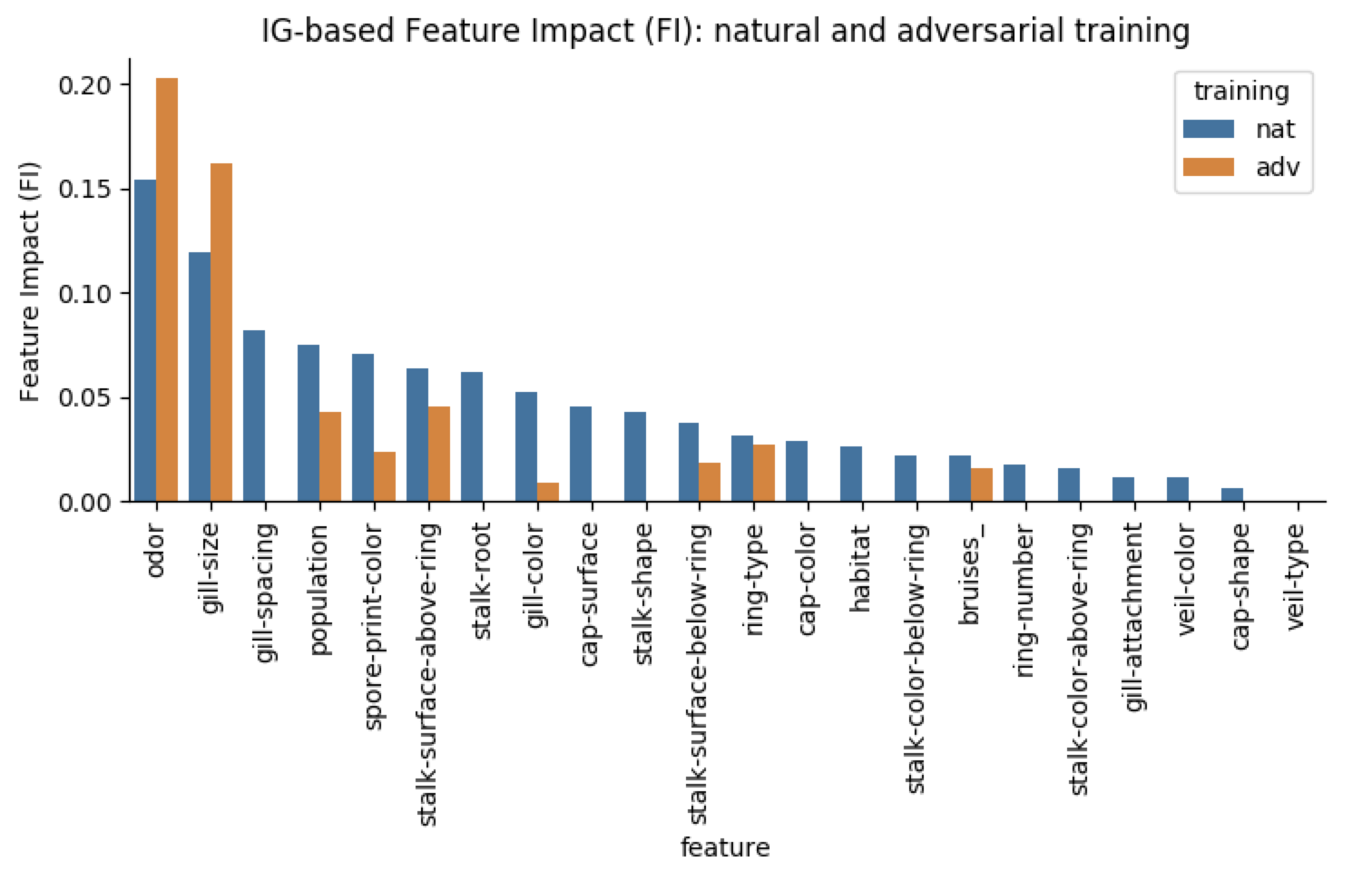}
\caption{\small
Comparison of aggregate Feature Impact (FI) for a naturally-trained model, and an adversarially-trained model 
with $\varepsilon=0.1$, on the mushroom dataset. The features are arranged left to right in decreasing order of the FI value in the naturally-trained model.
\label{fig-mushroom-attribs-bars}}
\end{figure}

\subsection{Spambase: Average IG-based Feature Impact} \label{sec-spambase}

We fix $\varepsilon=0.1$ for adversarial training and show in Figure 
\ref{fig-spambase-attribs-bars}
a bar-plot comparing the average Feature-Impacts (FI), 
between naturally-trained and adversarially-trained models.
Note how the adversarially trained model has significantly fewer features with non-negligible impacts, compared to a naturally trained model.

\begin{figure}[h] 
\centering
\includegraphics[scale=0.5]{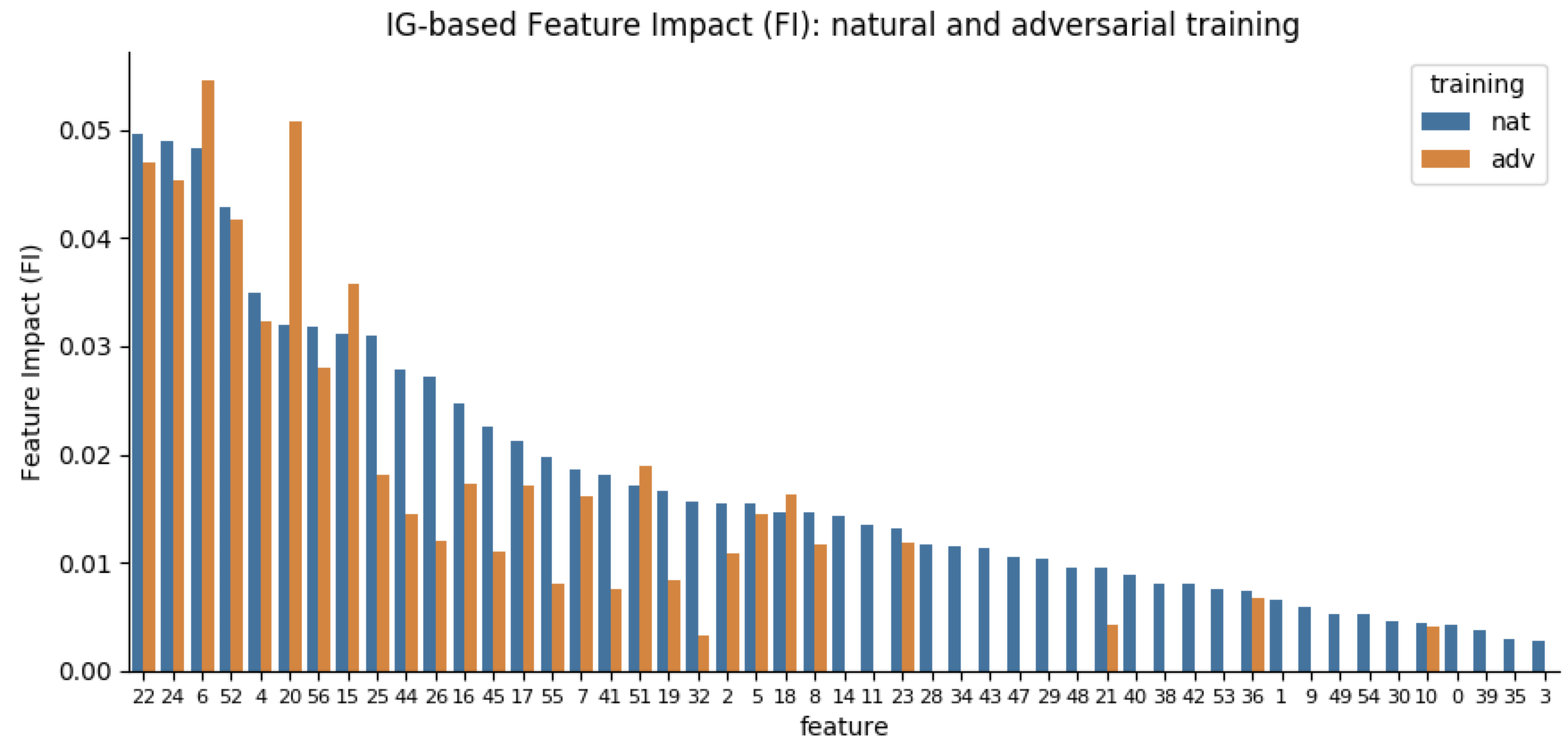}
\caption{\small
Comparison of aggregate Feature Impact (FI) for a naturally-trained model, and an adversarially-trained model 
with $\varepsilon=0.1$, on the spambase dataset.
The features are arranged left to right in decreasing order of their FI in the naturally-trained model.
To avoid clutter, we show only features that have an FI at least 5\% of the highest FI (across both models).
\label{fig-spambase-attribs-bars}}
\end{figure}

\subsection{MNIST, Fashion-MNIST and CIFAR-10: examples}
\label{sec-app-images}
Figs. \ref{fig:demo-mnist-ig}, \ref{fig:demo-fashion-mnist-ig} and \ref{fig:demo-cifar10-ig} below show IG-based saliency maps  of images correctly classified by  three model types: Naturally trained un-regularized model, 
naturally trained model with $\ell_1$-regularization, 
and an $\linfeps$-adversarially trained model.
The values of $\lambda$ and $\varepsilon$ are those indicated in Table \ref{tab-summary}.
In each example, all three models predict the correct class with high probability, 
and we compare the Gini Indices of the IG-vectors (with respect to the predicted probability of the true class).
The  sparseness of the saliency maps of the adversarially-trained models 
is visually striking compared to those of the other two models,
and this is reflected in the Gini Indices as well. 
Figs. \ref{fig:demo-mnist-shap} and \ref{fig:demo-fashion-mnist-shap} show analogous results, but using the DeepSHAP \cite{lundberg2017unified} attribution method instead of IG.
The effect of adversarial training on the sparseness of the saliency maps is even more visually striking when using DeepSHAP, compared to IG (We had difficulty running DeepSHAP on CIFAR-10 data, so we are only able to show results for DeepSHAP on MNIST and Fashion-MNIST).

\begin{figure}[htb]
    \centering 
    \begin{minipage}{\linewidth}
        \centering
         \hspace{4.5cm} \textbf{Natural Training} \hspace{0.2cm} \textbf{L1-norm Regularization} \hspace{0.1cm} \textbf{Adversarial Training}
    \end{minipage}
    \begin{subfigure}{\textwidth}
        \centering
        \begin{subfigure}{0.23\textwidth}
            \includegraphics[width=\linewidth,bb=0 0 449 464]{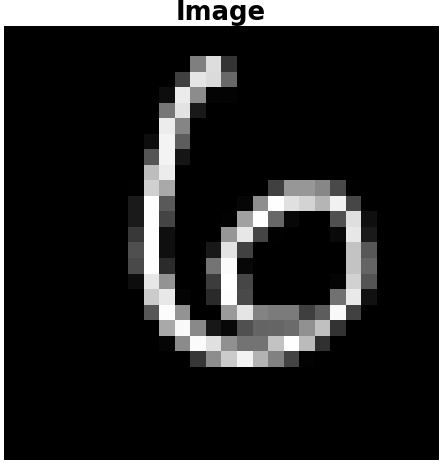} 
            \captionsetup{justification=centering}
            \caption*{}
        \end{subfigure}
        \begin{subfigure}{0.23\textwidth}
            \includegraphics[width=\linewidth,bb=0 0 449 464]{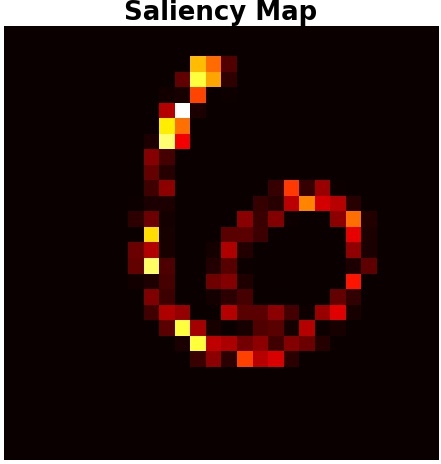} 
            \captionsetup{justification=centering}
            \caption*{Gini: 0.9271}
        \end{subfigure}
        \begin{subfigure}{0.23\textwidth}
            \includegraphics[width=\linewidth,bb=0 0 449 464]{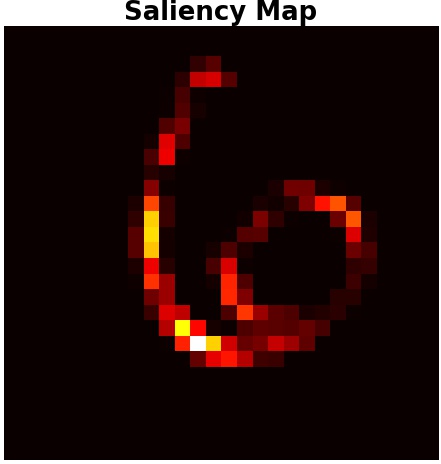} 
            \captionsetup{justification=centering}
            \caption*{Gini: 0.9266}
        \end{subfigure}
        \begin{subfigure}{0.23\textwidth}
            \includegraphics[width=\linewidth,bb=0 0 449 464]{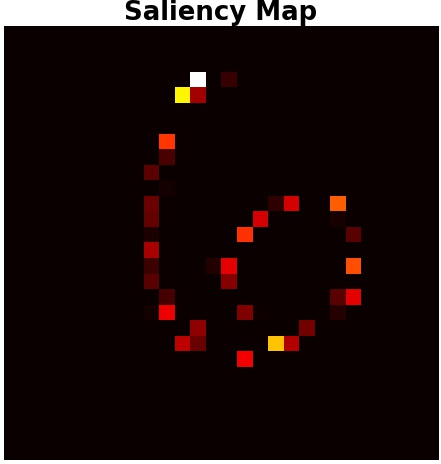} 
            \captionsetup{justification=centering}
            \caption*{Gini: 0.9728}
        \end{subfigure} 
        \caption{For all images, the models give \emph{correct} prediction -- 6.}
    \end{subfigure}
    \begin{subfigure}{\textwidth}
        \centering
        \begin{subfigure}{0.23\textwidth}
            \includegraphics[width=\linewidth,bb=0 0 449 464]{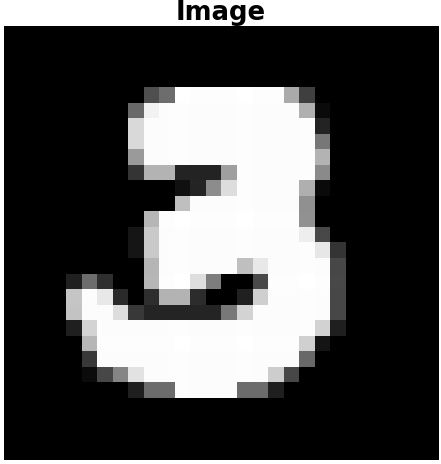} 
            \captionsetup{justification=centering}
            \caption*{}
        \end{subfigure}
        \begin{subfigure}{0.23\textwidth}
            \includegraphics[width=\linewidth,bb=0 0 449 464]{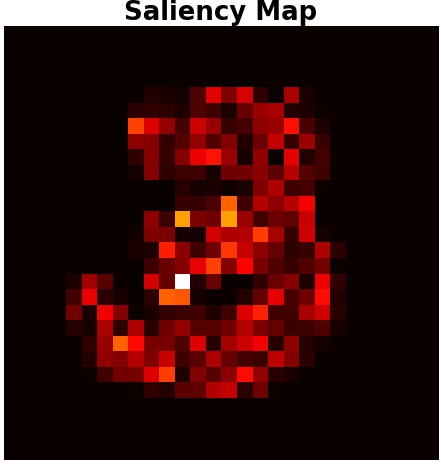} 
            \captionsetup{justification=centering}
            \caption*{Gini: 0.8112}
        \end{subfigure}
        \begin{subfigure}{0.23\textwidth}
            \includegraphics[width=\linewidth,bb=0 0 449 464]{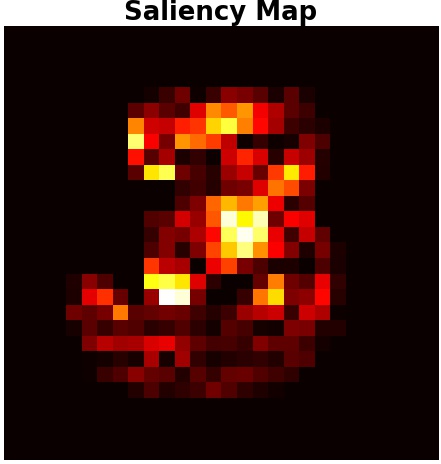} 
            \captionsetup{justification=centering}
            \caption*{Gini: 0.8356}
        \end{subfigure}
        \begin{subfigure}{0.23\textwidth}
            \includegraphics[width=\linewidth,bb=0 0 449 464]{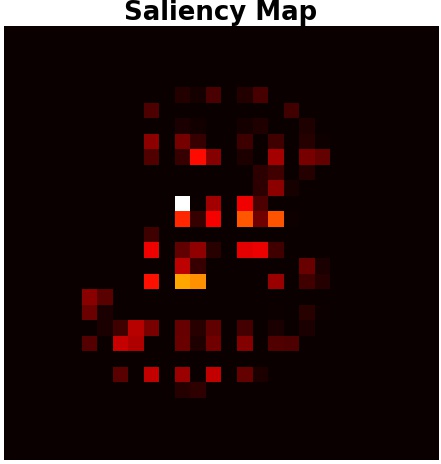} 
            \captionsetup{justification=centering}
            \caption*{Gini: 0.9383}
        \end{subfigure} 
        \caption{For all images, the models give \emph{correct} prediction -- 3.}
    \end{subfigure}
    \begin{subfigure}{\textwidth}
        \centering
        \begin{subfigure}{0.23\textwidth}
            \includegraphics[width=\linewidth,bb=0 0 449 464]{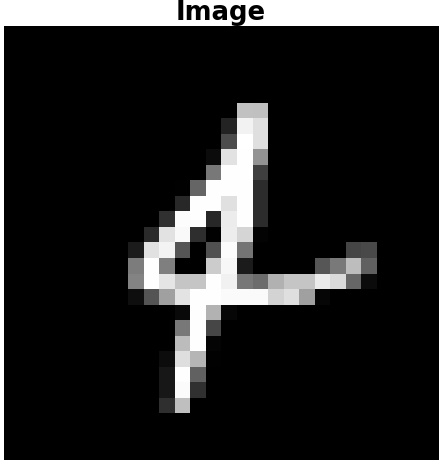} 
            \captionsetup{justification=centering}
            \caption*{}
        \end{subfigure}
        \begin{subfigure}{0.23\textwidth}
            \includegraphics[width=\linewidth,bb=0 0 449 464]{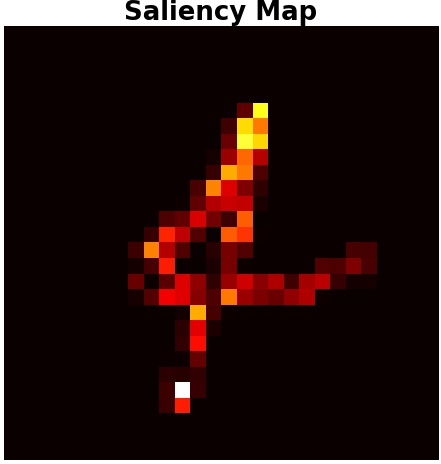} 
            \captionsetup{justification=centering}
            \caption*{Gini: 0.9315}
        \end{subfigure}
        \begin{subfigure}{0.23\textwidth}
            \includegraphics[width=\linewidth,bb=0 0 449 464]{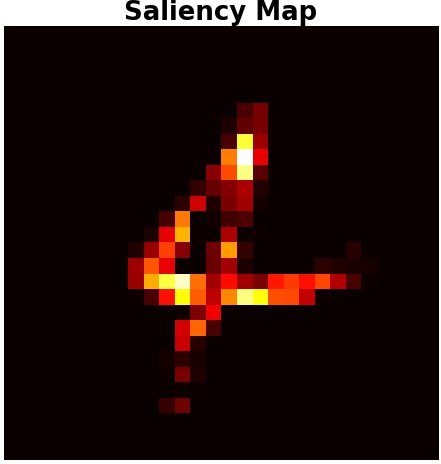} 
            \captionsetup{justification=centering}
            \caption*{Gini: 0.9366}
        \end{subfigure}
        \begin{subfigure}{0.23\textwidth}
            \includegraphics[width=\linewidth,bb=0 0 449 464]{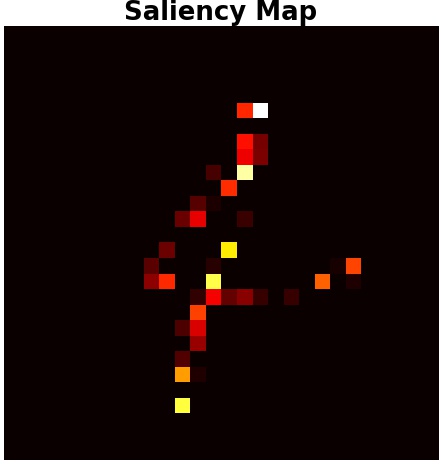} 
            \captionsetup{justification=centering}
            \caption*{Gini: 0.9738}
        \end{subfigure} 
        \caption{For all images, the models give \emph{correct} prediction -- 4.}
    \end{subfigure}
    \begin{subfigure}{\textwidth}
        \centering
        \begin{subfigure}{0.23\textwidth}
            \includegraphics[width=\linewidth,bb=0 0 449 464]{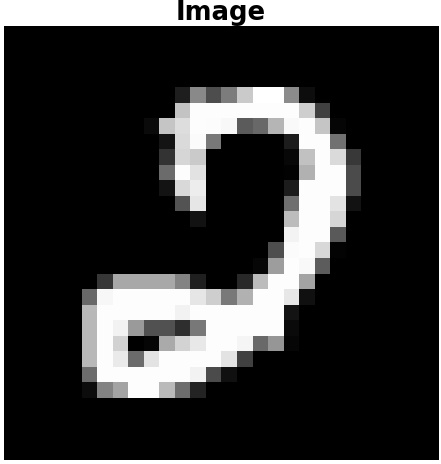} 
            \captionsetup{justification=centering}
            \caption*{}
        \end{subfigure}
        \begin{subfigure}{0.23\textwidth}
            \includegraphics[width=\linewidth,bb=0 0 449 464]{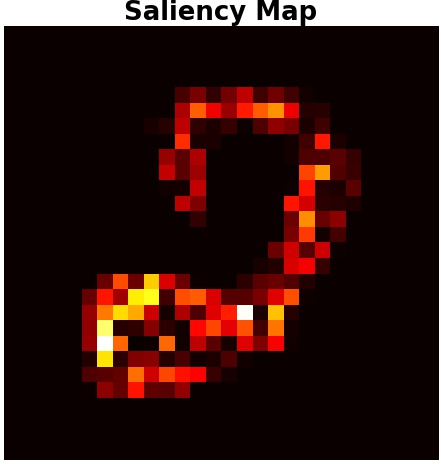} 
            \captionsetup{justification=centering}
            \caption*{Gini: 0.8843}
        \end{subfigure}
        \begin{subfigure}{0.23\textwidth}
            \includegraphics[width=\linewidth,bb=0 0 449 464]{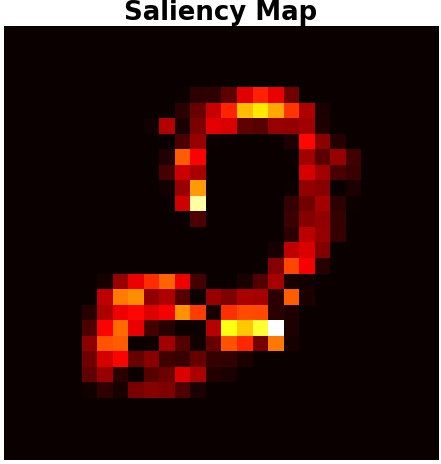} 
            \captionsetup{justification=centering}
            \caption*{Gini: 0.8807}
        \end{subfigure}
        \begin{subfigure}{0.23\textwidth}
            \includegraphics[width=\linewidth,bb=0 0 449 464]{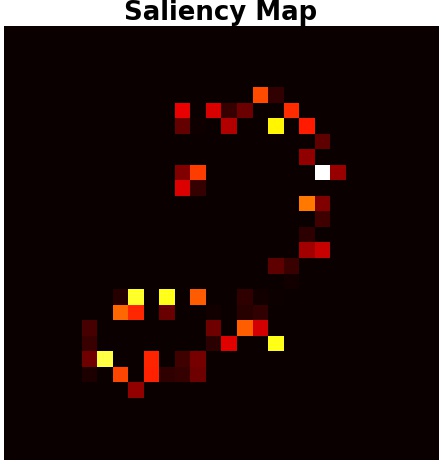} 
            \captionsetup{justification=centering}
            \caption*{Gini: 0.9595}
        \end{subfigure} 
        \caption{For all images, the models give \emph{correct} prediction -- 2.}
    \end{subfigure}
    \caption{Some examples on MNIST. We can see the saliency maps (also called feature importance maps), computed via IG, of adversarially trained model are much sparser compared to other models.}
    \label{fig:demo-mnist-ig}
\end{figure}

\begin{figure}[htb]
    \centering 
    \begin{minipage}{\linewidth}
        \centering
        \hspace{4.5cm} \textbf{Natural Training} \hspace{0.2cm} \textbf{L1-norm Regularization} \hspace{0.1cm} \textbf{Adversarial Training}
    \end{minipage}
    \begin{subfigure}{\textwidth}
        \centering
        \begin{subfigure}{0.23\textwidth}
            \includegraphics[width=\linewidth,bb=0 0 449 464]{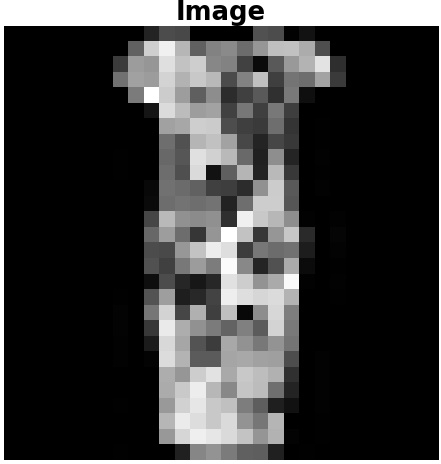} 
            \captionsetup{justification=centering}
            \caption*{}
        \end{subfigure}
        \begin{subfigure}{0.23\textwidth}
            \includegraphics[width=\linewidth,bb=0 0 449 464]{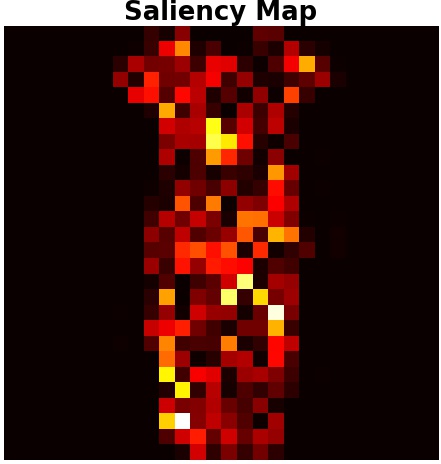} 
            \captionsetup{justification=centering}
            \caption*{Gini: 0.8190}
        \end{subfigure}
        \begin{subfigure}{0.23\textwidth}
            \includegraphics[width=\linewidth,bb=0 0 449 464]{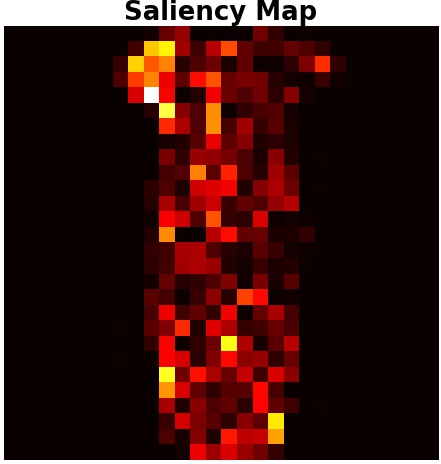} 
            \captionsetup{justification=centering}
            \caption*{Gini: 0.8183}
        \end{subfigure}
        \begin{subfigure}{0.23\textwidth}
            \includegraphics[width=\linewidth,bb=0 0 449 464]{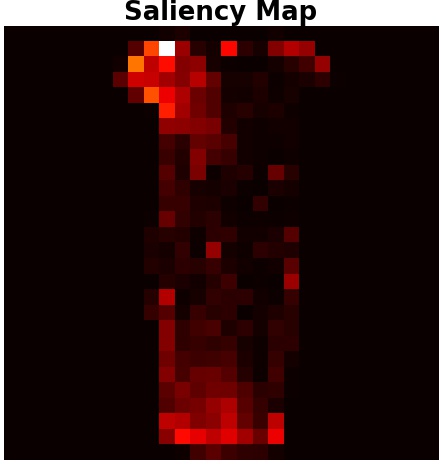} 
            \captionsetup{justification=centering}
            \caption*{Gini: 0.8532}
        \end{subfigure} 
        \caption{For all images, the models give \emph{correct} prediction -- Dress.}
    \end{subfigure}
    \begin{subfigure}{\textwidth}
        \centering
        \begin{subfigure}{0.23\textwidth}
            \includegraphics[width=\linewidth,bb=0 0 449 464]{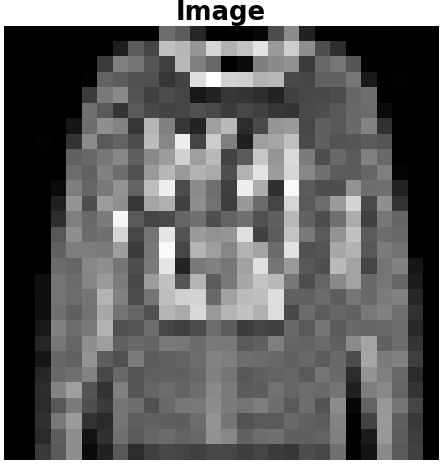} 
            \captionsetup{justification=centering}
            \caption*{}
        \end{subfigure}
        \begin{subfigure}{0.23\textwidth}
            \includegraphics[width=\linewidth,bb=0 0 449 464]{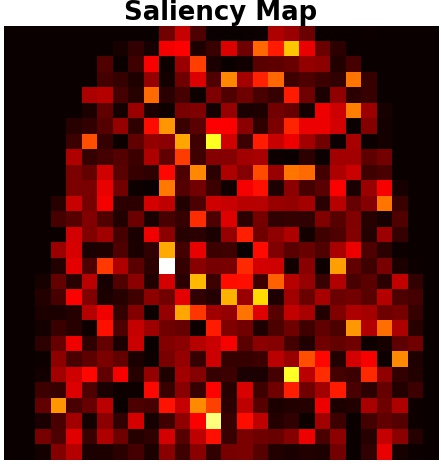} 
            \captionsetup{justification=centering}
            \caption*{Gini: 0.5777}
        \end{subfigure}
        \begin{subfigure}{0.23\textwidth}
            \includegraphics[width=\linewidth,bb=0 0 449 464]{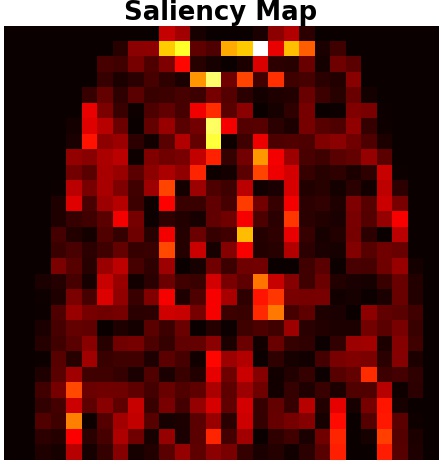} 
            \captionsetup{justification=centering}
            \caption*{Gini: 0.5925}
        \end{subfigure}
        \begin{subfigure}{0.23\textwidth}
            \includegraphics[width=\linewidth,bb=0 0 449 464]{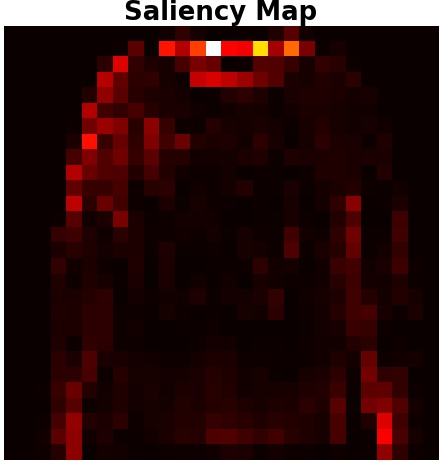} 
            \captionsetup{justification=centering}
            \caption*{Gini: 0.7024}
        \end{subfigure} 
        \caption{For all images, the models give \emph{correct} prediction -- Pullover.}
    \end{subfigure}
    \begin{subfigure}{\textwidth}
        \centering
        \begin{subfigure}{0.23\textwidth}
            \includegraphics[width=\linewidth,bb=0 0 449 464]{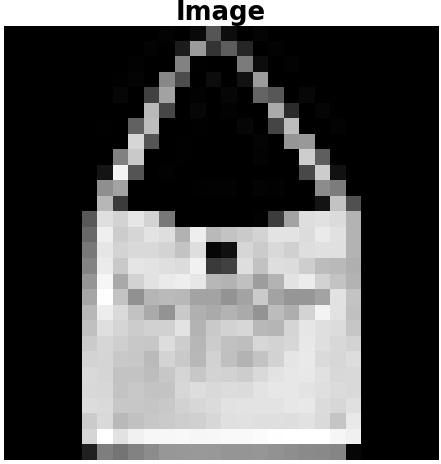} 
            \captionsetup{justification=centering}
            \caption*{}
        \end{subfigure}
        \begin{subfigure}{0.23\textwidth}
            \includegraphics[width=\linewidth,bb=0 0 449 464]{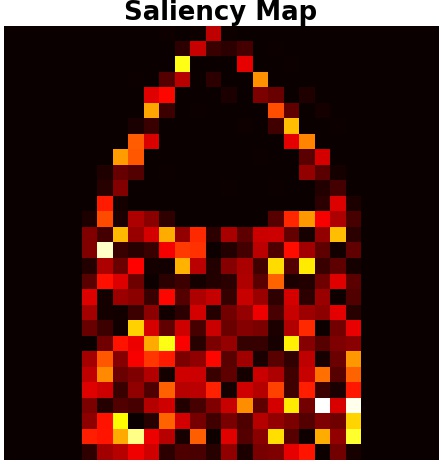} 
            \captionsetup{justification=centering}
            \caption*{Gini: 0.7698}
        \end{subfigure}
        \begin{subfigure}{0.23\textwidth}
            \includegraphics[width=\linewidth,bb=0 0 449 464]{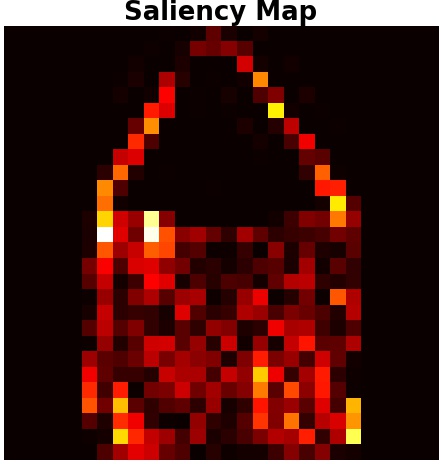} 
            \captionsetup{justification=centering}
            \caption*{Gini: 0.7784}
        \end{subfigure}
        \begin{subfigure}{0.23\textwidth}
            \includegraphics[width=\linewidth,bb=0 0 449 464]{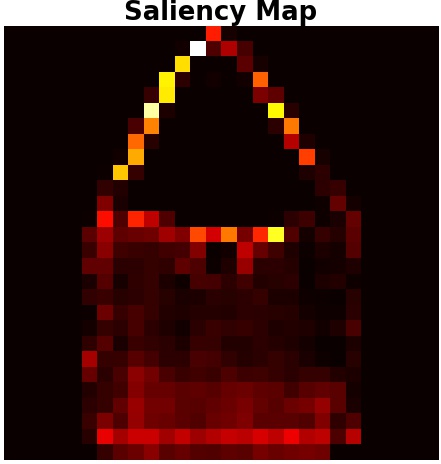} 
            \captionsetup{justification=centering}
            \caption*{Gini: 0.7981}
        \end{subfigure} 
        \caption{For all images, the models give \emph{correct} prediction -- Bag.}
    \end{subfigure}
    \begin{subfigure}{\textwidth}
        \centering
        \begin{subfigure}{0.23\textwidth}
            \includegraphics[width=\linewidth,bb=0 0 449 464]{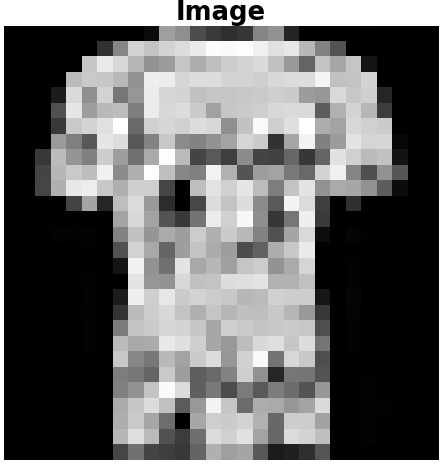} 
            \captionsetup{justification=centering}
            \caption*{}
        \end{subfigure}
        \begin{subfigure}{0.23\textwidth}
            \includegraphics[width=\linewidth,bb=0 0 449 464]{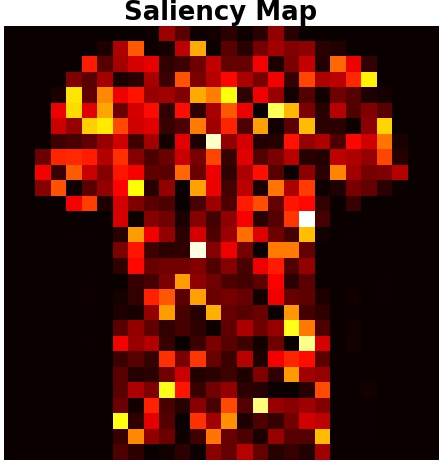} 
            \captionsetup{justification=centering}
            \caption*{Gini: 0.6840}
        \end{subfigure}
        \begin{subfigure}{0.23\textwidth}
            \includegraphics[width=\linewidth,bb=0 0 449 464]{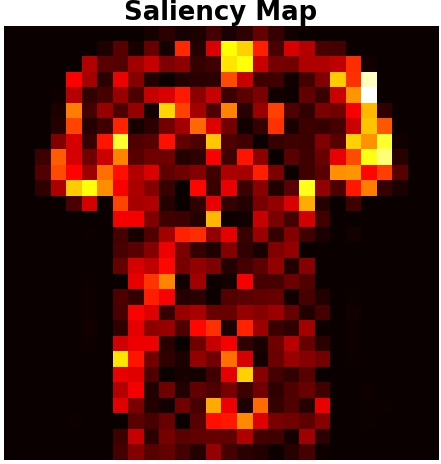} 
            \captionsetup{justification=centering}
            \caption*{Gini: 0.6899}
        \end{subfigure}
        \begin{subfigure}{0.23\textwidth}
            \includegraphics[width=\linewidth,bb=0 0 449 464]{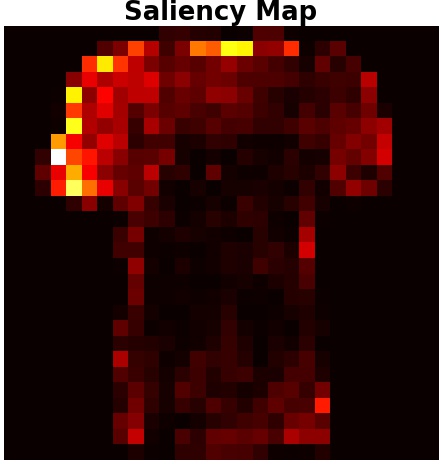} 
            \captionsetup{justification=centering}
            \caption*{Gini: 0.7503}
        \end{subfigure} 
        \caption{For all images, the models give \emph{correct} prediction -- T-shirt.}
    \end{subfigure}
    \caption{Some examples on Fashion-MNIST. We can see the saliency maps (also called feature importance maps), computed via IG, of adversarially trained model are much sparser compared to other models.}
    \label{fig:demo-fashion-mnist-ig}
\end{figure}

\begin{figure}[htb]
	\centering 
	\begin{minipage}{\linewidth}
		\centering
		\hspace{4.5cm} \textbf{Natural Training} \hspace{0.2cm} \textbf{L1-norm Regularization} \hspace{0.1cm} \textbf{Adversarial Training}
	\end{minipage}
	\begin{subfigure}{\textwidth}
		\centering
		\begin{subfigure}{0.23\textwidth}
			\includegraphics[width=\linewidth,bb=0 0 449 464]{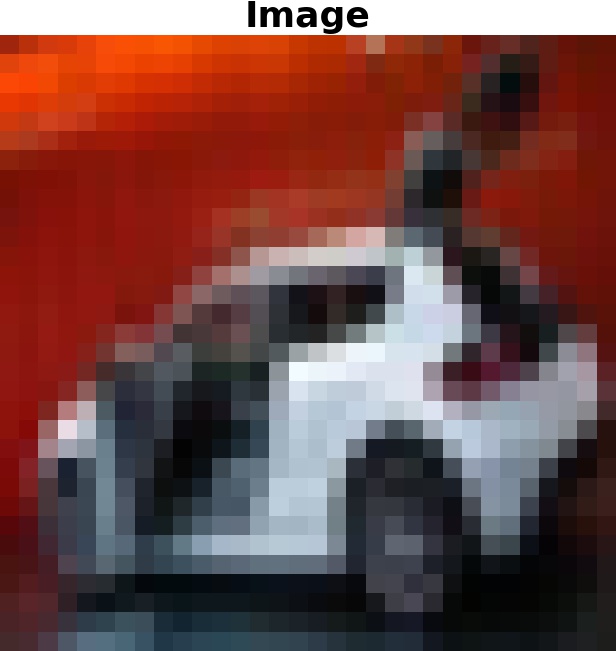} 
			\captionsetup{justification=centering}
			\caption*{}
		\end{subfigure}
		\begin{subfigure}{0.23\textwidth}
			\includegraphics[width=\linewidth,bb=0 0 449 464]{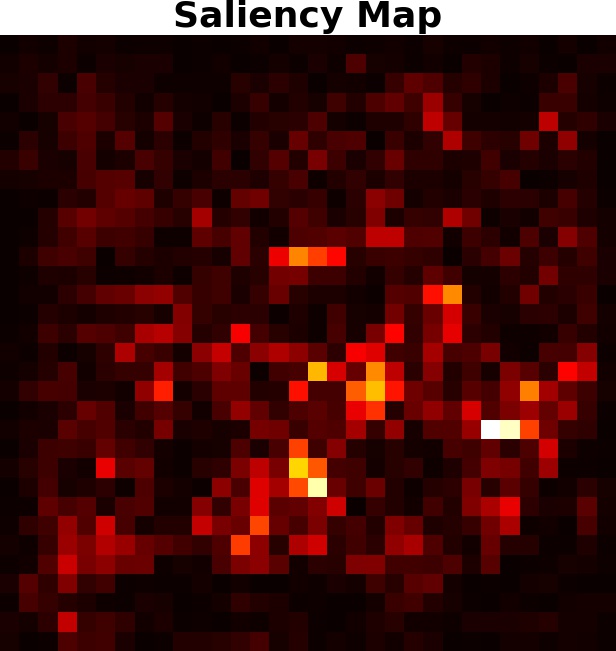} 
			\captionsetup{justification=centering}
			\caption*{Gini: 0.6395}
		\end{subfigure}
		\begin{subfigure}{0.23\textwidth}
			\includegraphics[width=\linewidth,bb=0 0 449 464]{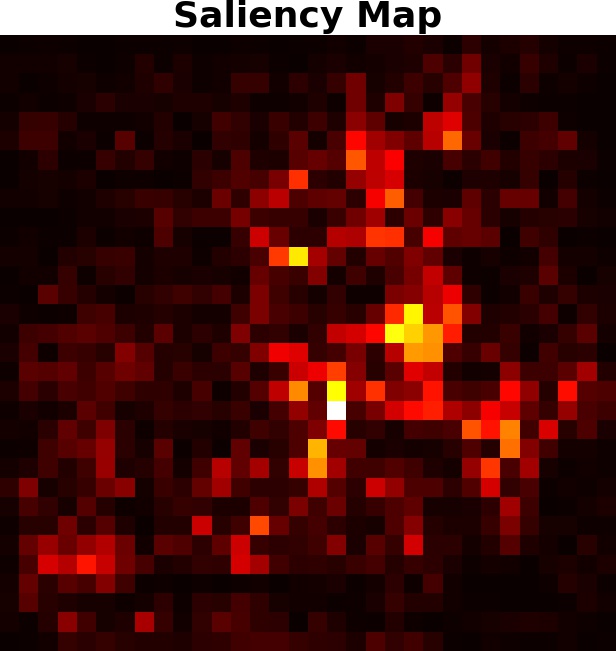} 
			\captionsetup{justification=centering}
			\caption*{Gini: 0.6556}
		\end{subfigure}
		\begin{subfigure}{0.23\textwidth}
			\includegraphics[width=\linewidth,bb=0 0 449 464]{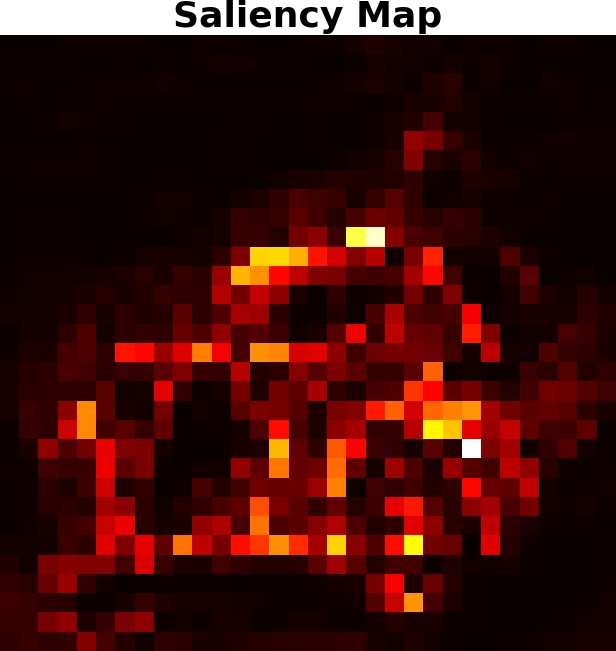} 
			\captionsetup{justification=centering}
			\caption*{Gini: 0.7309}
		\end{subfigure} 
		\caption{For all images, the models give \emph{correct} prediction -- automobile.}
	\end{subfigure}
	\begin{subfigure}{\textwidth}
		\centering
		\begin{subfigure}{0.23\textwidth}
			\includegraphics[width=\linewidth,bb=0 0 449 464]{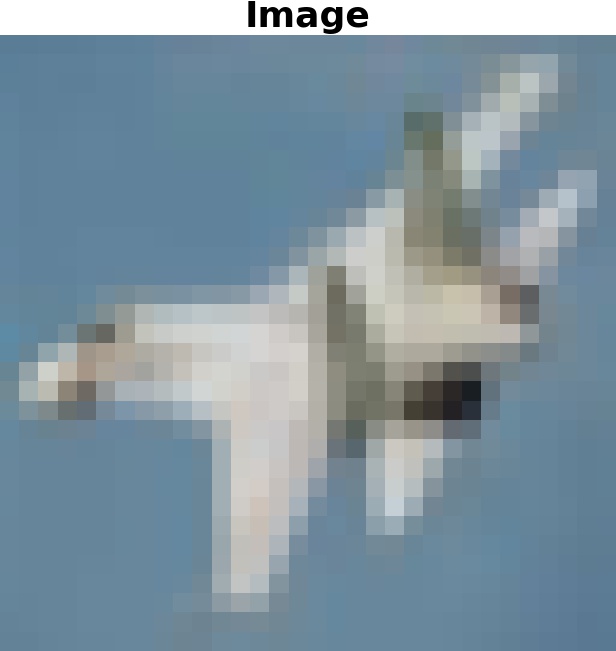} 
			\captionsetup{justification=centering}
			\caption*{}
		\end{subfigure}
		\begin{subfigure}{0.23\textwidth}
			\includegraphics[width=\linewidth,bb=0 0 449 464]{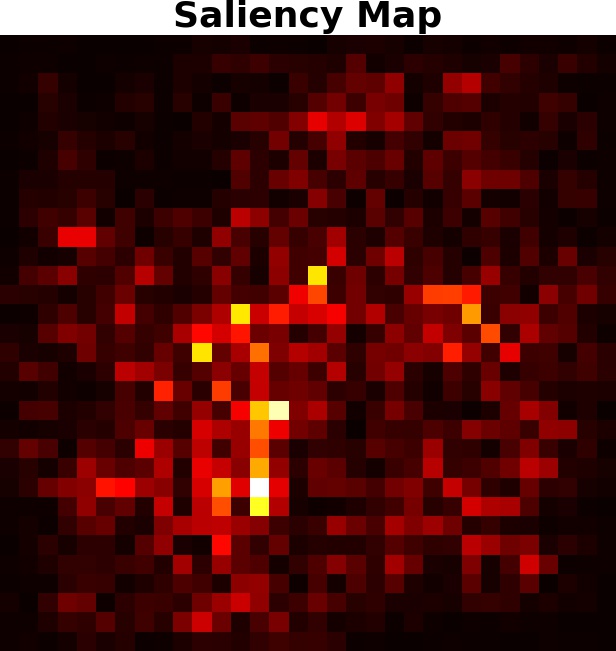} 
			\captionsetup{justification=centering}
			\caption*{Gini: 0.5736}
		\end{subfigure}
		\begin{subfigure}{0.23\textwidth}
			\includegraphics[width=\linewidth,bb=0 0 449 464]{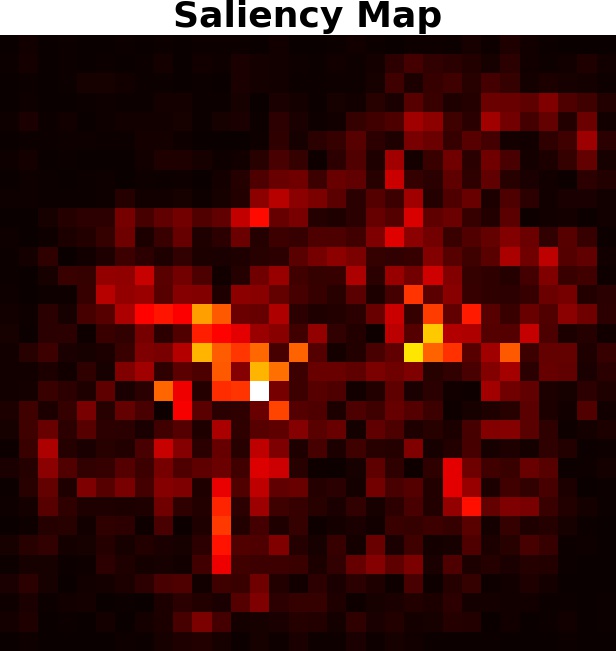} 
			\captionsetup{justification=centering}
			\caption*{Gini: 0.6175}
		\end{subfigure}
		\begin{subfigure}{0.23\textwidth}
			\includegraphics[width=\linewidth,bb=0 0 449 464]{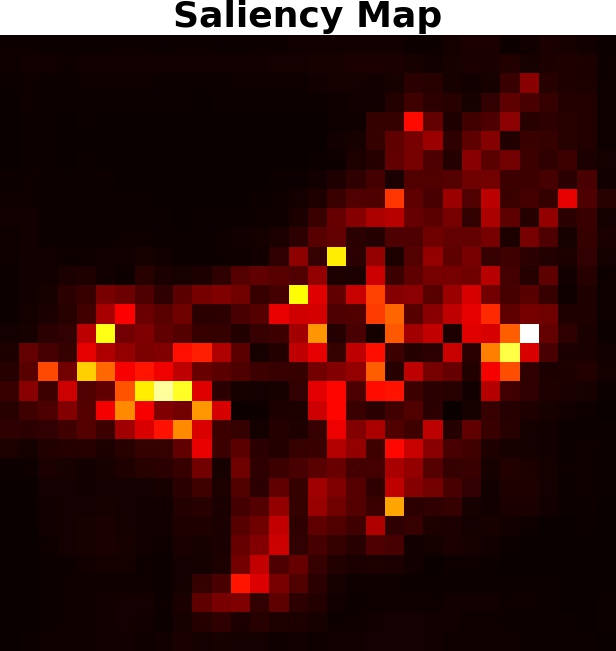} 
			\captionsetup{justification=centering}
			\caption*{Gini: 0.6914}
		\end{subfigure} 
		\caption{For all images, the models give \emph{correct} prediction -- airplane.}
	\end{subfigure}
	\begin{subfigure}{\textwidth}
		\centering
		\begin{subfigure}{0.23\textwidth}
			\includegraphics[width=\linewidth,bb=0 0 449 464]{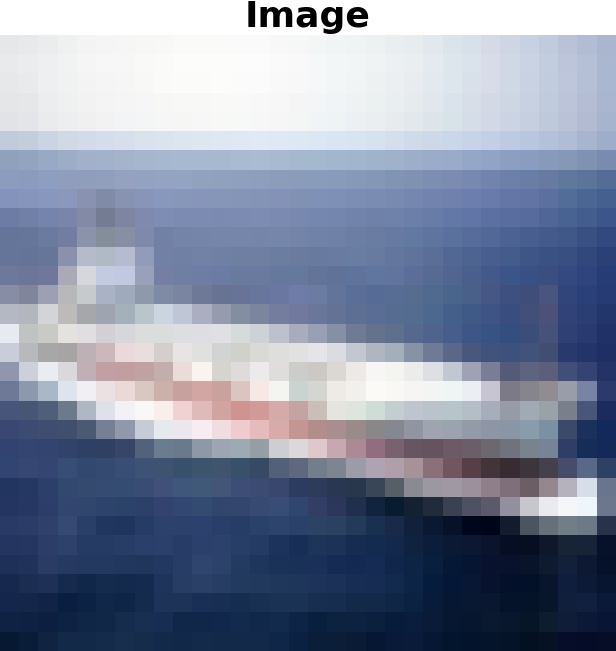} 
			\captionsetup{justification=centering}
			\caption*{}
		\end{subfigure}
		\begin{subfigure}{0.23\textwidth}
			\includegraphics[width=\linewidth,bb=0 0 449 464]{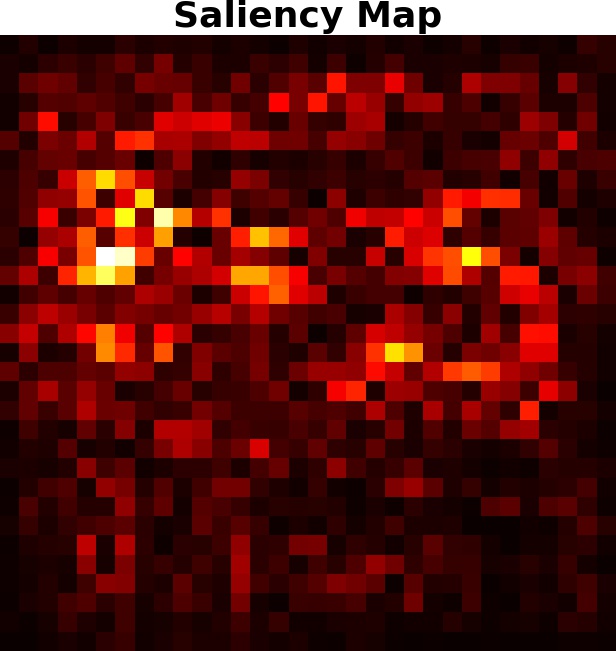} 
			\captionsetup{justification=centering}
			\caption*{Gini: 0.5505}
		\end{subfigure}
		\begin{subfigure}{0.23\textwidth}
			\includegraphics[width=\linewidth,bb=0 0 449 464]{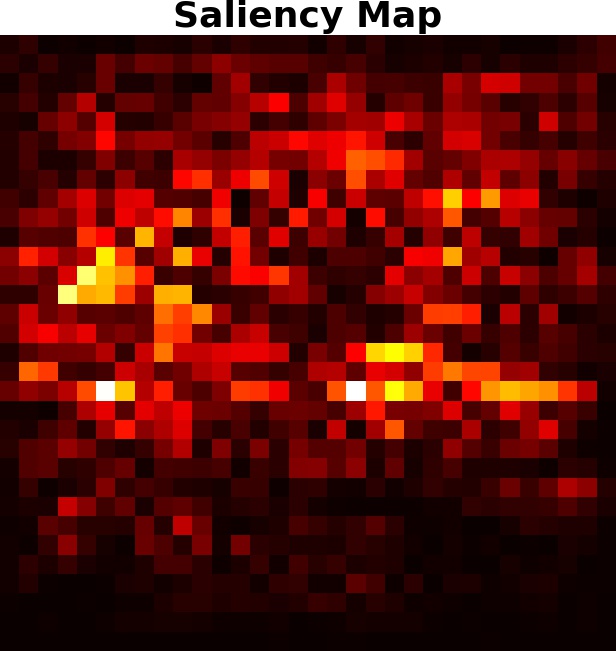} 
			\captionsetup{justification=centering}
			\caption*{Gini: 0.5844}
		\end{subfigure}
		\begin{subfigure}{0.23\textwidth}
			\includegraphics[width=\linewidth,bb=0 0 449 464]{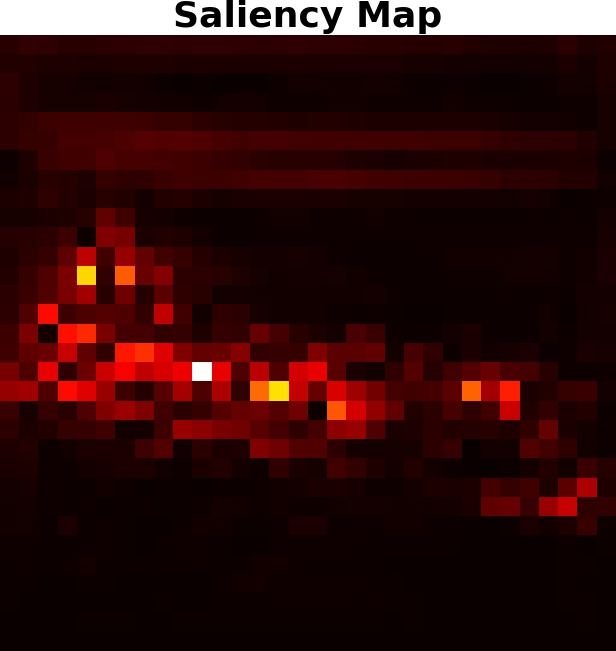} 
			\captionsetup{justification=centering}
			\caption*{Gini: 0.6778}
		\end{subfigure} 
		\caption{For all images, the models give \emph{correct} prediction -- ship.}
	\end{subfigure}
	\begin{subfigure}{\textwidth}
		\centering
		\begin{subfigure}{0.23\textwidth}
			\includegraphics[width=\linewidth,bb=0 0 449 464]{images/demo-figures/CIFAR-10/67/image.jpg} 
			\captionsetup{justification=centering}
			\caption*{}
		\end{subfigure}
		\begin{subfigure}{0.23\textwidth}
			\includegraphics[width=\linewidth,bb=0 0 449 464]{images/demo-figures/CIFAR-10/67/nat_saliency_map.jpg} 
			\captionsetup{justification=centering}
			\caption*{Gini: 0.6150}
		\end{subfigure}
		\begin{subfigure}{0.23\textwidth}
			\includegraphics[width=\linewidth,bb=0 0 449 464]{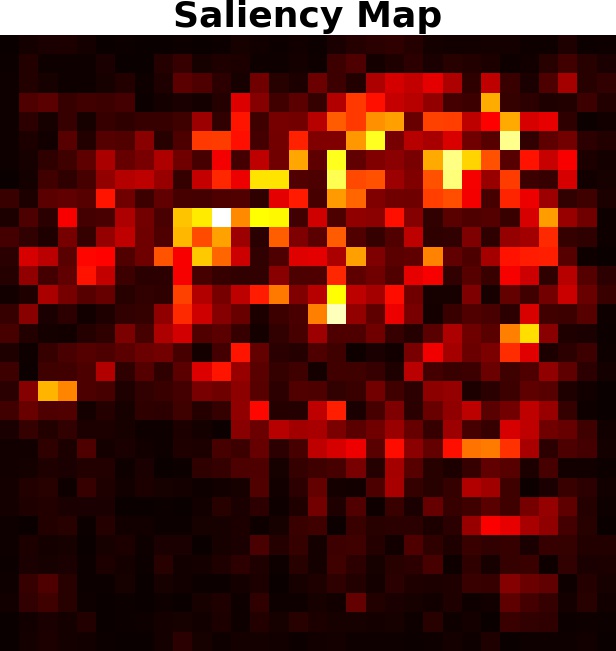} 
			\captionsetup{justification=centering}
			\caption*{Gini: 0.6198}
		\end{subfigure}
		\begin{subfigure}{0.23\textwidth}
			\includegraphics[width=\linewidth,bb=0 0 449 464]{images/demo-figures/CIFAR-10/67/adv_saliency_map.jpg} 
			\captionsetup{justification=centering}
			\caption*{Gini: 0.7084}
		\end{subfigure} 
		\caption{For all images, the models give \emph{correct} prediction -- bird.}
	\end{subfigure}
	\caption{Some examples on CIFAR-10. We can see the saliency maps (also called feature importance maps), computed via IG, of adversarially trained model are much sparser compared to other models.}
	\label{fig:demo-cifar10-ig}
\end{figure}

\begin{figure}[htb]
	\centering 
	\begin{minipage}{\linewidth}
		\centering
		\hspace{4.5cm} \textbf{Natural Training} \hspace{0.2cm} \textbf{L1-norm Regularization} \hspace{0.1cm} \textbf{Adversarial Training}
	\end{minipage}
	\begin{subfigure}{\textwidth}
		\centering
		\begin{subfigure}{0.23\textwidth}
			\includegraphics[width=\linewidth,bb=0 0 449 464]{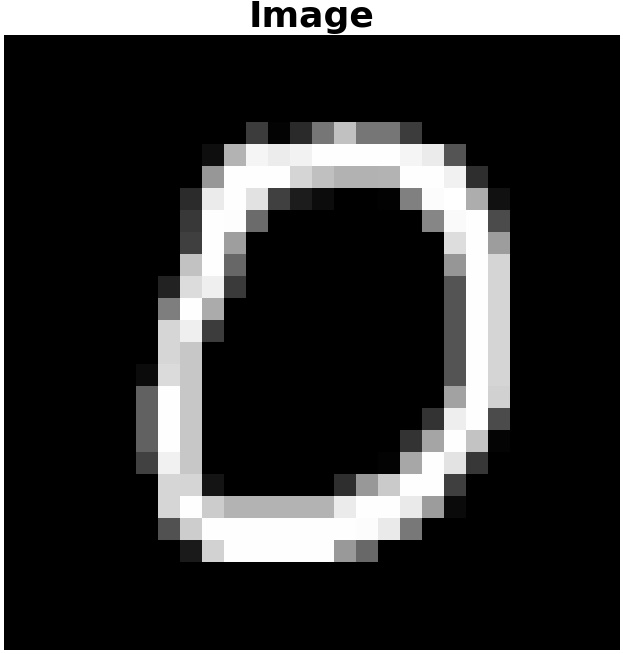} 
			\captionsetup{justification=centering}
			\caption*{}
		\end{subfigure}
		\begin{subfigure}{0.23\textwidth}
			\includegraphics[width=\linewidth,bb=0 0 449 464]{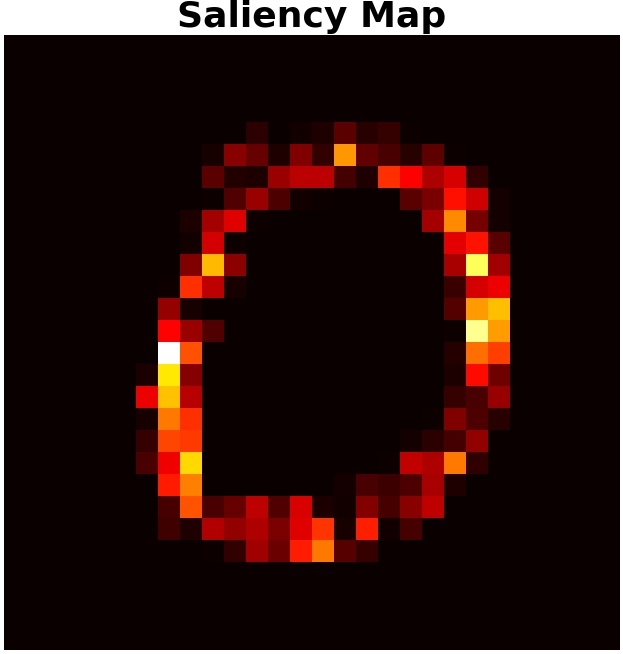} 
			\captionsetup{justification=centering}
			\caption*{Gini: 0.8982}
		\end{subfigure}
		\begin{subfigure}{0.23\textwidth}
			\includegraphics[width=\linewidth,bb=0 0 449 464]{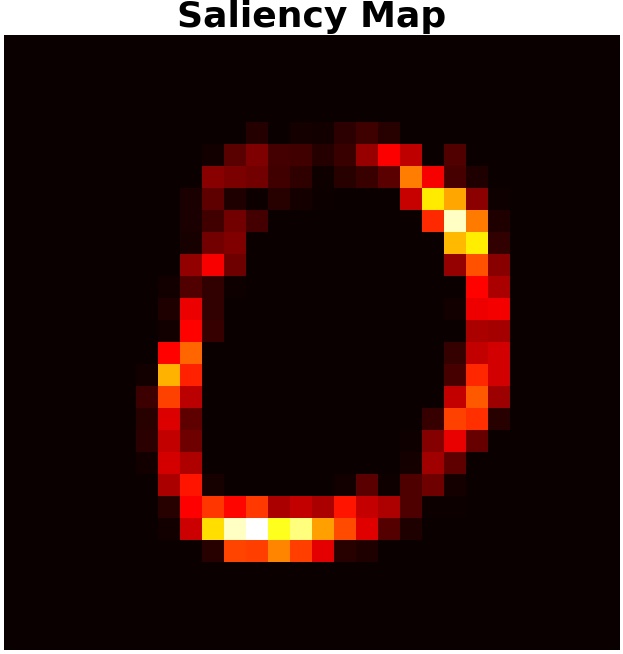} 
			\captionsetup{justification=centering}
			\caption*{Gini: 0.9000}
		\end{subfigure}
		\begin{subfigure}{0.23\textwidth}
			\includegraphics[width=\linewidth,bb=0 0 449 464]{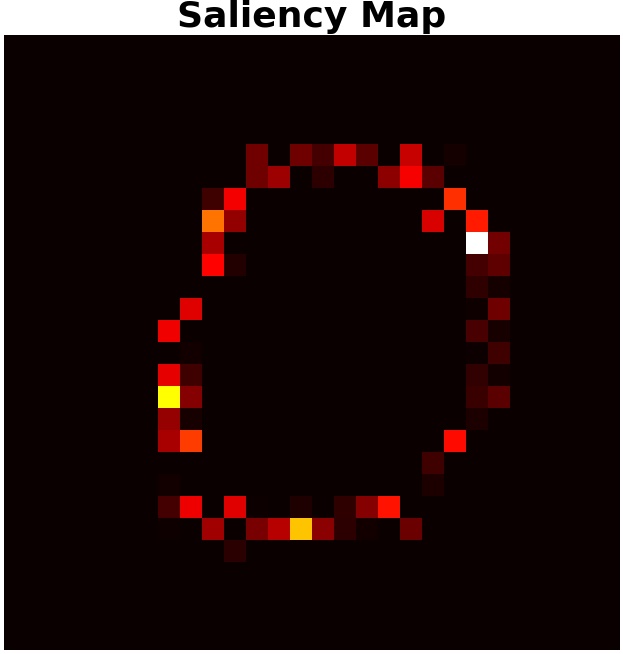} 
			\captionsetup{justification=centering}
			\caption*{Gini: 0.9528}
		\end{subfigure} 
		\caption{For all images, the models give \emph{correct} prediction -- 0.}
	\end{subfigure}
	\begin{subfigure}{\textwidth}
		\centering
		\begin{subfigure}{0.23\textwidth}
			\includegraphics[width=\linewidth,bb=0 0 449 464]{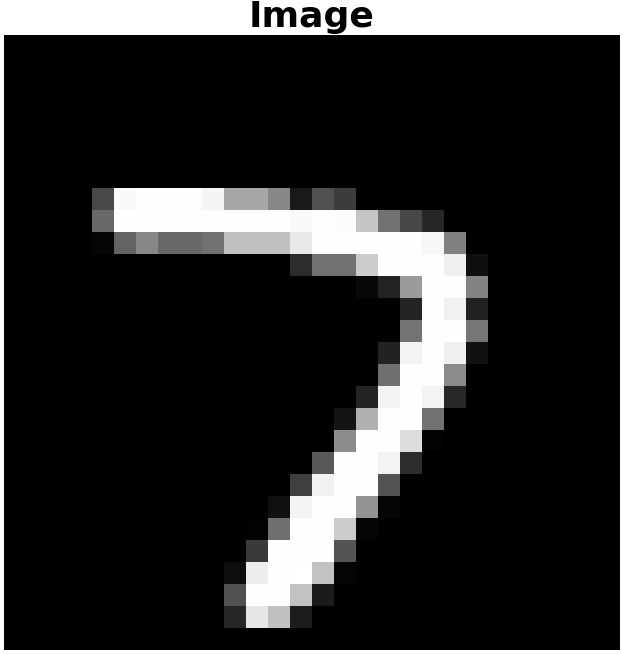} 
			\captionsetup{justification=centering}
			\caption*{}
		\end{subfigure}
		\begin{subfigure}{0.23\textwidth}
			\includegraphics[width=\linewidth,bb=0 0 449 464]{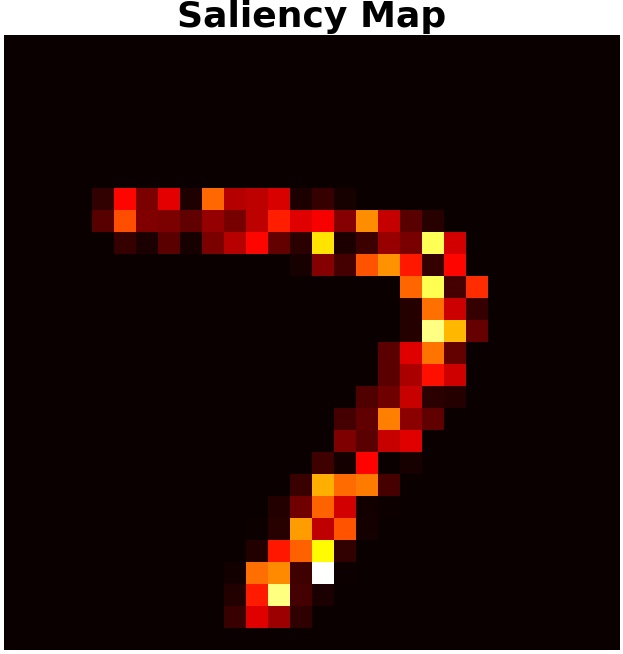} 
			\captionsetup{justification=centering}
			\caption*{Gini: 0.9156}
		\end{subfigure}
		\begin{subfigure}{0.23\textwidth}
			\includegraphics[width=\linewidth,bb=0 0 449 464]{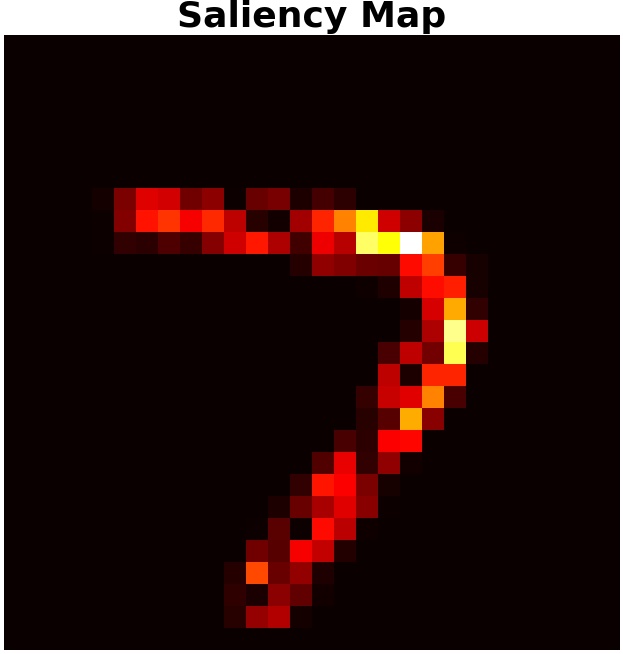} 
			\captionsetup{justification=centering}
			\caption*{Gini: 0.9156}
		\end{subfigure}
		\begin{subfigure}{0.23\textwidth}
			\includegraphics[width=\linewidth,bb=0 0 449 464]{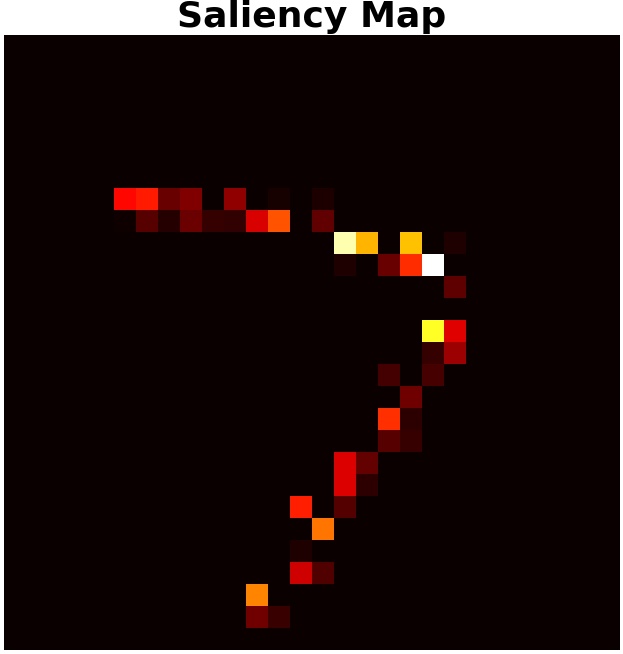} 
			\captionsetup{justification=centering}
			\caption*{Gini: 0.9685}
		\end{subfigure} 
		\caption{For all images, the models give \emph{correct} prediction -- 7.}
	\end{subfigure}
	\begin{subfigure}{\textwidth}
		\centering
		\begin{subfigure}{0.23\textwidth}
			\includegraphics[width=\linewidth,bb=0 0 449 464]{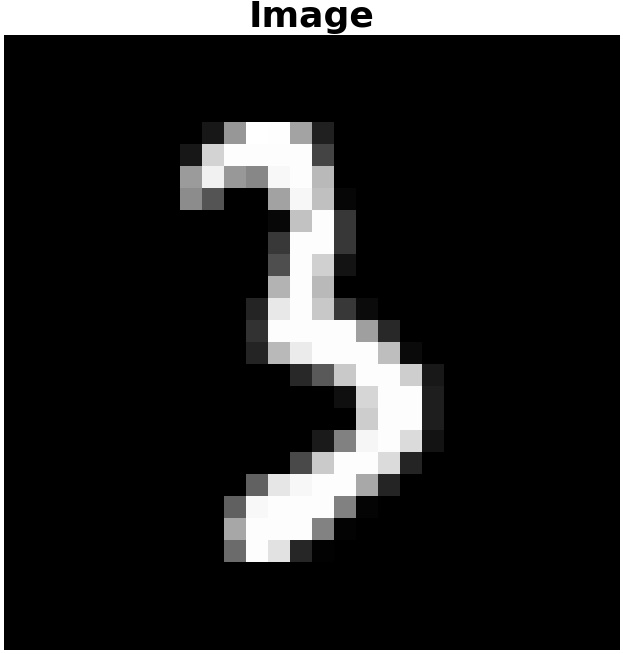} 
			\captionsetup{justification=centering}
			\caption*{}
		\end{subfigure}
		\begin{subfigure}{0.23\textwidth}
			\includegraphics[width=\linewidth,bb=0 0 449 464]{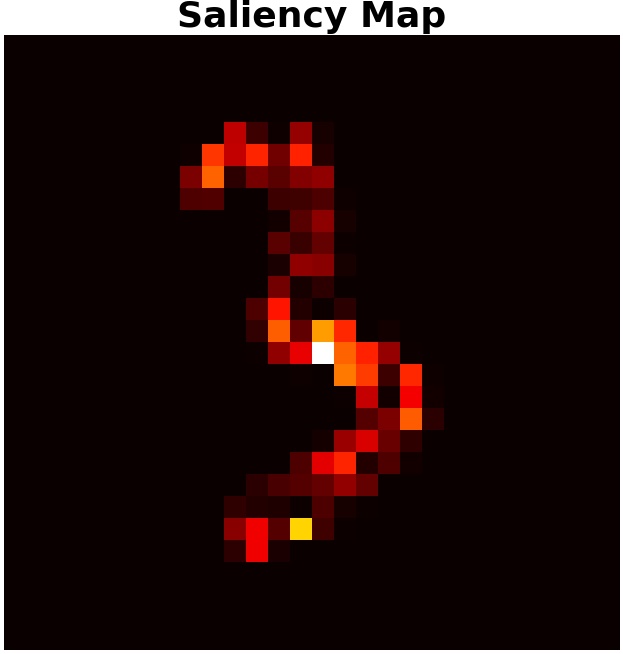} 
			\captionsetup{justification=centering}
			\caption*{Gini: 0.9373}
		\end{subfigure}
		\begin{subfigure}{0.23\textwidth}
			\includegraphics[width=\linewidth,bb=0 0 449 464]{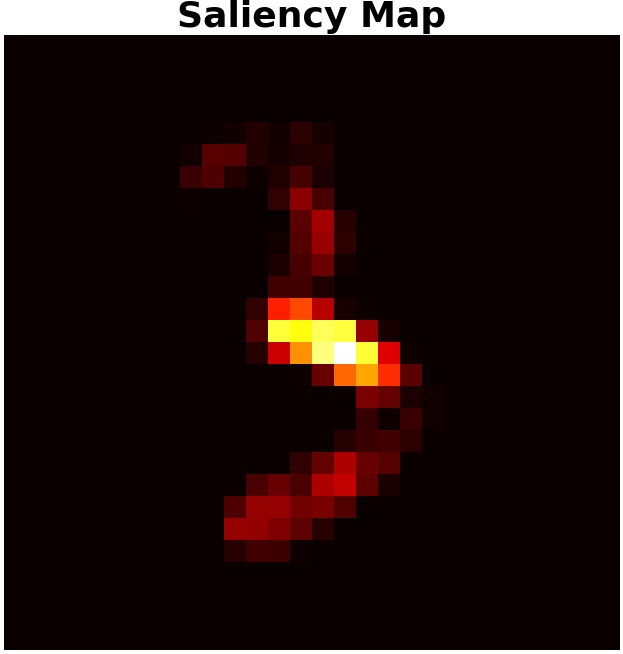} 
			\captionsetup{justification=centering}
			\caption*{Gini: 0.9452}
		\end{subfigure}
		\begin{subfigure}{0.23\textwidth}
			\includegraphics[width=\linewidth,bb=0 0 449 464]{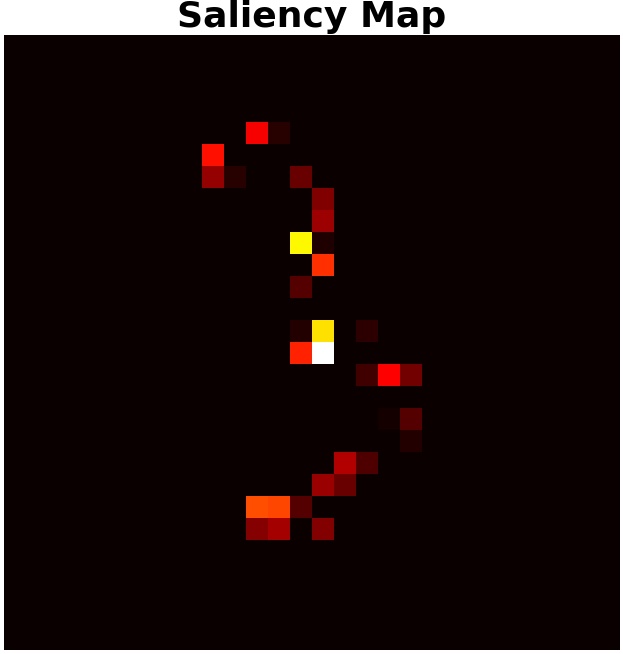} 
			\captionsetup{justification=centering}
			\caption*{Gini: 0.9773}
		\end{subfigure} 
		\caption{For all images, the models give \emph{correct} prediction -- 3.}
	\end{subfigure}
	\begin{subfigure}{\textwidth}
		\centering
		\begin{subfigure}{0.23\textwidth}
			\includegraphics[width=\linewidth,bb=0 0 449 464]{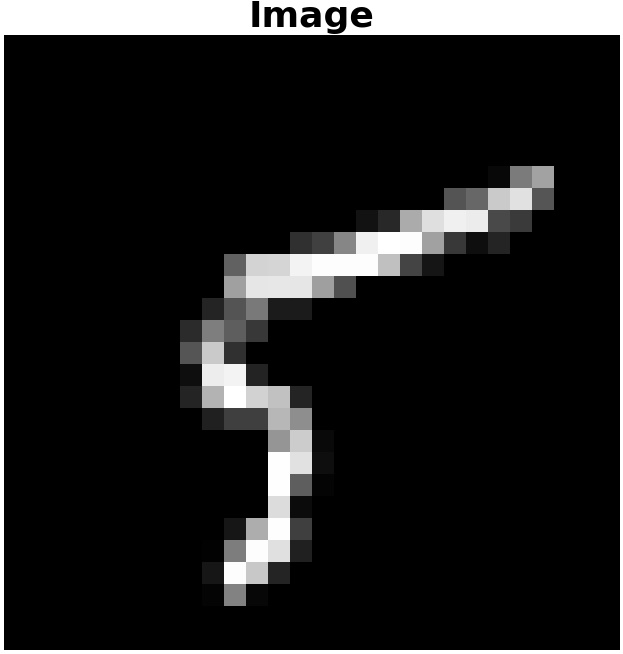} 
			\captionsetup{justification=centering}
			\caption*{}
		\end{subfigure}
		\begin{subfigure}{0.23\textwidth}
			\includegraphics[width=\linewidth,bb=0 0 449 464]{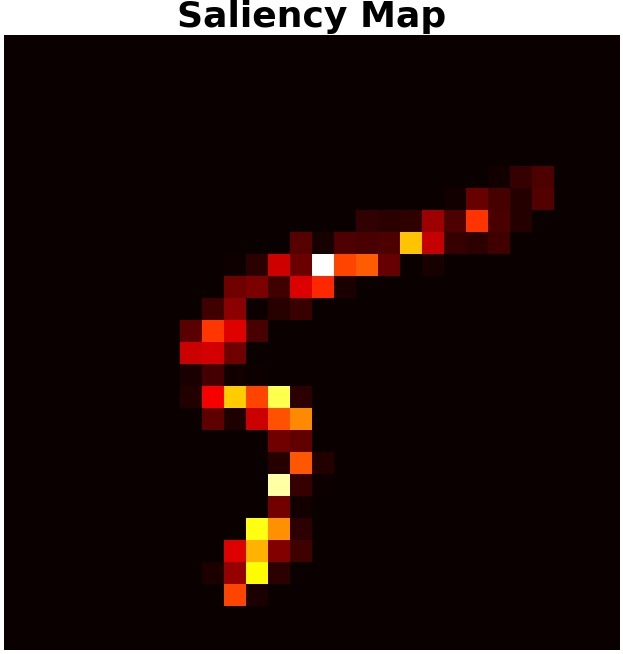} 
			\captionsetup{justification=centering}
			\caption*{Gini: 0.9476}
		\end{subfigure}
		\begin{subfigure}{0.23\textwidth}
			\includegraphics[width=\linewidth,bb=0 0 449 464]{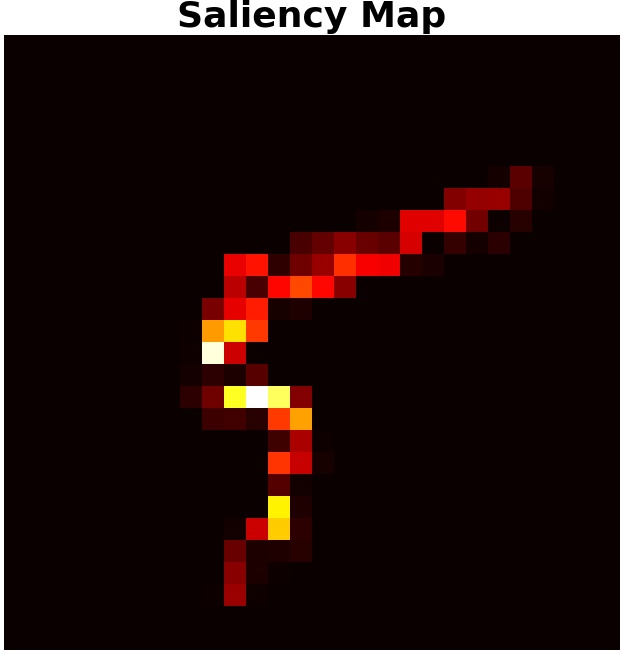} 
			\captionsetup{justification=centering}
			\caption*{Gini: 0.9473}
		\end{subfigure}
		\begin{subfigure}{0.23\textwidth}
			\includegraphics[width=\linewidth,bb=0 0 449 464]{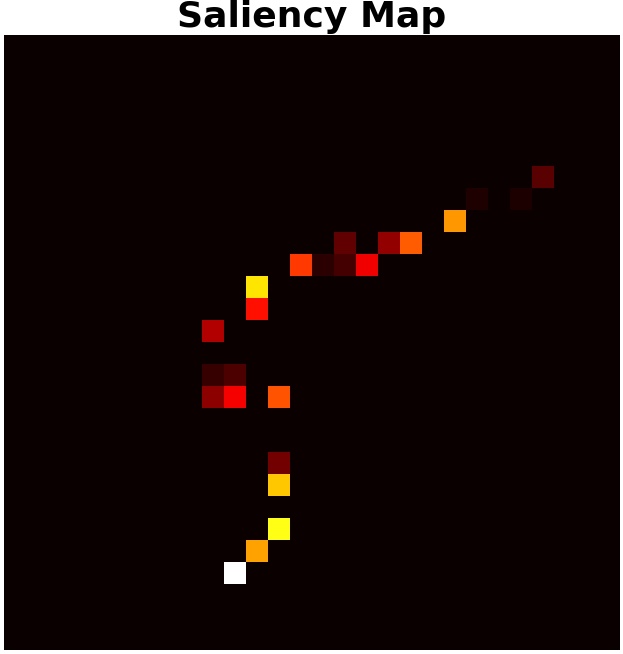} 
			\captionsetup{justification=centering}
			\caption*{Gini: 0.9825}
		\end{subfigure} 
		\caption{For all images, the models give \emph{correct} prediction -- 5.}
	\end{subfigure}
	\caption{Some examples on MNIST. We can see the saliency maps (also called feature importance maps), computed via DeepSHAP, of adversarially trained model are much sparser compared to other models.}
	\label{fig:demo-mnist-shap}
\end{figure}

\begin{figure}[htb]
	\centering 
	\begin{minipage}{\linewidth}
		\centering
		\hspace{4.5cm} \textbf{Natural Training} \hspace{0.2cm} \textbf{L1-norm Regularization} \hspace{0.1cm} \textbf{Adversarial Training}
	\end{minipage}
	\begin{subfigure}{\textwidth}
		\centering
		\begin{subfigure}{0.23\textwidth}
			\includegraphics[width=\linewidth,bb=0 0 449 464]{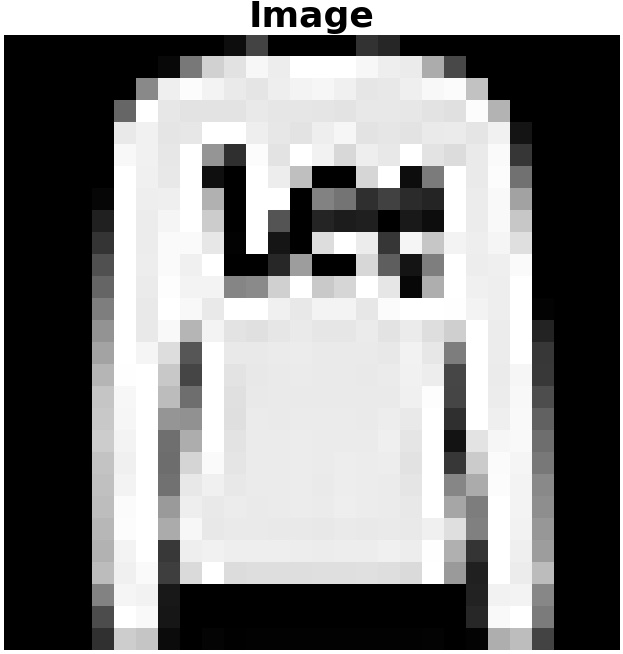} 
			\captionsetup{justification=centering}
			\caption*{}
		\end{subfigure}
		\begin{subfigure}{0.23\textwidth}
			\includegraphics[width=\linewidth,bb=0 0 449 464]{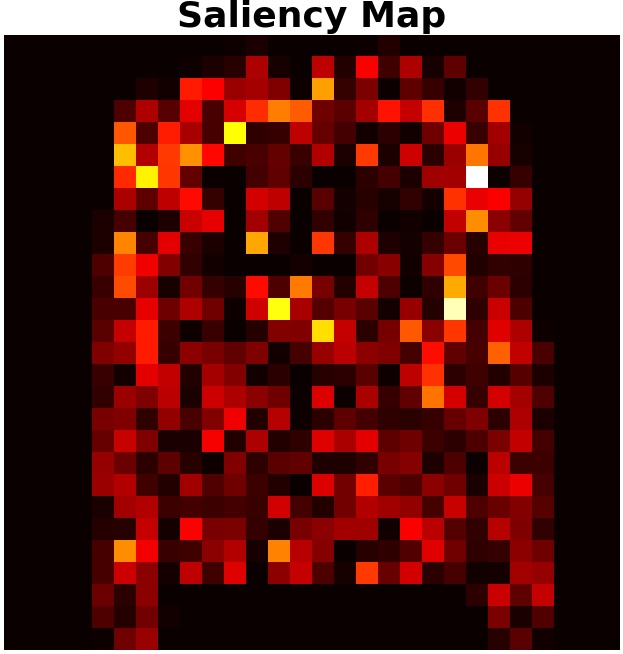} 
			\captionsetup{justification=centering}
			\caption*{Gini: 0.6749}
		\end{subfigure}
		\begin{subfigure}{0.23\textwidth}
			\includegraphics[width=\linewidth,bb=0 0 449 464]{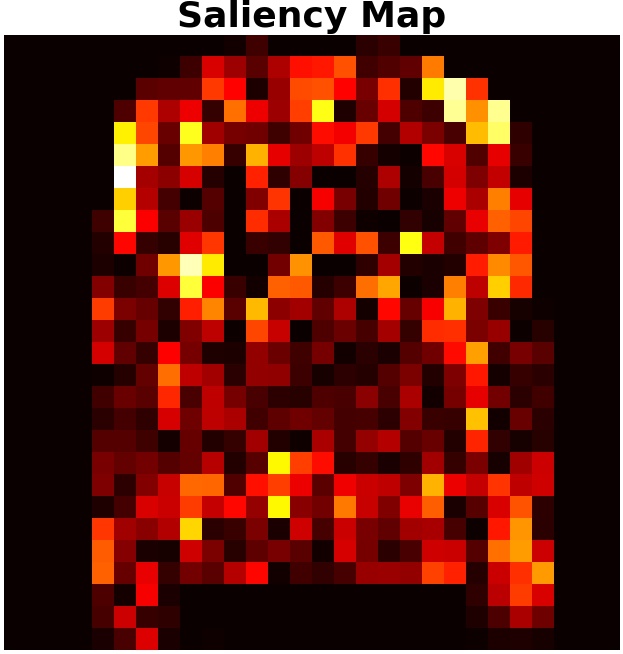} 
			\captionsetup{justification=centering}
			\caption*{Gini: 0.6676}
		\end{subfigure}
		\begin{subfigure}{0.23\textwidth}
			\includegraphics[width=\linewidth,bb=0 0 449 464]{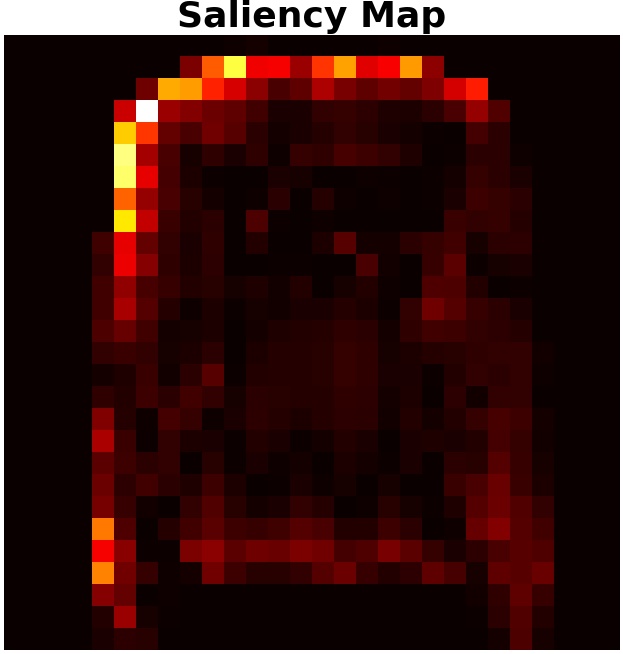} 
			\captionsetup{justification=centering}
			\caption*{Gini: 0.7435}
		\end{subfigure} 
		\caption{For all images, the models give \emph{correct} prediction -- Pullover.}
	\end{subfigure}
	\begin{subfigure}{\textwidth}
		\centering
		\begin{subfigure}{0.23\textwidth}
			\includegraphics[width=\linewidth,bb=0 0 449 464]{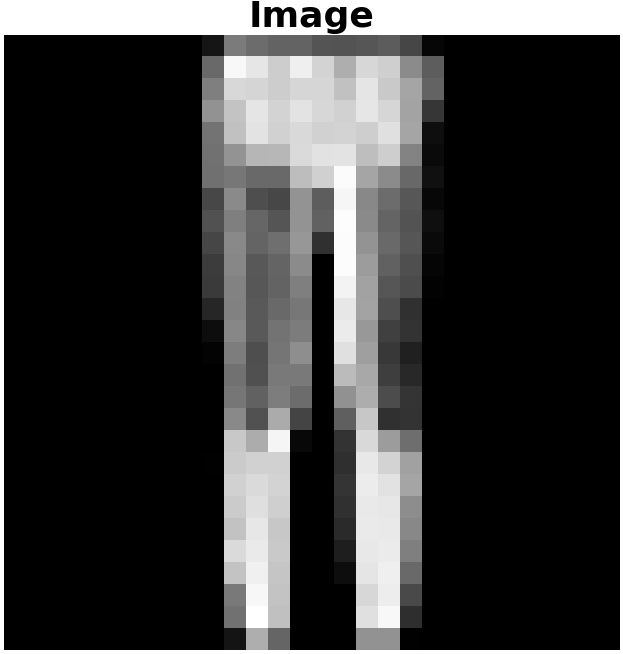} 
			\captionsetup{justification=centering}
			\caption*{}
		\end{subfigure}
		\begin{subfigure}{0.23\textwidth}
			\includegraphics[width=\linewidth,bb=0 0 449 464]{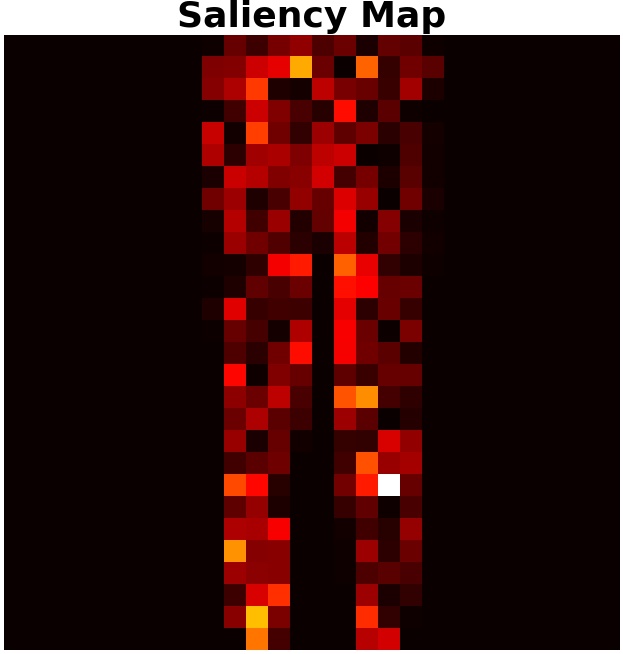} 
			\captionsetup{justification=centering}
			\caption*{Gini: 0.8322}
		\end{subfigure}
		\begin{subfigure}{0.23\textwidth}
			\includegraphics[width=\linewidth,bb=0 0 449 464]{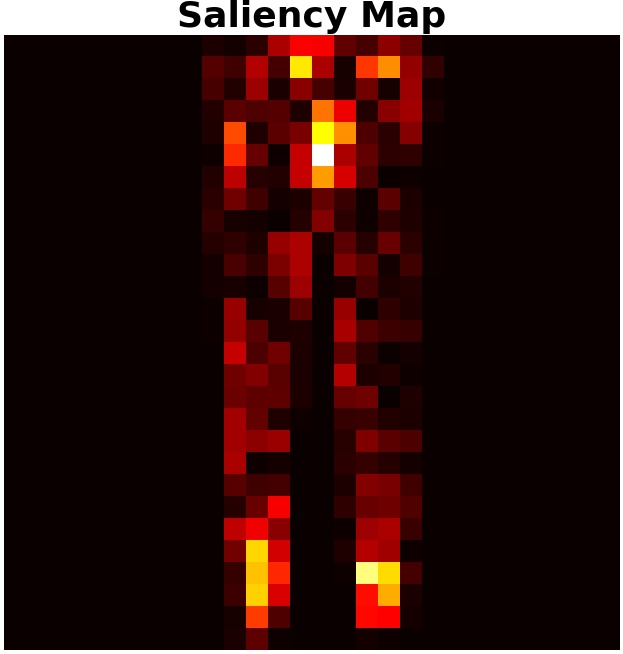} 
			\captionsetup{justification=centering}
			\caption*{Gini: 0.8628}
		\end{subfigure}
		\begin{subfigure}{0.23\textwidth}
			\includegraphics[width=\linewidth,bb=0 0 449 464]{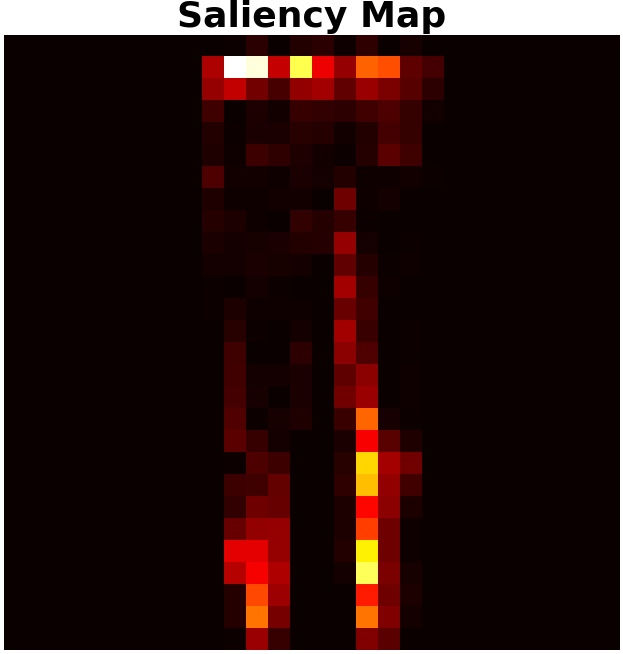} 
			\captionsetup{justification=centering}
			\caption*{Gini: 0.8953}
		\end{subfigure} 
		\caption{For all images, the models give \emph{correct} prediction -- Trouser.}
	\end{subfigure}
	\begin{subfigure}{\textwidth}
		\centering
		\begin{subfigure}{0.23\textwidth}
			\includegraphics[width=\linewidth,bb=0 0 449 464]{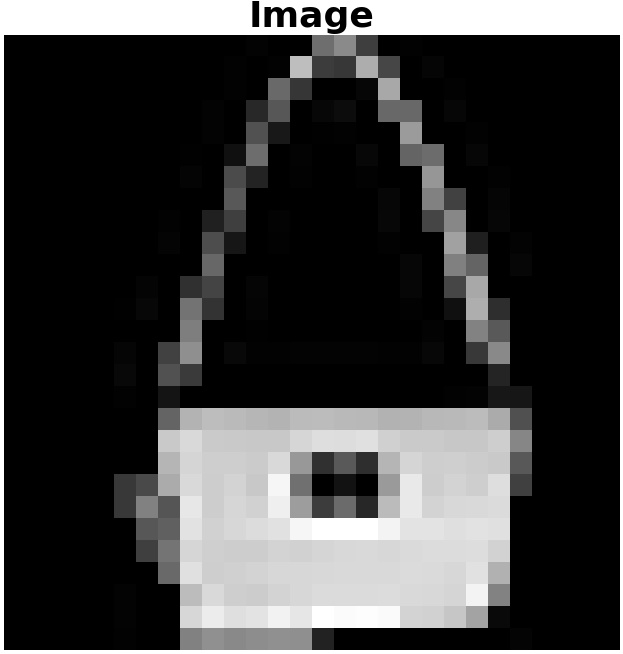} 
			\captionsetup{justification=centering}
			\caption*{}
		\end{subfigure}
		\begin{subfigure}{0.23\textwidth}
			\includegraphics[width=\linewidth,bb=0 0 449 464]{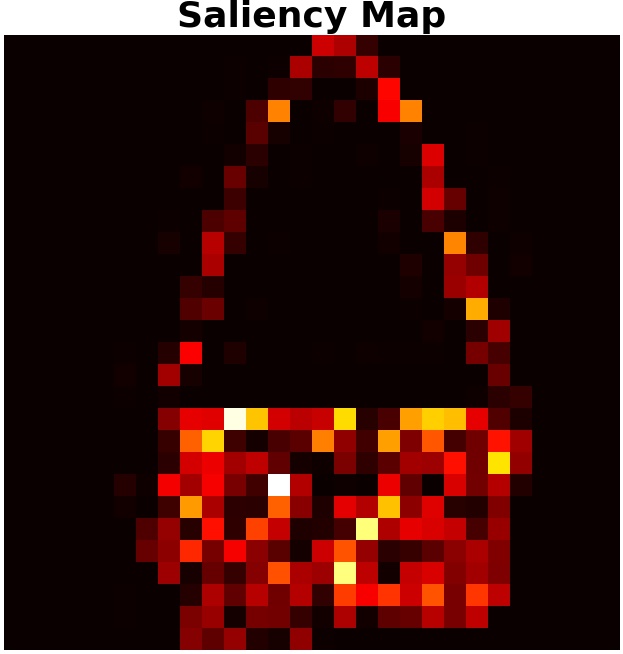} 
			\captionsetup{justification=centering}
			\caption*{Gini: 0.8343}
		\end{subfigure}
		\begin{subfigure}{0.23\textwidth}
			\includegraphics[width=\linewidth,bb=0 0 449 464]{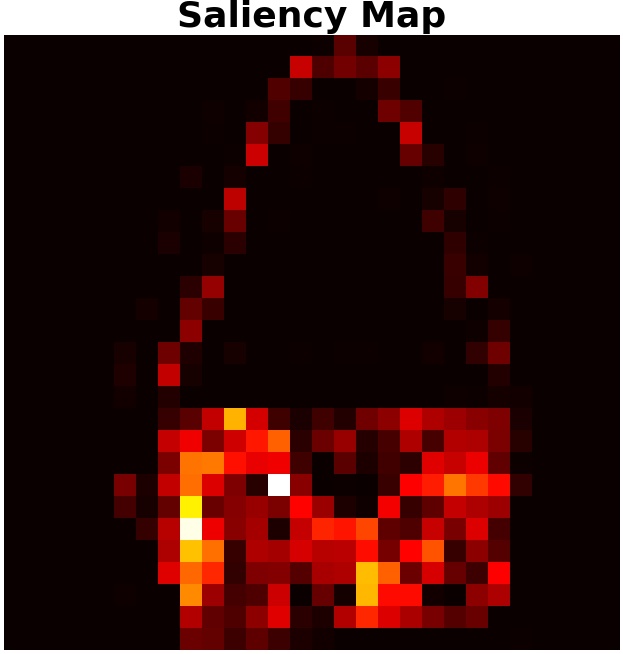} 
			\captionsetup{justification=centering}
			\caption*{Gini: 0.8374}
		\end{subfigure}
		\begin{subfigure}{0.23\textwidth}
			\includegraphics[width=\linewidth,bb=0 0 449 464]{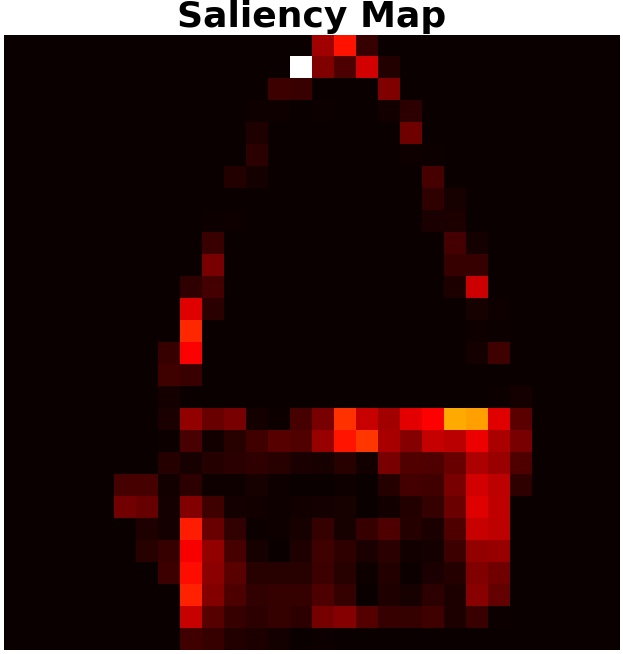} 
			\captionsetup{justification=centering}
			\caption*{Gini: 0.8683}
		\end{subfigure} 
		\caption{For all images, the models give \emph{correct} prediction -- Bag.}
	\end{subfigure}
	\begin{subfigure}{\textwidth}
		\centering
		\begin{subfigure}{0.23\textwidth}
			\includegraphics[width=\linewidth,bb=0 0 449 464]{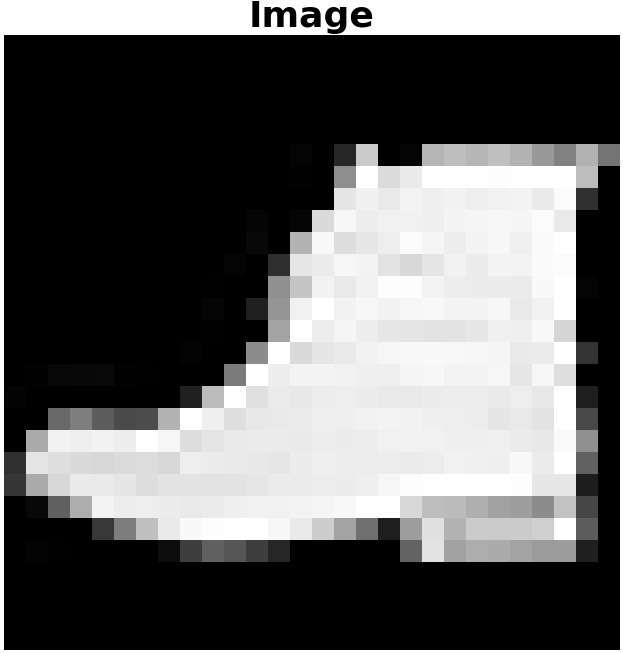} 
			\captionsetup{justification=centering}
			\caption*{}
		\end{subfigure}
		\begin{subfigure}{0.23\textwidth}
			\includegraphics[width=\linewidth,bb=0 0 449 464]{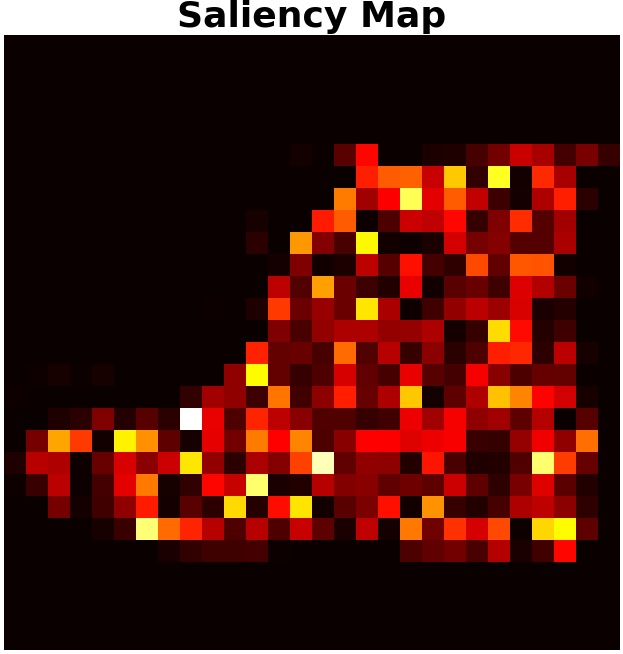} 
			\captionsetup{justification=centering}
			\caption*{Gini: 0.7701}
		\end{subfigure}
		\begin{subfigure}{0.23\textwidth}
			\includegraphics[width=\linewidth,bb=0 0 449 464]{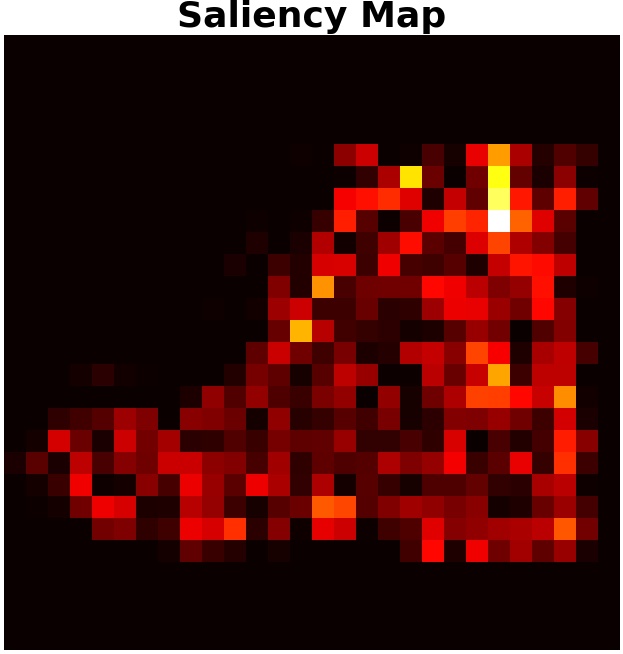} 
			\captionsetup{justification=centering}
			\caption*{Gini: 0.7575}
		\end{subfigure}
		\begin{subfigure}{0.23\textwidth}
			\includegraphics[width=\linewidth,bb=0 0 449 464]{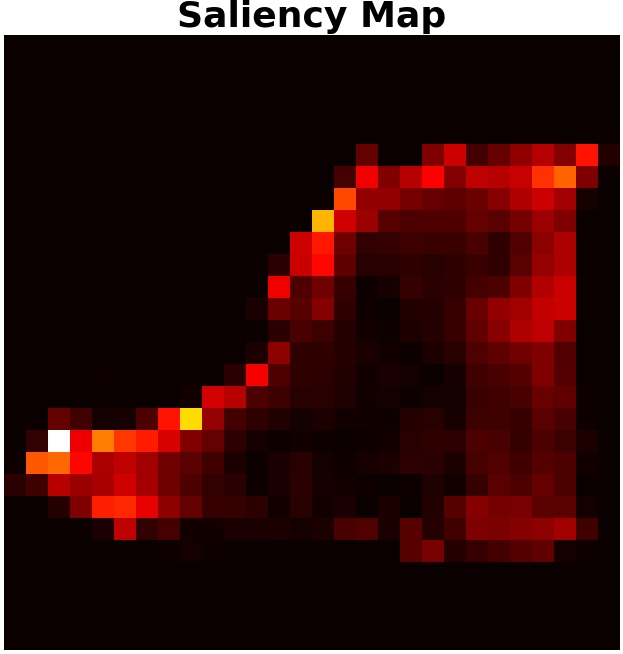} 
			\captionsetup{justification=centering}
			\caption*{Gini: 0.7920}
		\end{subfigure} 
		\caption{For all images, the models give \emph{correct} prediction -- Ankle boot.}
	\end{subfigure}
	\caption{Some examples on Fashion-MNIST. We can see the saliency maps (also called feature importance maps), computed via DeepSHAP, of adversarially trained model are much sparser compared to other models.}
	\label{fig:demo-fashion-mnist-shap}
\end{figure}



\end{appendices}

\end{document}